\documentclass{ieeetmlcn}

\usepackage{standalone} 
\usepackage{graphicx,color}
\usepackage{textcomp}
\usepackage{hyperref}
\hypersetup{hidelinks=true}
\usepackage{algorithm,algorithmic}
\usepackage{amssymb,amsfonts}
\usepackage[tbtags]{amsmath}
\usepackage[utf8]{inputenc}
\usepackage{mathtools}
\usepackage{bbm}
\usepackage{cite}
\usepackage{bm}
\usepackage{siunitx}
\usepackage{stfloats}
\usepackage{subfigure}
\usepackage{color}
\usepackage{xcolor}
\usepackage{amsthm}
\usepackage{enumerate}
\usepackage{multirow}
\usepackage{array}
\usepackage{url}

\usepackage{tabularx,booktabs}
\newtheorem{theorem}{Theorem}
\newtheorem{lemma}{Lemma}
\newtheorem{corollary}{Corollary}
\newtheorem{definition}{Definition}
\newtheorem{assumption}{Assumption}

\allowdisplaybreaks[4]

\def\BibTeX{{\rm B\kern-.05em{\sc i\kern-.025em b}\kern-.08em
    T\kern-.1667em\lower.7ex\hbox{E}\kern-.125emX}}
\AtBeginDocument{\definecolor{tmlcncolor}{cmyk}{0.93,0.59,0.15,0.02}\definecolor{NavyBlue}{RGB}{0,86,125}}

\def\OJlogo{\vspace{-4pt}\includegraphics[height=18pt]{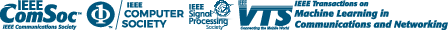}}
\def\seclogo{\vspace{10pt}\includegraphics[height=18pt]{new_logo}}

\def\authorrefmark#1{\ensuremath{^{\textbf{#1}}}}

\begin{document}
\bibliographystyle{IEEEtran}
\receiveddate{XX Month, XXXX}
\reviseddate{XX Month, XXXX}
\accepteddate{XX Month, XXXX}
\publisheddate{XX Month, XXXX}
\currentdate{XX Month, XXXX}
\doiinfo{TMLCN.2022.1234567}

\markboth{}{Author {et al.}}

\title{
Convergence-Privacy-Fairness Trade-Off in Personalized Federated Learning
 }

\author{Xiyu~Zhao\authorrefmark{1,2},
Qimei~Cui\authorrefmark{1,3} (Senior~Member,~IEEE), Weicai~Li\authorrefmark{1,4}, Wei Ni\authorrefmark{4} (Fellow,~IEEE), Ekram Hossain\authorrefmark{5} (Fellow,~IEEE),
Quan~Z. Sheng\authorrefmark{2},
Xiaofeng Tao\authorrefmark{1,3} (Senior Member,~IEEE), and
Ping Zhang\authorrefmark{1,3} (Fellow,~IEEE)}
\affil{School of Information and Communication Engineering, Beijing University of Posts and Telecommunications, Beijing 100876, China}
\affil{School of Computing, Macquarie University, Sydney, NSW 2109, Australia}
\affil{Department of Broadband Communication, Peng Cheng Laboratory, Shenzhen 518055, China}
\affil{School of Electrical and Data Engineering, University of Technology Sydney, Sydney, NSW 2007, Australia}
\affil{Department of Electrical and Computer Engineering, University of Manitoba, Winnipeg, MB R3T 2N2, Canada}
\corresp{Corresponding author: Qimei Cui (email: cuiqimei@bupt.edu.cn).}

\authornote{The work was supported by the Joint funds for Regional Innovation and Development of the National Natural Science Foundation of China(No.U21A20449) and the Fundamental Research Funds for the Central Universities (No.2242022k60006) .}

	\begin{abstract}
Personalized federated learning (PFL), e.g., the renowned Ditto, strikes a balance between personalization and generalization by conducting federated learning (FL) to guide personalized learning (PL). 
While FL is unaffected by personalized model training, in Ditto, PL depends on the outcome of the FL. However, the clients' concern about their privacy and consequent perturbation of their local models can affect the convergence and (performance) fairness of PL. 
This paper presents PFL, called DP-Ditto, which is a non-trivial extension of Ditto under the protection of differential privacy (DP), and analyzes the trade-off among its privacy guarantee, model convergence, and performance distribution fairness. 
We also analyze the convergence upper bound of the personalized models under DP-Ditto and derive the optimal number of global aggregations given a privacy budget. Further, we analyze the performance fairness of the personalized models, and reveal the feasibility of optimizing DP-Ditto jointly for convergence and fairness.
Experiments validate our analysis and demonstrate that DP-Ditto can surpass the DP-perturbed versions of the state-of-the-art PFL models, such as FedAMP, pFedMe, APPLE, and FedALA, by over $32.71\%$ in fairness and $9.66\%$ in accuracy.

	\end{abstract}
	\begin{IEEEkeywords}
	Personalized federated learning, differential privacy, convergence, fairness.
	\end{IEEEkeywords}

\maketitle
 
	\section{Introduction}

\IEEEPARstart{R}{Recently}, {personalized federated learning (PFL) has been proposed, such as federated learning (FL) with meta-learning. It tailors learning processes to provide personalized models for individual clients while benefiting from the global perception offered by FL, thereby capturing both generalization and personalization in the models.
PFL can strike a balance between personalized models and the global model, e.g., via a global-regularized multi-task framework~\cite{li2021ditto}.
It can provide customized services for applications with heterogeneous local data distributions or tasks, e.g., intelligent Internet of Things networks with geographically dispersed clients~\cite{zhang2021optimizing,zhang2024det,cui2022multi}.}
A popular PFL technique is called Ditto, which was developed to adapt to the heterogeneity in FL settings by simultaneously learning a global model and distinct personal models for multiple agents~\cite{li2021ditto}. 


While preserving personal data locally throughout its training process, similar to FL, PFL can still suffer from privacy leakage due to the incorporation of FL. Differential privacy (DP)~\cite{abadi2016deep} can be potentially applied to protect the privacy of Ditto.
In each round, every client trains two models separately, including its local model and personalized model, based on its local dataset and the global model broadcast by the server in the last round. With DP, the clients perturb their local models by adding noise before uploading them to the server, where the perturbed local models are aggregated to update the global model needed for the clients to train their local models and personalized models further. This privacy-preserving PFL model, namely DP-Ditto, is a new framework.
{It is important to carefully balance convergence, privacy, and fairness in DP-Ditto.  However, the
impact of privacy preservation, i.e., the incorporation of DP, on the convergence and fairness of the personalized models has not been investigated in the literature. Let alone a PFL design with a balanced consideration between convergence, privacy, and fairness.}

Although privacy and fairness have been studied separately in the contexts of both FL and PFL, e.g.,~\cite{zhang2022federated,wei2021low,elgabli2021harnessing,wei2020federated,zhao2020local,truex2020ldp,yuan2023amplitude,chen2022feddual,hu2020personalized,liu2022privacy,sun2021pain,liu2024differentially,okegbile2023differentially,park2023differential,yan2024peaches,wei2023personalized}, they have never been considered jointly in PFL, such as Ditto~\cite{li2021ditto}. Their interplay has been overlooked. 
The majority of the existing studies, e.g.,~\cite{t2020personalized,li2020federated,huang2021personalized,luo2022adapt,zhang2023fedala,fallah2020personalized,wei2023personalized,li2019fedmd,you2022semi,zhang2023federated}, have focused on the model accuracy of PFL.
Some other existing works, e.g.,~\cite{li2020fair,hu2022federated,li2021ditto}, have attempted to improve the performance distribution fairness of FL.
None of these studies has addressed the influence of DP on the fairness and accuracy of PFL.

This paper studies the trade-off between privacy guarantee, model convergence, and performance distribution fairness of privacy-preserving PFL, more specifically, DP-Ditto. 
We analyze the convergence upper bound of the personalized learning (PL) and accordingly optimize the aggregation number of FL given a privacy budget. We also analyze the fairness of PL in DP-Ditto on a class of linear problems, revealing the possibility of maximizing the fairness of privacy-preserving PFL given a privacy budget and aggregation number.
{
 To the best of our knowledge, this is the first work that investigates the trade-off among the privacy, convergence, and fairness of PFL, and optimizes PFL for convergence under the constraints of performance distribution fairness and privacy requirements.}
The major contributions of this paper are summarized as follows:
\begin{itemize}
\item 
We propose a new privacy-preserving PFL framework, i.e., DP-Ditto, by incorporating an $(\epsilon,\delta)$-DP perturbation mechanism into Ditto.
This extension is non-trivial and necessitates a delicate balance between convergence, privacy, and fairness.
\item 
A convergence upper bound of DP-Ditto is derived, capturing the impact of DP on the convergence of personalized models. 
The number of global aggregations is identified by minimizing the convergence upper bound.  
\item 
We analyze the fairness of DP-Ditto on a class of linear problems to reveal the conditional existence and uniqueness of the optimal setting balancing convergence and fairness, given a privacy requirement. 
\end{itemize}

Extensive experiments validate our convergence and fairness analysis of DP-Ditto and the obtained optimal FL aggregation number and weighting coefficients of FL versus PL. Three image classification tasks are performed using deep neural network (DNN), multi-class linear regression (MLR), and convolutional neural network (CNN) on the Federated MNIST, Federated FMNIST, and Federated CIFAR10 datasets. 
DP-Ditto can outperform its benchmarks, i.e., the DP-perturbed FedAMP~\cite{huang2021personalized}, pFedMe~\cite{t2020personalized}, APPLE~\cite{luo2022adapt}, and FedALA~\cite{zhang2023fedala}, by $99.98\%$, $32.71\%$, $97.04\%$, and $99.72\%$, respectively, in fairness and $59.06\%$, $9.66\%$, $28.67\%$, and $64.31\%$ in accuracy.

The rest of this paper is structured as follows. Section II presents a review of related works. Section III outlines the system and threat models of DP-Ditto and analyzes its privacy and DP noise variance. In Section~IV, the convergence upper bound of DP-Ditto is established, and the optimal number of FL global aggregations is obtained accordingly. In Section~V, we analyze the fairness of PL on a class of linear regression problems to demonstrate the feasibility of fairness maximization. The experimental results are discussed in Section~VI. The conclusions are given in Section~VII.

\section{Related Work}     

\subsection{Personalization}
PFL frameworks have been explored to combat statistical heterogeneity through transfer learning (TL) \cite{li2019fedmd}, meta-learning \cite{fallah2020personalized,wei2023personalized}, and other forms of multitask learning (MTL) \cite{t2020personalized,li2020federated,huang2021personalized,luo2022adapt,zhang2023fedala}. 
None of these has addressed the fairness among the participants of PFL.
TL 
conveys knowledge from an originating domain to a destination domain.
TL-based FL enhances personalization by diminishing domain discrepancy of the global and local models~\cite{tan2022towards}. 
FedMD \cite{li2019fedmd} 
is an FL structure grounded in TL and knowledge distillation (KD), enabling clients to formulate autonomous models utilizing their individual private data. Preceding the FL training and KD, TL is implemented by employing a model previously trained on a publicly available dataset.

Meta-learning finds utility in FL in enhancing the global model for rapid personalization.
In~\cite{fallah2020personalized}, a 
variation of FedAvg, named Per-FedAvg, was introduced, leveraging
the Model-Agnostic Meta-Learning (MAML). 
It acquired a proficient initial global model that is effective on a novel heterogeneous task and can be achieved through only a few gradient descent steps.
You \textit{et al.} \cite{you2022semi} further proposed a Semi-Synchronous Personalized FederatedAveraging (PerFedS) mechanism based on MAML, where the server sends a meta-model to a set of UEs participating in the global updating and the stragglers in each round.
In another meta-learning-based PFL framework~\cite{wei2023personalized}, a privacy budget allocation scheme based on R\'{e}nyi DP composition theory was designed to address information leakage arising from two-stage gradient descent.

MTL trains a model to simultaneously execute several related tasks. 
By considering an FL client as a task, there is the opportunity to comprehend the interdependence among the clients manifested by their diverse local data.
In~\cite{t2020personalized}, pFedMe 
employing Moreau envelopes as the regularized loss functions for clients was recommended to disentangle the optimization of personalized models from learning the global model.
The global model is obtained by aggregating the local models updated based on the personalized models of the clients. Each client's personalized model maintains a bounded distance from the global model.
In \cite{li2020federated}, 
FedProx was formulated by incorporating a proximal term into the local subproblem. Consequently, the contrast was delineated between the global and local models to ease the modulation of the influence of local updates.
In \cite{zhang2023federated}, a federated multi-task learning (FMTL) framework was developed, where the server broadcasts a set of global models aggregated based on the local models of different clusters of clients, and each client selects one of the global models for its local model updating.

Huang \textit{et al.} \cite{huang2021personalized} 
integrated PFL with supplementary terms and employed a federated attentive message passing (FedAMP) strategy to mitigate the impact of diverse data.
Consequently, the convergence of the FedAMP was guaranteed. 
A protocol named APPLE \cite{luo2022adapt} was proposed to improve the personalized model of each client based on the others' models. Clients obtain the personalized models locally by aggregating the core models of other clients downloaded from the server. The aggregation weights and the core models are locally learned from the personalized model by adding a proximal term to the local objectives. Instead of overwriting the old local model with the downloaded global model, FedALA \cite{zhang2023fedala} aggregates the downloaded global model and the old local model for local model initialization.

These existing PFL frameworks~\cite{li2019fedmd,fallah2020personalized,wei2023personalized,t2020personalized,li2020federated,huang2021personalized,luo2022adapt,zhang2023fedala} have focused primarily on model accuracy. None of these has taken the fairness of the personalized models into consideration.

\subsection{Privacy}
Existing studies \cite{wei2020federated,zhao2020local,truex2020ldp,yuan2023amplitude,chen2022feddual} 
have explored ways to integrate privacy techniques into FL to provide a demonstrable assurance of safeguarding privacy.
However, little to no consideration has been given to the personalization of learning models and their fairness.
In \cite{wei2020federated}, a DP-based framework was 
suggested to avert privacy leakage by introducing noise to obfuscate the local model parameters.
In \cite{zhao2020local}, three local DP (LDP) techniques 
were devised to uphold privacy, where LDP was incorporated into FL to forecast traffic status, mitigate privacy risks, and diminish communication overhead in crowd-sourcing scenarios.
The authors of \cite{truex2020ldp} 
suggested FL with LDP, wherein LDP-based perturbation was applied during model uploading, adhering to individual privacy budgets.
Liu \textit{et al.} \cite{liu2024differentially} proposed a transceiver protocol to maximize the convergence rate under privacy constraints in a MIMO-based DP FL system, where a server performs over-the-air model aggregation and parallel private information extraction from the uploaded
local gradients with a DP mechanism.

In \cite{yuan2023amplitude}, DP noises were adaptively added to local model parameters to preserve user privacy during FL. The amplitude of DP noises was adaptively adjusted to balance preserving privacy and facilitating convergence. 
FedDual \cite{chen2022feddual} was designed to preserve user privacy by adding DP noises locally and aggregating asynchronously via a gossip protocol. Noise-cutting was adopted
to alleviate the impact of the DP noise on the global model. 
Hu \textit{et al.} \cite{hu2020personalized} proposed privacy-preserving PFL using the Gaussian mechanism, which provides a privacy guarantee by adding Gaussian noise to the uploaded local updates. 
In \cite{liu2022privacy}, the Gaussian mechanism was considered in a mean-regularized MTL framework, and the accuracy was analyzed for single-round FL using a Bayesian framework. In~\cite{wei2023personalized}, the allocation of a privacy budget was considered for meta-learning-based PFL.
In \cite{okegbile2023differentially}, differentially private federated MTL (DPFML) was designed for human digital twin systems by integrating DPFML and a computational-efficient blockchain-enabled validation process.

These existing works~\cite{wei2020federated,zhao2020local,truex2020ldp,yuan2023amplitude,chen2022feddual,hu2020personalized,liu2022privacy,wei2023personalized} have given no consideration to fairness among the participants in FL, especially in the presence of statistical heterogeneity.

\subsection{Fairness}

Some existing studies, e.g., \cite{li2020fair,hu2022federated,li2021ditto}, have attempted to improve performance distribution fairness, i.e., by 
mitigating the variability in model accuracy among different clients.
Yet, none has taken user privacy into account.
In \cite{li2020fair}, $q$-FFL was proposed to achieve a more uniform accuracy distribution across clients. A parameter $q$ was used to re-weight the aggregation loss by assigning bigger weights to clients undergoing more significant losses.
In \cite{hu2022federated}, FedMGDA+ 
was suggested to enhance the robustness of the model while upholding fairness with positive intentions.
A multi-objective problem 
was structured to diminish the loss functions across all clients. It was tackled by employing Pareto-steady resolutions to pinpoint a collective descent direction suitable for all the chosen clients.
Li \textit{et al.} \cite{li2021ditto} designed a scalable federated MTL framework Ditto, which simultaneously learns personalized and global models in a global-regularized framework. 
Regularization was introduced to bring the personalized models in proximity to the optimal global model.
The optimal weighting coefficient of Ditto was designed in terms of fairness and robustness.
These studies~\cite{li2020fair,hu2022federated,li2021ditto} have overlooked privacy risks or failed to address the influence of DP on fairness.

\section{Framework of PFL}

\subsection{PFL}
PFL consists of a server and $N$ clients. $\mathbb{N}$ denotes the set of clients. $\mathcal{D}_n$ denotes the local dataset at client $n \in \mathbb{N}$. $\mathcal{D}$ is the collection of all data samples. $\left|\mathcal{D}\right|={\sum}_{n=1}^{N}\left|\mathcal{D}_{n}\right|$ is the total size of all data samples, $\left|\cdot\right|$ stands for cardinality. 
Like Ditto, PFL has both global and personal objectives for FL and PL, respectively. At the server, the global objective is to learn a global model with the minimum global training loss:
\begin{equation}
    \label{GlobalObj}
    \underset{\boldsymbol{\omega}}{\min}\, F(F_{1}(\boldsymbol{\omega}),\cdots,F_{N}(\boldsymbol{\omega}))  \,,
\end{equation}
{where $\boldsymbol{\omega} \in \mathbb{R}^{d}$ is the model parameter with $d$ elements}, $F_{n}(\cdot)$ is the local loss function of client $n \in \mathbb{N}$, and $F(\cdot, \cdots, \cdot)$ is the global loss function:
\begin{equation}
    \label{g_loss}
    F(F_{1}(\boldsymbol{\omega}),\cdots,F_{N}(\boldsymbol{\omega}))={\sum}_{n=1}^{N}p_nF_n(\boldsymbol{\omega}) \,,
\end{equation}
where $p_{n}\triangleq\left|\mathcal{D}_{n}\right|/\left|\mathcal{D}\right|$ with ${\sum}_{n=1}^{N}p_n=1$. We assume the size of each client's local dataset is the same, i.e., $p_n=\frac{1}{N}$. 

{To capture both generalization and personalization} as in Ditto, for client $n$, we encourage its personalized model to be close to the optimal global model, i.e.,
	\begin{subequations} 
        \label{objective_p}
	\begin{align}
     \underset{\boldsymbol{\varpi}_{n}}{\min} &\,f_{n}(\boldsymbol{\varpi}_{n};\boldsymbol{\omega}^{\ast})\!=\!\left(1\!-\!\frac{\lambda}{2}\right)F_{n}(\boldsymbol{\varpi}_{n})\!+\!\frac{\lambda}{2}\parallel\!\boldsymbol{\varpi}_{n}\!-\!\boldsymbol{\omega}^{\ast}\!\parallel^{2} \label{fn} \\
     \label{fn_1}
     \textrm{s.t.}  &\,\,\boldsymbol{\omega}^{\ast}=\underset{\boldsymbol{\omega}}{\arg\min}\, \frac{1}{N} \sum_{n=1}^{N} F_{n}\left(\boldsymbol{\omega}\right) \,, 
	\end{align}
	\end{subequations}
 where $f_n(\cdot)$ is the loss function of the personalized model; $\lambda \in [0,2]$ is a {weighting coefficient} that controls the trade-off between the global and local models. When $\lambda=0$, PFL trains a local model for each client based on its local datasets. When $\lambda=2$, the personal objective becomes obtaining an optimal global model with no personalization. Let $\boldsymbol{u}_{n}^{\ast}$ and $\boldsymbol{\varpi}_{n}^{\ast}$ be the optimal local model based on the local data and the optimal personalized model, i.e.,
	\begin{align}
    & \boldsymbol{u}_{n}^{\ast}=\underset{\boldsymbol{u}_{n}}{\arg\min}\,F_{n}(\boldsymbol{u}_{n})  \,;
    \boldsymbol{\varpi}_{n}^{\ast}=\underset{\boldsymbol{\varpi}_{n}}{\arg\min}\,f_{n}(\boldsymbol{\varpi}_{n};\boldsymbol{\omega}^{\ast}) \,. \label{Optimal_parameter} 
	\end{align}
According to (\ref{GlobalObj})--(\ref{objective_p}), a local model $\boldsymbol{u}_n$ is trained for global aggregation. A personalized model $\boldsymbol{\varpi}_{n}$ is adjusted according to local training and the global model at each client $n$. The training of the global model and that of the personalized models are assumed to be synchronized, i.e., at each round $t+1$. Client $n$ updates its personalized model $\boldsymbol{\varpi}_{n}^{t+1}$ based on the global model $\boldsymbol{\omega}^t$ updated at the $t$-th round. 

\begin{table}[t]
\caption{Notation and definitions}
\label{parameters1}
\centering
\begin{tabular}{*{2}{l}}
\hline
\textbf{Notation} & \textbf{Definition} \\ 
\hline
$\mathcal{D}$, $\mathcal{D}_n$ & The local datasets of all clients and client $n$ \\
$F(\cdot)$, $F_n(\cdot)$ & The global and the local loss function of client $n$\\
$f_n(\cdot)$ & The loss function of client $n$'s personalized model \\
$\nabla F(\cdot)$ & Gradient of a function $F(\cdot)$ \\
$g_{n}(\boldsymbol{\varpi}_{n}^{t};\boldsymbol{\omega}^{t})$ & The stochastic gradient of $f_{n}(\boldsymbol{\varpi}_{n}^{t};\boldsymbol{\omega}^{t})$ \\
$\boldsymbol{\omega}$, $\boldsymbol{\omega}^{\ast}$ & The global and optimal global models  \\
$\boldsymbol{u}_{n}$, $\boldsymbol{u}_{n}^{\ast}$ & The local and optimal local models of client $n$ \\
$\boldsymbol{\varpi}_{n}$, $\boldsymbol{\varpi}_{n}^{\ast} $ & The personalized and optimal personalized models \\
&of client $n$ \\ 
$\overset{\sim}{\boldsymbol{\omega}}^{t}$ & The global model with DP noise at $t$-th aggregation\\
$\overset{\sim}{\boldsymbol{u}}_{n}^{t}$, $\overset{\sim}{\boldsymbol{\varpi}}_{n}^{t}$ & The local and personalized models of client $n$ \\ &with DP noise at $t$-th aggregation\\
$\mathbf{z}_n^t$& The DP noise added by client $n$ at $t$-th aggregation \\
$\mathbf{z}^t$ & The equivalent noise term imposed on the global\\ & model after $t$-th aggregation  \\
$\sigma_u$, $\sigma_z$ & The standard deviation of the DP noise $\mathbf{z}_n^t$ and $\mathbf{z}^t$ \\
$C$ & The clipping threshold of local model \\
$T$ & The maximum number of global aggregations \\
$\epsilon$, $\delta$ & The DP requirement \\
$\lambda$ & The weighting coefficient of the personalized models \\
$\Delta s$ & The sensitivity of client $n$'s local training process \\
$\eta_{\mathrm{G}}$, $\eta_{\mathrm{L}}$ & The learning rates of the local model and \\ &the personalized model \\
$\varrho(\cdot)$ & The fairness of a group of models \\
$\mathbf{Y}_{n}$, $\mathbf{X}_{n}$ & The local observations of client $n$ \\
$b$ & The number of local samples \\
$\hat{\boldsymbol{u}}_{n}$ & The estimate of $\boldsymbol{u}_{n}^{\ast}$ \\
$\boldsymbol{\tau}_{n}$ & The errors between $\boldsymbol{u}_{n}^{\ast}$ and $\boldsymbol{\omega}^{\ast}$ \\
$\boldsymbol{\nu}_{n}$ & The errors between $\mathbf{Y}_{n}$ and $\mathbf{X}_{n}\boldsymbol{u}_{n}^{\ast}$ \\
$\zeta^{2}$, $\sigma^{2}$ & The diagonal elements of $\boldsymbol{\tau}_{n}$ and $\boldsymbol{\nu}_{n}$ \\ 
$\sigma_{w}^{2}$ & The diagonal element of the diagonal matrix $\boldsymbol{\vartheta}_{n}$ \\
$\rho$ & The diagonal element of $\mathbf{X}_{n}^{\intercal}\mathbf{X}_{n}$ \\
\hline
\end{tabular}
\end{table}

\subsection{Threat Model}
\label{threat_model}
The server may attempt to recover the training datasets or infer the private features based on the models uploaded by the clients. There may be external attackers who intend to breach the privacy of the clients. Although the clients train models locally, the local models that the clients share with the server can be analyzed to potentially compromise their privacy under inference attacks during learning \cite{nasr2019comprehensive} and model-inversion attacks during testing \cite{fredrikson2015model}. The private information can be recovered by the attackers.

{Typical privacy-preserving methods for FL include homomorphic encryption, secure multi-party computation, and DP~\cite{lyu2022privacy}. 
While preventing the server from deciphering local models, homomorphic encryption requires all devices to use the same private key and cannot stop them from eavesdropping on each other. 
Secure multi-party computation enables clients to collaboratively compute an arbitrary functionality,  
but requires multiple interactions in a learning process, e.g., public key sharing among clients for key agreement, at the expense of high computation and communication overhead~\cite{bonawitz2017practical}.
Typically, homomorphic encryption and secure multi-party computation are computationally expensive and need a trusted third party for key agreement~\cite{acar2018survey}. To this end, DP is employed to preserve the privacy of PFL in this paper.}

\subsection{FL With DP}
{
Consider the threat model described in Section~III--\ref{threat_model}. The risk of privacy breaches arises from uploading FL local models to the server for FL global model aggregation. 
}%
To preserve data privacy from the uploaded local models, a Gaussian DP mechanism can be used to guarantee $(\epsilon,\delta)$-DP by adding artificial Gaussian noises \cite{dwork2014algorithmic}. 
Let $\mathbf{z}_n^t\sim \mathcal{N}(0,\sigma_u^2)$ denote the Gaussian noise added by client $n$ at the $t$-th communication round. $\mathbf{z}^t=\sum_{n=1}^{N} \mathbf{z}_n^t$ is the noise imposed on the global model after the $t$-th aggregation. The additive noise of each client is independent and identically distributed (i.i.d.). Each element in $\mathbf{z}^t$ follows $\mathcal{N}(0,\sigma_z^2)$ with $\sigma_z^2=N\sigma_u^2$.

Note that the DP noise is only added when a client uploads its local model for global model aggregation. 
{Before uploading its local model in round $t$, each client $n$ clips its local model to prevent gradient explosion, as given by
\begin{align}
\boldsymbol{u}_{n}^{t+1}=\boldsymbol{u}_{n}^{t+1}/\max(1,\frac{\parallel \boldsymbol{u}_{n}^{t+1}\parallel}{C})  \,,
    \label{clipping}
	\end{align}
 where $C$ is the pre-determined clipping threshold to ensure $\parallel \boldsymbol{u}_{n}\parallel \leq C$~\cite{yuan2023amplitude}.
}%

To guarantee $(\epsilon,\delta)$-DP with respect to the data used in the training of the local model, the standard deviation of $\mathbf{z}_n^t$ from the Gaussian mechanism should satisfy 
    $\sigma_u=\frac{\Delta s\sqrt{2qT\ln\left(1/\delta\right)}}{\epsilon}  $~\cite{wei2020federated},
where $T$ is the maximum number of communication rounds and $q$ is the sampling ratio. {Assume that all clients train and upload their local models at each communication round, 
i.e., $q=1/N$.  } 
{$\Delta s$ is the sensitivity of client $n$'s local training process, which captures the magnitude that a sample can change the training model in the worst case, as given by
\begin{equation}
    \label{sensitivity}
    \Delta s=\underset{\mathcal{D}_{n},\mathcal{D}'_{n}}{\max}\left\Vert \boldsymbol{u}_{n}(\mathcal{D}_{n})-\boldsymbol{u}_{n}(\mathcal{D}'_{n})\right\Vert  \,,
\end{equation}
where $\boldsymbol{u}_{n}(\mathcal{D}_{n})$ and $\boldsymbol{u}_{n}(\mathcal{D}'_{n})$ are the local models obtained from the datasets $\mathcal{D}_{n}$ and $\mathcal{D}'_{n}$, respectively. Here, $\mathcal{D}_{n}=\mathcal{D}\cup s$ and $\mathcal{D}'_{n}=\mathcal{D}\cup s'$ are two adjacent datasets with the same size and differ by one sample, i.e., $s \in \mathcal{D}_{n}$, $s' \in \mathcal{D}'_{n}$, and $s\neq s'$. Considering the local model training from $\mathcal{D}_{n}$ and $\mathcal{D}'_{n}$, we have~\cite{wei2020federated}
\begin{align}
    \label{sensitivity_1}
    \Delta s=&\underset{\mathcal{D}_{n},\mathcal{D}'_{n}}{\max}\left\Vert \frac{1}{\mid\mathcal{D}_{n}\mid}\sum_{s''\in\mathcal{D}_{n}}\arg\min_{\boldsymbol{\omega}}F_n(\boldsymbol{\omega},\mathcal{D}_{n})\right.\nonumber\\
    &\left.-\frac{1}{\mid\mathcal{D}'_{n}\mid}\sum_{s''\in\mathcal{D}'_{n}}\arg\min_{\boldsymbol{\omega}}F_n(\boldsymbol{\omega},\mathcal{D}'_{n})\right\Vert \nonumber \\
    =&\frac{\underset{s,s'}{\max}\left\Vert \arg\min_{\boldsymbol{\omega}}F_n(\boldsymbol{\omega},s)-\arg\min_{\boldsymbol{\omega}}F_n(\boldsymbol{\omega},s')\right\Vert}{\mid\mathcal{D}_{n}\mid} \nonumber \\
    =& \frac{2C}{\mid\mathcal{D}_{n}\mid},
\end{align}
}%
{Clearly, $\Delta s$ only depends on the size of datasets $\mid\mathcal{D}_{n}\mid$ and the clipping threshold $C$. 
}%

According to~(\ref{sensitivity}), the standard deviations of the DP noise $\sigma_u$ per client and $\mathbf{z}^t$ are given by 
    \begin{align}
    \label{noiseSigma1}
    &\sigma_u=\frac{\Delta s\sqrt{2TN\ln\left(1/\delta\right)}}{\epsilon N}  \,; 
    \sigma_z=\frac{\Delta s\sqrt{2T\ln\left(1/\delta\right)}}{\epsilon}  \,.
    \end{align}
 The operations of PFL are summarized in Algorithm~\ref{algorithm} and illustrated in Fig.~\ref{fig:PFL_model}. At each round~$t$, client $n$ executes local training, and updates its local model $\boldsymbol{\omega}_n^{t+1}$ and personalized model $\boldsymbol{\varpi}_{n}^{t+1}$. The learning rates of the local and personalized models are $\eta_{\mathrm{G}}$ and $\eta_{\mathrm{L}}$, respectively. The noisy local model $\overset{\sim}{\boldsymbol{u}}_{n}^{t+1}$ after clipping and DP perturbation is uploaded by client $n$ to the server for updating the global model 
 $\overset{\sim}{\boldsymbol{\omega}}^{t+1}$.
 	\begin{algorithm}[t]
	\caption{Privacy-Preserving PFL}
        \label{algorithm}
	\begin{algorithmic}[1]
	\REQUIRE $T$, $\lambda$, $\boldsymbol{\omega}^0$, $\{\mathbf{\boldsymbol{\varpi}}_{n}^0\}_{n\in\mathbb{N}}$, $N$, $\eta_{\mathrm{G}}$, $\eta_{\mathrm{L}}$, $\epsilon$ and $\delta$.
	\ENSURE $\boldsymbol{\omega}^T$, $\{\mathbf{\boldsymbol{\varpi}}_{n}^T\}_{n\in\mathbb{N}}$.
	\FOR {$t=\{0,\cdots,T-1\}$}
        \STATE $//$ Local training process for global model;
        \FOR {$n \in \mathbb{N}$}
        \STATE Obtain $\boldsymbol{\omega}^{t}$ and let $\boldsymbol{u}_{n}^{t}=\boldsymbol{\omega}^{t}$ ;
	\STATE Update the local model: 
        $\boldsymbol{u}_{n}^{t+1}=\boldsymbol{u}_{n}^{t}-\eta_{\mathrm{G}}\nabla F_{n}(\boldsymbol{u}_{n}^{t})$.
        \STATE Clip the local model:
        $\boldsymbol{u}_{n}^{t+1}=\boldsymbol{u}_{n}^{t+1}/\max(1,\frac{\parallel \boldsymbol{u}_{n}^{t+1}\parallel}{C})$;
        \STATE Add noise and upload:
        $\overset{\sim}{\boldsymbol{u}}_{n}^{t+1}=\boldsymbol{u}_{n}^{t+1}+\mathbf{z}_{n}^{t+1}$

        \STATE $//$ Local training process for personalized model;
        \STATE Update personalized model $\boldsymbol{\varpi}_{n}^{t+1}$ :
        \STATE $\boldsymbol{\varpi}_{n}^{t+1}=\boldsymbol{\varpi}_{n}^{t}-\eta_{\mathrm{L}}(\left(1-\frac{\lambda}{2}\right)\nabla F_{n}(\boldsymbol{\varpi}_{n}^{t})+\lambda(\boldsymbol{\varpi}_{n}^{t}-\boldsymbol{\omega}^{t}))$;
        \ENDFOR
        \STATE $//$ Global model aggregating process;
        \STATE Update the global model:
        $\overset{\sim}{\boldsymbol{\omega}}^{t+1}=\frac{1}{N}{\sum}_{n=1}^{N}\overset{\sim}{\boldsymbol{u}}_{n}^{t+1}$;
        \ENDFOR
	\end{algorithmic}
	\end{algorithm}

	\begin{figure}[!t]
	\centering
        \includegraphics[width=0.45\textwidth]{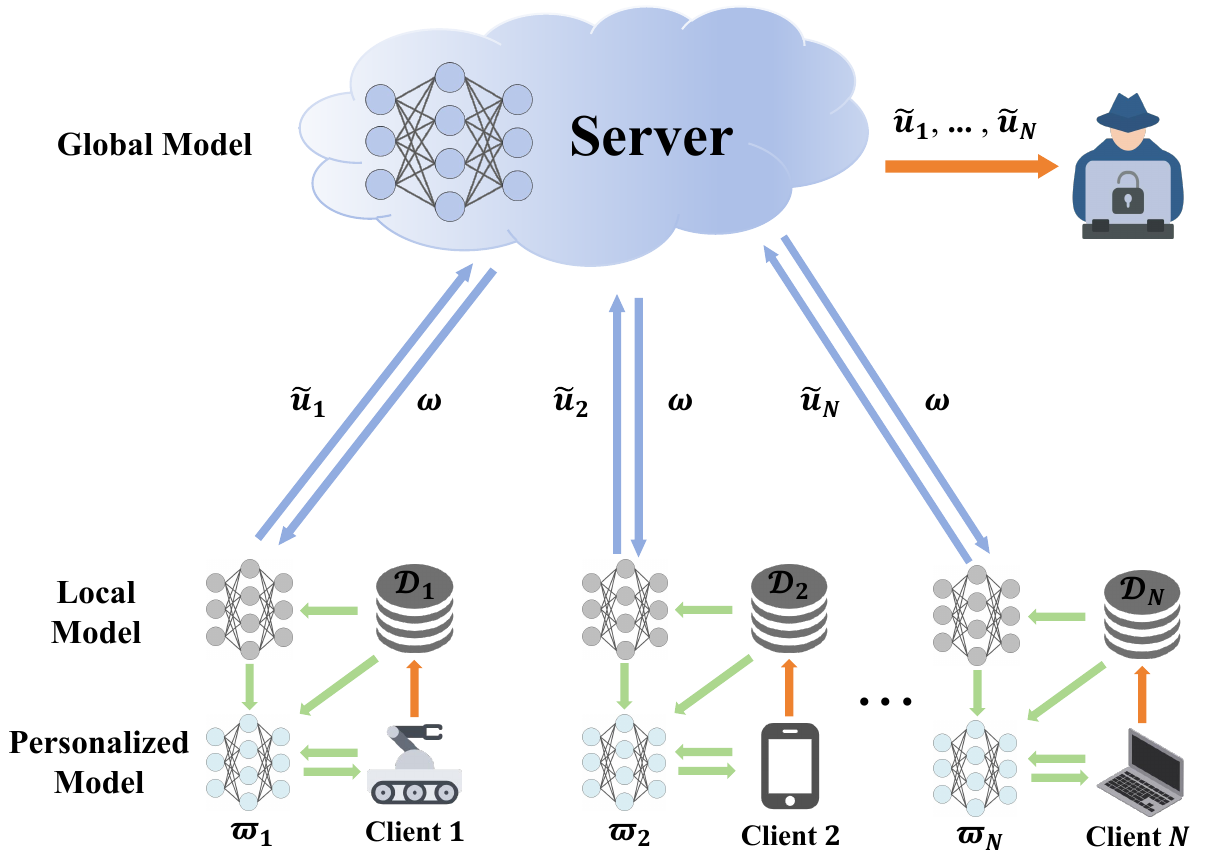}
	\caption{The diagram of PFL: In each round, every client trains its local model and its personalized model based on its local dataset and the global model broadcast by the server in the last round. Then, the clients perturb and upload their local models to the server, and the server aggregates the perturbed local models into the global model and broadcasts the global model.}
	\label{fig:PFL_model}
	\end{figure}

\section{Convergence of Privacy-Preserving PFL
}

This section establishes the convergence upper bound of PFL 
and optimizes the number $T$ of aggregation rounds to minimize the convergence upper bound of PL. 
The following assumptions facilitate the convergence analysis of {PFL}.
\begin{assumption}
\label{assumption1}
$\forall n \in \mathbb{N}$, 
\begin{itemize}
    \item 
$F_{n}\left(\cdot\right)$ is $\mu$-strongly convex \cite{karimi2016linear,10123399,10542235} and $L$-smooth \cite{o2006metric,10123399,10542235}, i.e., $F\left(\boldsymbol{\omega}\right)-F\left(\boldsymbol{\omega}^{\ast}\right)\leq\frac{1}{2\mu}\parallel\nabla F\left(\boldsymbol{\omega}\right)\parallel^{2}$ and $\parallel\nabla F\left(\boldsymbol{\omega}\right)-\nabla F\left(\boldsymbol{\omega}'\right)\parallel\leq L\parallel\boldsymbol{\omega}-\boldsymbol{\omega}'\parallel$. Here, $\mu$ and $L$ are constants;
\item 
The global learning rate $\eta_{\mathrm{G}}\leq\frac{2}{L}$, and $\mu>\frac{2-2\lambda}{2-\lambda}$;
\item 
$\mathbb{E}\left[\parallel\nabla F_{n}(\boldsymbol{\omega}^{t})\parallel^{2}\right]\leq G_{0}^{2}$ {with $G_{0}$ being a constant.};
\item 
$\parallel \boldsymbol{u}_{n}^{\ast}-\boldsymbol{\omega}^{\ast}\parallel\leq M$, where $M$ is a constant.

\end{itemize}
\end{assumption}

\subsection{Convergence Analysis}
\subsubsection{Convergence of FL}
The convergence upper bound of FL with DP has been established in the literature \cite[Eq.~(16)]{wei2020federated} 
    	\begin{equation} 
        \label{gloCon}
\mathbb{E}\left[F\left(\overset{\sim}{\boldsymbol{\omega}}^{T}\right)-F\left({\boldsymbol{\omega}}^{\ast}\right)\right]\leq \varepsilon_{\mathrm{G}}^T \Psi_{1}
+(1-\varepsilon_{\mathrm{G}}^T)\frac{\varphi T}{N\epsilon^2}
        \,,
	\end{equation}
where, for conciseness, $\varepsilon_{\mathrm{G}}=1-2\mu\eta_{\mathrm{G}}+\mu\eta_{\mathrm{G}}^{2}L$, $\Psi_{1}=F\left(\boldsymbol{\omega}^{0}\right)-F\left(\boldsymbol{\omega}^{\ast}\right)$, and $\varphi=\frac{L^2 \Delta s^{2} \ln\left(1/\delta\right)}{\mu}$. 
Clearly, 
the DP noise increases the convergence upper bound with $\varphi=\frac{L^2\sigma_z^2}{2T}$.

\subsubsection{Convergence of PL}
Under \textbf{Assumption \ref{assumption1}}, the convergence upper bound of PL is established in
 the following. 
\begin{lemma}
\label{Lemma1}
Given the PL rate $\eta_{\mathrm{L}}$ and the weighting coefficient $\lambda$, under \textbf{Assumption \ref{assumption1}}, the expected difference between the personalized model $\overset{\sim}{\boldsymbol{\varpi}}_n^{t+1}$ and the optimum $\boldsymbol{\varpi}_n^{\ast}$ 
at the end of the $t$-th communication round is upper-bounded by
       \begin{equation} 
	\begin{split}
        \label{PerOneStep2_l}
&\mathbb{E}\left[\parallel\overset{\sim}{\boldsymbol{\varpi}}_{n}^{t+1}-\boldsymbol{\varpi}_{n}^{\ast}\parallel^{2}\right]
\leq \varepsilon_{\mathrm{L}}\mathbb{E}\left[\parallel\overset{\sim}{\boldsymbol{\varpi}}_{n}^{t}-\boldsymbol{\varpi}_{n}^{\ast}\parallel^{2}\right]+\\
&\left(\eta_{\mathrm{L}}^{2}+\eta_{\mathrm{L}}^{2}\lambda^{2}\right)G+\frac{4\eta_{\mathrm{L}}^{2}\lambda^{2}\!\!+\!\!2\eta_{\mathrm{L}}\lambda^{2}}{\mu}\mathbb{E}\!\!\left[F\!\left(\overset{\sim}{\boldsymbol{\omega}}^{t}\right)\!\!-\!\!F\!\left(\boldsymbol{\omega}^{\ast}\right)\right]
        \,,
	\end{split}
	\end{equation}
where {$G \triangleq \left(\!\left(1\!-\!\frac{\lambda}{2}\right)G_{0}\!+\!\lambda\left(\frac{G_{0}}{\mu}\!+\!M\right)\!\right)^{2}$,} and $\varepsilon_{\mathrm{L}}\!=\!1\!-\!\eta_{\mathrm{L}}\!\left(\!\left(1\!-\!\frac{\lambda}{2}\right)\!\mu\!+\!\lambda\right)\!+\!\eta_{\mathrm{L}}$. Clearly, $\varepsilon_{\mathrm{L}}$ increases with $\lambda$ when $\mu > 2$. 
\end{lemma}

\begin{proof}    See \textbf{Appendix \ref{pLemma1}}.
\end{proof}

\vspace{2 mm}
\textbf{Lemma \ref{Lemma1}} can be verified in two special cases. 
One is that PFL is only executed locally ($\lambda=0$). By \textbf{Lemma \ref{Lemma1}}, we have
         \begin{equation} 
        \label{peronestep_lambda0}
        \mathbb{E}\left[\parallel\overset{\sim}{\boldsymbol{\varpi}}_{n}^{t+1}-\boldsymbol{\varpi}_{n}^{\ast}\parallel^{2}\right]
        \leq \varepsilon_{\mathrm{L}}\mathbb{E}\left[\parallel\overset{\sim}{\boldsymbol{\varpi}}_{n}^{t}-\boldsymbol{\varpi}_{n}^{\ast}\parallel^{2}\right]+\eta_{\mathrm{L}}^{2}G
        \,,
	\end{equation}
which depends only on the FL local training. The obtained personalized models are obviously unaffected by the DP noise.

Another special case involves no PL, i.e., $\lambda=2$. Then, $\boldsymbol{\varpi}_{n}^{\ast}=\underset{\boldsymbol{\varpi}_{n}}{\arg\min}\,\parallel\boldsymbol{\varpi}_{n}-\boldsymbol{\omega}^{\ast}\parallel^{2}$, i.e., $\boldsymbol{\varpi}_{n}^{\ast}=\boldsymbol{\omega}^{\ast}$. Hence,
{\small
         \begin{subequations} 
         \begin{align}
        \label{peronestep_lambda2_1}
        \begin{split}
       \mathbb{E} &\bigg[\parallel\overset{\sim}{\boldsymbol{\varpi}}_{n}^{t+1}-\boldsymbol{\omega}^{\ast}\parallel^{2}\bigg]
        \leq  \varepsilon_{\mathrm{L}}\mathbb{E}\left[\parallel\overset{\sim}{\boldsymbol{\varpi}}_{n}^{t}-\boldsymbol{\omega}^{\ast}\parallel^{2}\right]+5\eta_{\mathrm{L}}^{2}G\\
        &+\frac{16\eta_{\mathrm{L}}^{2}+8\eta_{\mathrm{L}}}{\mu}\mathbb{E}\left[F\left(\overset{\sim}{\boldsymbol{\omega}}^{t}\right)-F\left(\boldsymbol{\omega}^{\ast}\right)\right] 
        \end{split} \\
        \label{peronestep_lambda2_3}
        \begin{split}
        &\leq \!
        (\varepsilon_{\mathrm{L}}\!\!+\!\!5\eta_{\mathrm{L}}^{2}) (\frac{G_{0}}{\mu}\!+\!M)^{2}\!\!+\!\!\frac{16\eta_{\mathrm{L}}^{2}\!\!+\!\!8\eta_{\mathrm{L}}}{\mu}\mathbb{E}\!\left[\!F\!\left(\overset{\sim}{\boldsymbol{\omega}}^{t}\right)\!\!-\!\!F\!\left(\boldsymbol{\omega}^{\ast}\right)\!\right],
        \end{split}       
        \end{align}
	\end{subequations} %
 }%
where (\ref{peronestep_lambda2_3}) is obtained by substituting (\ref{E2}) into (\ref{peronestep_lambda2_1}). 
As revealed in (\ref{peronestep_lambda2_3}), the convergence of PL depends only on the FL and hence the DP in PFL.

\begin{lemma}
\label{lemma2}
Given $F\left(\boldsymbol{\omega}\right)-F\left(\boldsymbol{\omega}^{\ast}\right)\leq\frac{1}{2\mu}\parallel\nabla F\left(\boldsymbol{\omega}\right)\parallel^{2}$ under \textbf{Assumption 1}, the expectation of the difference between the FL model in the $t$-th communication round, i.e., $\overset{\sim}{\boldsymbol{\omega}}^{t+1}$, and the optimal global model $\boldsymbol{\omega}^{\ast}$ is upper bounded by 
    {\small
    \begin{equation} 
        \label{gloOneStep}
\mathbb{E}\left[F(\overset{\sim}{\boldsymbol{\omega}}^{t+1})-F\left(\boldsymbol{\omega}^{\ast}\right)\right]\leq\!{\varepsilon_{\mathrm{G}}}\mathbb{E}\left[F(\overset{\sim}{\boldsymbol{\omega}}^{t})\!-\!F\left(\boldsymbol{\omega}^{\ast}\right)\right]\!+\!
\varphi_{\mathrm{L}} {T},
	\end{equation}
 }%
 {where $\varphi_{\mathrm{L}}=\sigma_z^2\frac{dL}{2TN^2}=\frac{\Delta s^{2}Ld\ln\left(1/\delta\right)}{N^2\epsilon^{2}}$ based on~\eqref{noiseSigma1}.}
\end{lemma}

\begin{proof}
    See \textbf{Appendix \ref{plemma2}}.
\end{proof}

\begin{theorem}
\label{con_UpperBound}
Under \textbf{Assumption \ref{assumption1}},
the convergence upper bound of PL after $T$ aggregation rounds is given as follows:

When $\varepsilon_{\mathrm{L}} \neq \varepsilon_{\mathrm{G}}$,
{\small
         \begin{equation} 
	\begin{split}
        \label{theoConverT}
&\mathbb{E}\left[\parallel\overset{\sim}{\boldsymbol{\varpi}}_{n}^{T}-\boldsymbol{\varpi}_{n}^{\ast}\parallel^{2}\right]
\leq \varepsilon_{\mathrm{L}}^{T}\Psi_{2}+\left(1+\lambda^{2}\right)\eta_{\mathrm{L}}^{2}G\frac{\varepsilon_{\mathrm{L}}^{T}-1}{\varepsilon_{\mathrm{L}}-1}\\
&+\!\!\frac{(4\eta_{\mathrm{L}}^{2}\!\!+\!\!2\eta_{\mathrm{L}})\lambda^{2}}{\mu}\!\!\left[\!\frac{\varepsilon_{\mathrm{L}}^{T}\!\!-\!\!\varepsilon_{\mathrm{G}}^{T}}{\varepsilon_{\mathrm{L}}\!\!-\!\!\varepsilon_{\mathrm{G}}}\Psi_{1}\!\!+\!\!\left(\!\!\frac{\varepsilon_{\mathrm{L}}^{T\!-\!1}\!\!-\!\!\varepsilon_{\mathrm{G}}^{T\!-\!1}}{\frac{\varepsilon_{\mathrm{L}}}{\varepsilon_{\mathrm{G}}}-1}\!\!-\!\!\frac{\varepsilon_{\mathrm{L}}^{T\!-\!1}\!\!-\!\!1}{\varepsilon_{\mathrm{L}}\!\!-\!\!1}\!\!\right)\!\!\frac{\varphi_{\mathrm{L}}T}{\varepsilon_{\mathrm{G}}\!\!-\!\!1}\!\right];
	\end{split}
	\end{equation}
 }

When $\varepsilon_{\mathrm{L}} = \varepsilon_{\mathrm{G}}$,
{\small
         \begin{equation} 
	\begin{split}
        \label{theoConverT_LG}
&\mathbb{E}\left[\parallel\overset{\sim}{\boldsymbol{\varpi}}_{n}^{T}-\boldsymbol{\varpi}_{n}^{\ast}\parallel^{2}\right]
\leq \varepsilon_{\mathrm{L}}^{T}\Psi_{2}+\left(1+\lambda^{2}\right)\eta_{\mathrm{L}}^{2}G\frac{\varepsilon_{\mathrm{L}}^{T}-1}{\varepsilon_{\mathrm{L}}-1}\\
&+\!\frac{(4\eta_{\mathrm{L}}^{2}\!\!+\!\!2\eta_{\mathrm{L}})\lambda^{2}}{\mu}\!\!\left[\!T\varepsilon_{\mathrm{L}}^{T\!-\!1}\Psi_1 \!\!+\!\!\left(\!(T\!\!-\!\!1)\varepsilon_{\mathrm{L}}^{T\!-\!1}\!\!-\!\!\frac{\varepsilon_{\mathrm{L}}^{T\!-\!1}\!-\!1}{\varepsilon_{\mathrm{L}}\!\!-\!\!1}\right)\!\frac{{\varphi_{\mathrm{L}}}T}{\varepsilon_{\mathrm{L}}\!\!-\!\!1}\!\right],
	\end{split}
	\end{equation} %
 }%
where, for conciseness, $\Psi_{2}=\parallel\boldsymbol{\varpi}_{n}^{0}-\boldsymbol{\varpi}_{n}^{\ast}\parallel^{2}$.

\end{theorem}

\begin{proof}
    See \textbf{Appendix \ref{pcon_UpperBound}}.
\end{proof}

\vspace{2 mm}
A trade-off between convergence and privacy is revealed in {(\ref{theoConverT})  and (\ref{theoConverT_LG})}: 
As the DP noise variance $\sigma_z^2$ increases, the convergence upper bound of PL increases since the term resulting from the DP noise is positive {since $\varphi_{\mathrm{L}}>0$ and $\sum_{x=0}^{t-1}\sum_{y=0}^{t-1-x}\varepsilon_{\mathrm{L}}^{x}\varepsilon_{\mathrm{G}}^{y}>0$ in the bound.} 
Moreover, the convergence upper bound of PL depends on $\lambda$. A larger $\lambda$ leads to a more significant impact of the DP noise on the convergence of PL when $\mu\geq 2$.
The reason is that, when $\mu > 2$, $\varepsilon_{\mathrm{L}}=1+\eta_{\mathrm{L}}-\eta_{\mathrm{L}}\left[\mu+(1-\frac{\mu}{2})\lambda\right]$ increases with $\lambda$ (see \textbf{Lemma~{\ref{Lemma1}}}); when $\mu=2$, $\varepsilon_{\mathrm{L}}$ is independent of $\lambda$, but the impact of DP still grows with $\lambda$ due to the coefficient $\lambda^2$ in (\ref{PerConvergence}). 
{ Note that \textbf{Theorem 1} holds under DP noises with other distributions, since the DP noise has no impact on \textbf{Lemma~\ref{Lemma1}} while the RHS of \eqref{gloOneStep} in \textbf{Lemma~\ref{lemma2}} depends only on the mean and variance $\sigma_z^2$ of the DP noise.}

On the other hand, the impact of DP on the convergence of PL may not always intensify with $\lambda$ when $\mu <2$, because the sign of the first derivative of $\lambda^2 \left(\frac{\varepsilon_{\mathrm{L}}^{t}-\varepsilon_{\mathrm{G}}^{t}}{\frac{\varepsilon_{\mathrm{L}}}{\varepsilon_{\mathrm{G}}}-1}-\frac{\varepsilon_{\mathrm{L}}^{t}-1}{\varepsilon_{\mathrm{L}}-1}\right)$  { in \eqref{PerConvergence2} and $\lambda^2 \left( t \varepsilon_{\mathrm{L}}^t-\frac{\varepsilon_{\mathrm{L}}^{t}-1}{\varepsilon_{\mathrm{L}}-1}\right)$} in \eqref{PerConvergence2_LG} with respect to $\lambda$ 
depends on $\eta_{\mathrm{L}}$, $\varepsilon_{\mathrm{G}}$, and $t$ when $\mu <2$.

\subsection{Optimization for Convergence}

 Let $h(T,\lambda)$ denote the convergence upper bound of PL after the $T$ aggregation rounds of FL, given the privacy budget. 
 \subsubsection{When $\varepsilon_{\mathrm{L}} \neq \varepsilon_{\mathrm{G}}$} $h(T,\lambda)$ is given by the right-hand side (RHS) of (\ref{theoConverT}).
After reorganization, we have
{\small
\begin{subequations} 
          \label{2Devi}
	\begin{align}
        \label{2Devi1}
        \begin{split}
h(T,&\lambda)\!=\!\Big(\Psi_{2}\!\!+\!\!\frac{\beta\Psi_{1}}{\varepsilon_{\mathrm{L}}\!\!-\!\!\varepsilon_{\mathrm{G}}}\!\!-\!\!\eta_{\mathrm{L}}^{2}\!\left(1\!\!+\!\!\lambda^{2}\right)\!\frac{G}{1\!\!-\!\!\varepsilon_{\mathrm{L}}}\!\!-\!\!\frac{\beta\varphi_{\mathrm{L}}}{(1\!\!-\!\!\varepsilon_{\mathrm{L}})\!(\varepsilon_{\mathrm{L}}\!\!-\!\!\varepsilon_{\mathrm{G}})}T\!\Big)\varepsilon_{\mathrm{L}}^{T}\!\!+\!\!\\&\Big(-\frac{\beta\Psi_{1}}{\varepsilon_{\mathrm{L}}-\varepsilon_{\mathrm{G}}}+\frac{\beta\varphi_{\mathrm{L}}}{(1-\varepsilon_{\mathrm{G}})(\varepsilon_{\mathrm{L}}-\varepsilon_{\mathrm{G}})}T\Big)\varepsilon_{\mathrm{G}}^{T}+\\&\frac{\beta\varphi_{\mathrm{L}}}{(1-\varepsilon_{\mathrm{G}})(1-\varepsilon_{\mathrm{L}})}T+\eta_{\mathrm{L}}^{2}\left(1+\lambda^{2}\right)\frac{G}{1-\varepsilon_{\mathrm{L}}}
\end{split}
\\
=&(H_{1}\!+\!H_{2}T)\varepsilon_{\mathrm{L}}^{T}\!+\!(H_{3}\!+\!H_{4}T)\varepsilon_{\mathrm{G}}^{T}\!+\!H_{5}T\!+\!H_{6} \label{2Devi2} 
        \,,
	\end{align}
	\end{subequations}
where  $\beta=\frac{4\eta_{\mathrm{L}}^{2}\lambda^{2} +2\eta_{\mathrm{L}}\lambda^{2}}{\mu}$, $H_{1}=\Psi_{2}+\frac{\beta\Psi_{1}}{\varepsilon_{\mathrm{L}}-\varepsilon_{\mathrm{G}}}-\eta_{\mathrm{L}}^{2}\!\left(1\!\!+\!\!\lambda^{2}\right)\frac{G}{1-\varepsilon_{\mathrm{L}}}$, $H_{2}=-\frac{\beta\varphi_{\mathrm{L}}}{(1-\varepsilon_{\mathrm{L}})(\varepsilon_{\mathrm{L}}-\varepsilon_{\mathrm{G}})}$, $H_{3}=-\frac{\beta\Psi_{1}}{\varepsilon_{\mathrm{L}}-\varepsilon_{\mathrm{G}}}$, $H_{4}=\frac{\beta\varphi_{\mathrm{L}}}{(1-\varepsilon_{\mathrm{G}})(\varepsilon_{\mathrm{L}}-\varepsilon_{\mathrm{G}})}$, $H_{5}=\frac{\beta\varphi_{\mathrm{L}}}{(1-\varepsilon_{\mathrm{G}})(1-\varepsilon_{\mathrm{L}})}>0$, and $H_{6}=\eta_{\mathrm{L}}^{2}\!\left(1\!\!+\!\!\lambda^{2}\right)\frac{G}{1-\varepsilon_{\mathrm{L}}}>0$. 
}

Let $H_{0}$ denote the lower bound of $(H_{1}+H_{2}T)\varepsilon_{\mathrm{L}}^{T}+(H_{3}+H_{4}T)\varepsilon_{\mathrm{G}}^{T}+H_{6}$ in (\ref{2Devi2}).
Then, $h(T,\lambda)$ has a linear lower bound, denoted by $h_{\rm Low}(T)$; that is
\begin{equation}
\begin{aligned}
    h(T,\lambda)&\geq h_{\rm Low}(T)=H_{0}+H_5T, 
    \\
    \text{with} \,(H_{1}+H_{2}T)&\varepsilon_{\mathrm{L}}^{T}+(H_{3}+H_{4}T)\varepsilon_{\mathrm{G}}^{T}+H_{6}\geq H_0,
    \end{aligned}
\end{equation}

\begin{enumerate}
\item When $\varepsilon_{\mathrm{L}}>\varepsilon_{\mathrm{G}}$, $H_2<0$. 
Then, $H_2T\varepsilon_{\mathrm{L}}^T$ first decreases and then increases in $T$. There exists $\underset{T}{\min}(H_2T\varepsilon_{\mathrm{Low}}^T)$.  
Also, $H_3<0$. $H_3\varepsilon_{\mathrm{G}}^T \geq H_3$ since $H_3\varepsilon_{\mathrm{G}}^T$ is an increasing function of~$T$.  
$H_4>0$; thus, $H_4T\varepsilon_{\mathrm{G}}^T\geq 0$. 
Considering two possible cases concerning $H_1$, $H_{0}$ is given by 
\begin{enumerate}
    \item If  {$H_1\geq 0$}, then $H_{0}=\underset{T}{\min}(H_2T\varepsilon_{\mathrm{Low}}^T)+H_{3}+H_6$;
    \item If $H_1<0$, then $H_{0}=H_1+\underset{T}{\min}(H_2T\varepsilon_{\mathrm{Low}}^T)+H_{3}+H_6$.
\end{enumerate}
\item When $\varepsilon_{\mathrm{L}}<\varepsilon_{\mathrm{G}}$, $H_2>0$. Hence, $H_2T\varepsilon_{\mathrm{L}}^T \geq 0$. Also, $H_3>0$ and hence $H_3\varepsilon_{\mathrm{G}}^T>0$. Moreover, $H_4<0$ and thus $H_4T\varepsilon_{\mathrm{G}}^T$ first decreases and then increases with respect to $T$, with the existence of $\underset{T}{\min}(H_4T\varepsilon_{\mathrm{Low}}^T)$ confirmed. 
Considering two possible cases concerning $H_1$, $H_{0}$ can be given by
\begin{enumerate}
    \item If  {$H_1\geq 0$}, then $H_{0}=\underset{T}{\min}(H_4T\varepsilon_{\mathrm{Low}}^T)+H_6$;
    \item If $H_1<0$, then $H_{0}=H_1+\underset{T}{\min}(H_4T\varepsilon_{\mathrm{Low}}^T)+H_6$.
\end{enumerate}
\end{enumerate}
As illustrated in Fig.~\ref{fig:lower_bound_1}, given $\lambda$, the optimal $T$, denoted by $T^*$, which minimizes $h(T,\lambda)$, is within $(0,T')$, where $T'$ satisfies $h_{\rm Low}(T')=h(0,\lambda)$, i.e., $T'=\frac{H_{1}+H_{3}+H_6-H_{0}}{H_{5}}$. This is because $h(T,\lambda)>h_{\rm Low}(T')=h(0,\lambda)\geq h(T^*,\lambda),\,\forall T>T'$. As a result, $T^*$ can be obtained effectively using a one-dimensional search with a step size of 1 within~$(0,T')$.  

 \subsubsection{When $\varepsilon_{\mathrm{L}} = \varepsilon_{\mathrm{G}}$} $h(T,\lambda)$ is given by the RHS of (\ref{theoConverT_LG}):
\begin{subequations} 
          \label{2Devi_LG}
	\begin{align}
        \label{2Devi1_LG}
        \begin{split}
h(T,&\lambda)=\Big(-\frac{\left(1+\lambda^{2}\right)\eta_{\mathrm{L}}^{2}G}{1-\varepsilon_{\mathrm{L}}}+\Psi_{2}+\frac{(4\eta_{\mathrm{L}}^{2}+2\eta_{\mathrm{L}})\lambda^{2}}{\varepsilon_{\mathrm{L}}\mu}\cdot\\
&\frac{\Psi_{1}(1-\varepsilon_{\mathrm{L}})^{2}-\varepsilon_{\mathrm{L}}\varphi_{\mathrm{L}}}{(1-\varepsilon_{\mathrm{L}})^{2}}T-\frac{(4\eta_{\mathrm{L}}^{2}+2\eta_{\mathrm{L}})\lambda^{2}\varphi_{\mathrm{L}}}{\varepsilon_{\mathrm{L}}\mu(1-\varepsilon_{\mathrm{L}})}T^{2}\Big)\varepsilon_{\mathrm{L}}^{T}\\
&-\frac{\left(1+\lambda^{2}\right)\eta_{\mathrm{L}}^{2}G}{1-\varepsilon_{\mathrm{L}}}-\frac{(4\eta_{\mathrm{L}}^{2}+2\eta_{\mathrm{L}})\lambda^{2}}{\mu(1-\varepsilon_{\mathrm{L}})}
\end{split}
\\
&=\left(\mathcal{H}_1+\mathcal{H}_2T+\mathcal{H}_3T^{2}\right)\varepsilon_{\mathrm{L}}^{T}+\mathcal{H}_4
\label{2Devi2_LG} 
        \,,
	\end{align}
	\end{subequations}
 where, for the brevity of notation, $\mathcal{H}_{1}=-\frac{\left(1+\lambda^{2}\right)\eta_{\mathrm{L}}^{2}G}{1-\varepsilon_{\mathrm{L}}}+\Psi_{2}$, $\mathcal{H}_2=\frac{(4\eta_{\mathrm{L}}^{2}+2\eta_{\mathrm{L}})\lambda^{2}}{\varepsilon_{\mathrm{L}}\mu}(\Psi_{1}-\frac{\varepsilon_{\mathrm{L}}\varphi_{\mathrm{L}}}{(1-\varepsilon_{\mathrm{L}})^{2}})$, $\mathcal{H}_{3}=-\frac{(4\eta_{\mathrm{L}}^{2}+2\eta_{\mathrm{L}})\lambda^{2}\varphi_{\mathrm{L}}}{\varepsilon_{\mathrm{L}}\mu(1-\varepsilon_{\mathrm{L}})}<0$, and $\mathcal{H}_{4}=-\frac{\left(1+\lambda^{2}\right)\eta_{\mathrm{L}}^{2}G}{1-\varepsilon_{\mathrm{L}}}-\frac{(4\eta_{\mathrm{L}}^{2}+2\eta_{\mathrm{L}})\lambda^{2}}{\mu(1-\varepsilon_{\mathrm{L}})}<0$.

Let 
$\mathcal{H}_0$ 
denote the lower bound of $\left(\mathcal{H}_1+\mathcal{H}_2T\right)\varepsilon_{\mathrm{L}}^{T}+\mathcal{H}_4$ in (\ref{2Devi2_LG}). Then, the lower bound of $h(T,\lambda)$ is given by
\begin{equation}
    h(T,\lambda)\geq h_{\rm Low}(T)=\mathcal{H}_0+\mathcal{H}_3T^{2}\varepsilon_{\mathrm{L}}^{T},
\end{equation}
Since $\mathcal{H}_3<0$, $\mathcal{H}_3T^{2}\varepsilon_{\mathrm{L}}^{T}$ first decreases and then increases.

\begin{enumerate}
\item When $\mathcal{H}_2<0$, $\mathcal{H}_2T\varepsilon_{\mathrm{L}}^T$ first decreases and then increases with respect to $T$. Then, $\underset{T}{\min}(\mathcal{H}_2T\varepsilon_{\mathrm{Low}}^T)$ exists. Considering two possible cases concerning $\mathcal{H}_1$, $\mathcal{H}_0$ is given by
\begin{enumerate}
\item If $\mathcal{H}_1 \geq 0$, then $\mathcal{H}_0=\underset{T}{\min}(\mathcal{H}_2T\varepsilon_{\mathrm{Low}}^T)+\mathcal{H}_4$;
\item If $\mathcal{H}_1 < 0$, then $\mathcal{H}_0=\mathcal{H}_1+\underset{T}{\min}(\mathcal{H}_2T\varepsilon_{\mathrm{Low}}^T)+\mathcal{H}_4$.
\end{enumerate}
\item When $\mathcal{H}_2 \geq 0$, $\mathcal{H}_2T\varepsilon_{\mathrm{L}}^T \geq 0$. Considering the two possible cases concerning $\mathcal{H}_1$, $\mathcal{H}_0$ is given by
\begin{enumerate}
\item If $\mathcal{H}_1 \geq 0$, then $\mathcal{H}_0=\mathcal{H}_4$;
\item If $\mathcal{H}_1 < 0$, then $\mathcal{H}_0=\mathcal{H}_1+\mathcal{H}_4$.
\end{enumerate}
\end{enumerate}
Similarly, the optimal $T^*$ under a given $\lambda$ can be obtained using a one-dimensional search with a step size of 1 within $(0,T'')$, where $T''$ satisfies $h_{\rm Low}(T'')=h(0,\lambda)$.

In both cases, clearly, the convergence of PL depends heavily on that of FL and, in turn, the DP in PFL.

\begin{figure}[!t]
\centering
        \includegraphics[width=0.4\textwidth]{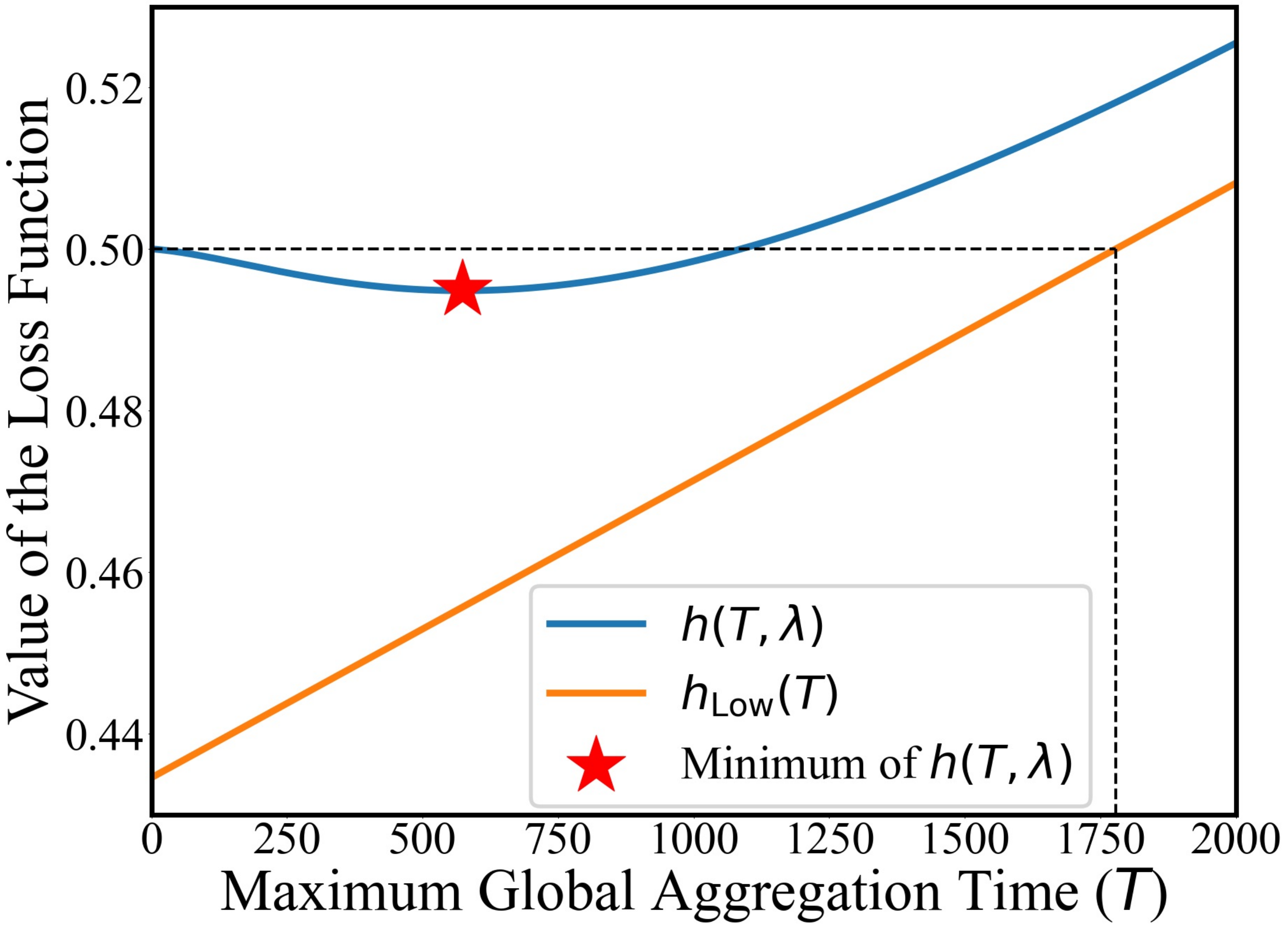}
 	\caption{Comparison between $h(T,\lambda)$ and its lower bound {$h_{\rm Low}(T)$}, where the simulation parameters are as specified for the MLR model trained on the MNIST dataset, i.e., $\lambda=0.1$, $\epsilon=100$, and $\delta=0.01$; see Section VI-B.}
 	\label{fig:lower_bound_1}
 	\end{figure}

\section{Fairness Analysis of Privacy-Preserving PFL}

In this section, we analyze the fairness of PL under privacy constraints and uncover an opportunity to maximize the fairness by optimizing $\lambda$. 
{
We focus on the fairness of performance distribution to measure the degree of uniformity in performance across the clients. 
Several fairness metrics have been adopted in FL, including distance (e.g., cosine distance~\cite{heidarian2016hybrid} and Euclidean distance~\cite{aggarwal2001surprising}), variance~\cite{li2021ditto}, risk difference (e.g., demographic parity~\cite{dwork2012fairness} and equal opportunity~\cite{hardt2016equality}), and Jain's fairness index (JFI)~\cite{cho2020bandit}. Among these metrics, variance and JFI are suitable for measuring the fairness of performance distribution for FL/PFL. This is because distance metrics are suitable for the contribution fairness of participants, while risk difference metrics are adequate for group fairness to eliminate prejudice toward some specific groups. While variance and JFI are typically negatively correlated, variance is more sensitive to outliers. For this reason, we adopt variance as the metric of fairness to encourage uniformly distributed performance across the PL models.
The definition of fairness measured by variance is provided as follows.

\vspace{4 mm}
\begin{definition}[Fairness \cite{li2021ditto}]
\label{DefFairness}
For a group of personalized models $\{\mathbf{\boldsymbol{\varpi}}_{n}\}_{n\in\mathbb{N}}$, the performance distribution fairness is measured by 
\begin{equation}
    \label{DefEquaFair}    \varrho(\boldsymbol{\varpi}_{n})=\mathrm{var_{\mathit{N}}}\left[F_{n}\left(\boldsymbol{\varpi}_{n}\left(\lambda\right)\right)\right] \,,
\end{equation} 
where 
$\mathrm{var_{\mathit{N}}}\left[F_{n}\left(\boldsymbol{\varpi}_{n}\left(\lambda\right)\right)\right]$ is the variance across the local training losses of the personalized models of all clients. A set of models $\{\mathbf{\boldsymbol{\varpi}}'_{n}\}_{n\in\mathbb{N}}$ is fairer than another set $\{\mathbf{\boldsymbol{\varpi}}_{n}\}_{n\in\mathbb{N}}$ if {$\varrho'=\mathrm{var_{\mathit{N}}}\{F_n({\boldsymbol{\varpi}}'_{n})\}_{n\in\mathbb{N}}<\varrho$. }
The optimal $\lambda$, denoted by $\lambda^{\ast}$, is defined as
\begin{equation}
    \label{DefOptimalLambda}
    \lambda^{\ast}=\underset{\lambda}{\arg\min}\,\mathbb{E}\left\{ \varrho(\boldsymbol{\varpi}_{n}^{\ast}(\lambda))\right\} ,
\end{equation}
where, given~$\lambda$, $\boldsymbol{\varpi}_{n}^{\ast}\left(\lambda\right)$ is client $n$'s optimal PL model.
\end{definition}}

\subsection{Personalized Bayesian Linear Regression With DP}
It is generally different to analyze the performance fairness of PFL, since the performance, e.g., loss, of PFL is typically analytically unavailable and can only be empirically obtained a-posteriori after PFL training.
With reference to Ditto~\cite{li2021ditto},
we consider that each client trains a personalized Bayesian linear regression model to shed useful light on the fairness of general ML models.
Bayesian linear regression models treat regression coefficients and the disturbance variance as random variables 
\cite{seber2003linear}. We set the optimal FL global model $\boldsymbol{\omega}^{\ast}$ as the non-information prior on $\mathbb{R}^{d}$, i.e., uniformly distributed on $\mathbb{R}^{d}$. 

Suppose the optimal FL local model $\boldsymbol{u_n}^{\ast}$ of client $n$ is distributed around the optimal FL global model $\boldsymbol{\omega}^{\ast}$:
\begin{equation}
    \label{u_n}
    \boldsymbol{u}_{n}^{\ast}=\boldsymbol{\omega}^{\ast}+\boldsymbol{\tau}_{n}  \,,
\end{equation}
where $\boldsymbol{\tau}_{n}\sim \mathcal{N}(0,\zeta^{2}\mathbf{I}_{d})$, $\forall n$ are i.i.d. random variables, and $\mathbf{I}_{d}$ is the $d\times d$ identity matrix. 
The local data of client $n$ satisfies 
\begin{equation}
    \label{localmodel}
    \mathbf{Y}_{n}=\mathbf{X}_{n}\boldsymbol{u}_{n}^{\ast}+\boldsymbol{\nu}_{n}  \,.
\end{equation}
Here, $\mathbf{X}_{n}\in\mathbb{R}^{b\times d}$ and $\mathbf{Y}_{n} \in \mathbb{R}^{b}$ denote $b$ samples of client $n$, and $\boldsymbol{\nu}_{n}\in\mathbb{R}^{b}$ under the assumption of {$\boldsymbol{\nu}_{n}\sim \mathcal{N}(0,\sigma^{2}\mathbf{I}_{b})$. }

Given the samples $(\mathbf{X}_{n}$, $\mathbf{Y}_{n})$, the loss function of the Bayesian linear regression problem at client $n$ is written as
\begin{equation}
    \label{loss}
    F_{n}\left(\boldsymbol{u}_{n}\right)=\frac{1}{b}\parallel \mathbf{X}_{n}\boldsymbol{u}_{n}-\mathbf{Y}_{n}\parallel^{2}.
\end{equation}
By minimizing (\ref{loss}), the estimate of $\boldsymbol{u}_{n}^{\ast}$ is~\cite[Eq.~(12)]{li2021ditto}
\begin{equation}
    \label{estimatedmodel}    \hat{\boldsymbol{u}}_{n}=\left(\mathbf{X}_{n}^{\intercal}\mathbf{X}_{n}\right)^{-1}\mathbf{X}_{n}^{\intercal}\mathbf{Y}_{n}  .
\end{equation}
By plugging (\ref{loss}) into (\ref{fn_1}), the optimal global model is
\begin{subequations}
    \label{g_loss}
\begin{align}
\boldsymbol{\omega}^{\ast}&=\underset{\boldsymbol{\omega}}{\arg\min}\, \frac{1}{Nb} \sum_{n=1}^{N} \parallel \mathbf{X}_{n}\boldsymbol{\omega}-\mathbf{Y}_{n}\parallel^{2} \\
&=\sum_{n=1}^{N}\left(\mathbf{X}^{\intercal}\mathbf{X}\right)^{-1}\mathbf{X}_{n}^{\intercal}\mathbf{X}_{n}\hat{\boldsymbol{u}}_n ,    
\end{align}
\end{subequations}
where $\mathbf{X}\in\mathbb{R}^{Nb\times d}$ collects $Nb$ samples of all clients, i.e., $\mathbf{X}=(\mathbf{X}_{1}^{\top}\: \mathbf{X}_{2}^{\top}\: \cdots \: \mathbf{X}_{N}^{\top} )^{\top}$.

With DP, the server receives the perturbed version of 
the estimated local model, denoted by $\overset{\sim}{\boldsymbol{u}}_{n}$, as given by 
\begin{equation}
    \label{noisyestimation1}
    \overset{\sim}{\boldsymbol{u}}_{n}=\hat{\boldsymbol{u}}_n+\mathbf{z}_{n}  = \left(\mathbf{X}_{n}^{\intercal}\mathbf{X}_{n}\right)^{-1}\mathbf{X}_{n}^{\intercal}\mathbf{Y}_{n}+\mathbf{z}_{n} \,.
\end{equation} 
By replacing $\hat{\boldsymbol{u}}_n$ with $\overset{\sim}{\boldsymbol{u}}_n$, the DP perturbed version of the optimal global model, denoted by $\overset{\sim}{\boldsymbol{\omega}}^{\ast}$, is given by
	\begin{equation} 
     \overset{\sim}{\boldsymbol{\omega}}^{\ast}=\sum_{n=1}^{N} \left(\mathbf{X}^{\intercal}\mathbf{X}\right)^{-1}\mathbf{X}_{n}^{\intercal}\mathbf{X}_{n}\overset{\sim}{\boldsymbol{u}}_n,
     \label{omega*2}
	\end{equation}

\begin{lemma}
\label{noisy_pL}
With $\mathbf{X}_{n}^{\intercal}\mathbf{X}_{n}=\rho\mathbf{I}_{d}$, $\forall n \in \mathbb{N}$ under independently sampled or generated data samples \cite{li2021ditto}, 
the optimal PL model with DP, $\overset{\sim}{\boldsymbol{\varpi}_{n}^{\ast}}\left(\lambda\right)$, and the optimal FL local model, $\boldsymbol{u}_n^{\ast}$ in (\ref{u_n}), can be written as
	\begin{equation} 
	\begin{split}
        \label{Optimal_w_n}
     \overset{\sim}{\boldsymbol{\varpi}}_{n}^{\ast}\left(\lambda\right)=&\frac{b}{\left(2-\lambda\right)\rho+b\lambda}\left(\left(\frac{\left({2-\lambda}\right)\rho}{b}+\frac{\lambda}{N}\right)\hat{\boldsymbol{u}}_{n} \right.\\
     &\left.+\frac{\lambda}{N}\underset{m\in \mathbb{N},m\neq n}{\sum}\hat{\boldsymbol{u}}_{m}+\frac{\lambda}{N} \sum_{n=1}^{N} {\mathbf{z}_{n}}\right) ;
	\end{split}
	\end{equation} 
	\begin{equation} 
        \label{Optimal_u_n}
     \boldsymbol{u}_{n}^{\ast}=\frac{\sigma_{w}^{2}\rho}{\sigma^{2}}\hat{\boldsymbol{u}}_{n}+ 
     \frac{\sigma_{w}^{2}\rho}{\sigma^{2}+N\zeta^{2}\rho}\underset{m\in\mathbb{N},m\neq n}{\sum}\hat{\boldsymbol{u}}_{m}+\boldsymbol{\vartheta}_{n} , 
	\end{equation}
where $\boldsymbol{\vartheta}_{n}\sim \mathcal{N}(0,\sigma_{w}^{2} \mathbf{I}_{d})$ and $\sigma_{w}^{2}=\left(\frac{N-1}{\frac{\sigma^{2}}{\rho}+N\zeta^{2}}+\frac{\rho}{\sigma^{2}}\right)^{-1}$.
\end{lemma}

\begin{proof}
    See \textbf{Appendix \ref{BLR_models}}.
\end{proof}

We note that in Ditto, the optimal personalized models ${\boldsymbol{\varpi}_{n}^{\ast}}\left(\lambda\right)$, $\forall n \in \mathbb{N}$ are deterministic, provided the estimated local models $\hat{\boldsymbol{u}}_{n}$ ($\forall n \in \mathbb{N}$) and $\lambda$.
{
By contrast, in DP-Ditto, the PL models $\overset{\sim}{\boldsymbol{\varpi}}_{n}^{\ast}\left(\lambda\right)$, $\forall n \in \mathbb{N}$ are not deterministic. This is because the DP noise, $\mathbf{z}_{n}$, is added on the estimated local models $\hat{\boldsymbol{u}}_n$ for the FL global model updating and, consequently, it is coupled with $\lambda$ in $\overset{\sim}{\boldsymbol{\varpi}}_{n}^{\ast}\left(\lambda\right)$; see (\ref{Optimal_w_n}). Considering the difference between the optimal FL local model $\boldsymbol{u}_{n}^{\ast}$ in (\ref{Optimal_u_n}) and its estimate $\hat{\boldsymbol{u}}_{n}$ is random and captured by $\boldsymbol{\vartheta}_{n}$, we have to analyze the joint distribution of the two random variables $\mathbf{z}_{n}$ in (\ref{Optimal_w_n}) and $\boldsymbol{\vartheta}_{n}$ {in (\ref{Optimal_u_n})}, $\forall n \in \mathbb{N}$, and obtain the fairness expression with respect to $\lambda$.
}

Given the optimal FL local model without DP, $\boldsymbol{u}_n^{\ast}$, and the optimal PL model with DP, $\overset{\sim}{\boldsymbol{\varpi}}_n^{\ast}(\lambda)$, the fairness $\varrho(\overset{\sim}{\boldsymbol{\varpi}_{n}^{\ast}}(\lambda))$ of the PL models among all clients is established, as follows.

\begin{theorem}
\label{FAIRNESS}
Given $\lambda$ and the variance $\sigma_{z}^{2}$ of the DP noise, the fairness of the personalized Bayesian linear regression model, $R(\lambda)$, can be measured by
\begin{equation} 
	\begin{split}
        \label{fairness_R}
        \!\!\!\!R(\lambda)&\!=   \! 2d\!\left[\sigma_{w}^{2}\!+\!\left[\alpha_{0}\left(\lambda\right)\right]^{2}\frac{\sigma_{z}^{2}}{N^{2}}\right]\!\!+\!4\!\left[\sigma_{w}^{2}\!+\!\left[\alpha_{0}\left(\lambda\right)\right]^{2}\!\frac{\sigma_{z}^{2}}{N^{2}}\right] \!\times\! \\
        &\left(S_{1}\!-\!S_{2}\alpha_{0}\left(\lambda\right)\right)^{2}G_{1}\!+\!\left[S_{1}\!-\!S_{2}\alpha_{0}\left(\lambda\right)\right]^{4}\left(G_{2}\!-\!G_{1}^{2}\right),
	\end{split}
	\end{equation} 
where $\alpha_{0}\left(\lambda\right)=\frac{b\lambda}{\left(2-\lambda\right)\rho+b\lambda}$, $S_{1}=\frac{\sigma^{2}}{N\left(\sigma^{2}+\rho\zeta^{2}\right)}$, $S_{2}=\frac{1}{N}$, $G_{1}= \sum_{l=1}^{d} \frac{1}{N}\sum_{n=1}^{N}\left[\alpha_{nl}^{2}\right]$, and $G_{2}=\frac{1}{N}\sum_{n=1}^{N}\left[\left( \sum_{l=1}^{d} \alpha_{nl}^{2}\right)^{2}\right]$. 
\end{theorem}

\begin{proof}
    See \textbf{Appendix \ref{pFAIRNESS}}.
\end{proof}


According to \textbf{Theorem \ref{FAIRNESS}}, DP degrades the fairness of PL, as $R(\lambda)$ increases with $\sigma_z^2$. On the other hand, the dependence of fairness on $\lambda$ is much more complex{, which is different from Ditto}. As will be revealed later, given $\sigma_z^2$, fairness depends on not only $\lambda$ but also the model clipping threshold $C$. The uniqueness of the optimal $\lambda$, i.e., $\lambda^*$, can be ascertained when~$C$ is sufficiently small.
{ Note that \textbf{Theorem~\ref{FAIRNESS}} holds under DP noises with other distributions, 
since the fairness $R(\lambda)$ depends only on the mean and variance 
of the DP noises, according to \textbf{Definition \ref{DefFairness}}.}

\subsection{Convergence-Privacy-Fairness Trade-off}

We analyze the existence of the optimal $\lambda^{\ast}$ and $T^*$ to balance the trade-off between the convergence, privacy, and fairness of DP-Ditto.
 For conciseness, we rewrite $\alpha_{0}\left(\lambda\right)$ and $R(\lambda)$ as $\alpha_{0}$ and $R$, respectively. 
\begin{theorem}
\label{Op_lambda}
    Given the DP noise variance $\sigma_{z}^{2}$, the optimal $\lambda^{\ast}$, which maximizes $R(\lambda)$, exists and is unique when the model clipping threshold $C<\frac{\sqrt{d}}{2NS_1}$. $\lambda^{\ast} \in [0,2]$ satisfies
	\begin{align}        
        &4d\frac{\sigma_{z}^{2}}{N^{2}}\alpha_{0}^{\ast}\!\!+\!\!8G_{1}\!\frac{\sigma_{z}^{2}}{N^{2}}\!\!\left(S_{1}\!\!-\!\!S_{2}\alpha_{0}^{\ast}\right)^{2}\!\alpha_{0}^{\ast} \!\!-\!\!8S_{2}G_{1}\!\!\left(\!\sigma_{w}^{2}\!\!+\!\!\left[\alpha_{0}^{\ast}\right]^{2}\!\!\frac{\sigma_{z}^{2}}{N^{2}}\!\right)\!\cdot  \nonumber \\ 
        &\left(S_{1}-S_{2}\alpha_{0}^{\ast}\right)-4S_{2}\left(G_{2}-G_{1}^{2}\right)\left(S_{1}-S_{2}\alpha_{0}^{\ast}\right)^{3}=0,\label{opAlpha0}
	\end{align}
where $\alpha_{0}^{\ast}=\frac{b\lambda^{\ast}}{\left(2-\lambda^{\ast}\right)\rho+b\lambda^{\ast}}$.

\end{theorem}
\begin{proof}
    See \textbf{Appendix \ref{pOp_lambda}}
\end{proof}

With the privacy consideration, we jointly optimize $\lambda$ and $T$ to improve the trade-off between the convergence, privacy, and fairness of DP-Ditto. 
From (\ref{2Devi}) and (\ref{opAlpha0}), $\lambda^\ast$ and $T^\ast$ satisfy
	\begin{equation} 
        \label{joint_op}
     \underset{\lambda,T}{\min}\, h(T,\lambda) ,\quad
     s.t.~\text{(\ref{opAlpha0})} ,
	\end{equation}
which can be solved through an iterative search.
For $T$, a one-dimensional search can be carried out. 
Given the aggregation round number $T$, \eqref{opAlpha0} can be solved analytically, e.g., using the Cardano method~\cite{shmakov2011universal}. The optimal $\lambda^*$ depends on~$\sigma_{z}^{2}$.  

\begin{corollary}
\label{o_fairness}
The optimal $\lambda^{\ast}$, which minimizes the fairness measure $R(\lambda)$, decreases as the DP noise variance $\sigma_z^2$ increases (i.e., the privacy budget $\epsilon$ decreases).
\end{corollary}
\begin{proof} 
    See \textbf{Appendix \ref{po_fairness}}.
\end{proof}


For a given $T$, up to three feasible solutions to $\lambda$ can be obtained by solving \eqref{opAlpha0}, as \eqref{opAlpha0} is a three-order polynomial equation. 
As revealed in \textbf{Theorem~\ref{Op_lambda}}, one of the three solutions is within $[0,2]$.
By comparing $h(T,\lambda)$ among all the obtained $(T,\lambda)$ pairs, the optimal $(T^*,\lambda^*)$ can be achieved and the existence of $(T^*,\lambda^*)$ is guaranteed. 
{The complexity of this iterative search is determined by the one-dimensional search for $T$ and the Cardano method for solving \eqref{opAlpha0} under each given $T$. The worst-case complexity of the one-dimensional search with a step size of 1 is $\mathcal{O}(T_{\max})$, where $T_{\max}$ is the maximum number of communication rounds permitted. Being an analytical method, the Cardano method provides closed-form solutions and incurs a complexity of $\mathcal{O}(1)$~\cite{artin2011algebra}.
As a result, the overall complexity of the iterative search is $\mathcal{O}(T_{\max})$.}

In a more general case, the ML model is not linear, and $\lambda$ cannot be analytically solved since there is no explicit analytical expression of $\lambda$. Different $\lambda$ values can be tested. Per $\lambda$, the corresponding optimal $T$ can be obtained via a one-dimensional search. Given the optimal $T$, the corresponding optimal $\lambda$ can be obtained by testing different $\lambda$ values (e.g., one-dimensional search for $\lambda$). We can restart the search for the optimal $T$ corresponding to the optimal $\lambda$, so on and so forth, until convergence (i.e., the optimal $T^*$ and $\lambda^*$ stop changing), as done experimentally in Section VI.

\section{Experiments and Results}
In this section, we assess the trade-off between the convergence, accuracy, and fairness of DP-Ditto experimentally. The impact of privacy considerations on those aspects of DP-Ditto is discussed.
We set $N=20$ clients by default. The clipping threshold is $C=20$ and the privacy budget is $\delta=0.01$~\cite{wei2020federated}. We consider three network models, i.e., MLR, DNN, and CNN.
\begin{itemize}
\item \textbf{MLR:} This classification method generalizes logistic regression to multiclass problems. It constructs a linear predictor function to predict the probability of an outcome based on an input observation.
\item \textbf{DNN:} This model consists of an input layer, a fully connected hidden layer (with 100 neurons), and an output layer. The rectified linear unit (ReLU) activation function is applied to the hidden layer.
\item \textbf{CNN:} This model contains two convolutional layers with 32 and 64 convolutional filters per layer, and a pooling layer between the two convolutional layers to prevent over-fitting. Following the convolutional layers are two fully connected layers. We use the ReLU in the convolutional and fully connected layers. 
\end{itemize} 
The learning rates of FL and PL are $\eta_{\mathrm{G}}=0.005$ and $\eta_{\mathrm{L}}=0.005$, respectively.

We consider {four} widely used public datasets, i.e., MNIST, Fashion-MNIST (FMNIST){, and CIFAR10}.
Cross-entropy loss is considered for the datasets.  
Apart from Ditto~\cite{li2021ditto}, the following benchmarks are considered:
\begin{itemize} 
    \item 
    \textbf{pFedMe~\cite{t2020personalized}:} The global FL model is updated in the same way as the typical FL. Learning from the global model, each personalized model is updated based on a regularized loss function using the Moreau envelope. 
    \item
    \textbf{APPLE~\cite{luo2022adapt}:} Each client uploads to the server a core model learned from its personalized model and downloads the other clients' core models in each round. The personalized model is obtained by locally aggregating the core models with learnable weights.
    \item 
    \textbf{FedAMP~\cite{huang2021personalized}:} The server has a personalized cloud model. Each client has a local personalized model. In each round, the server updates the personalized cloud models using an attention-inducing function of the uploaded local models and combination weights. Upon receiving the cloud model, each client locally updates its personalized model based on a regularized loss function.
    \item 
    \textbf{FedALA~\cite{zhang2023fedala}:} In every round of FedALA, each client adaptively initializes its local model by aggregating the downloaded global model and the old local model with learned aggregation weights before local training.

\end{itemize}

\subsubsection{Comparison With the State of the Art}

\begin{figure}[t]
\centering
\subfigure[{Accuracy vs. $T$ (DNN, MNIST)}]{
\label{benchmarks_accuracy_mnist_dnn-minmax_test_1}
\includegraphics[width=0.23\textwidth]{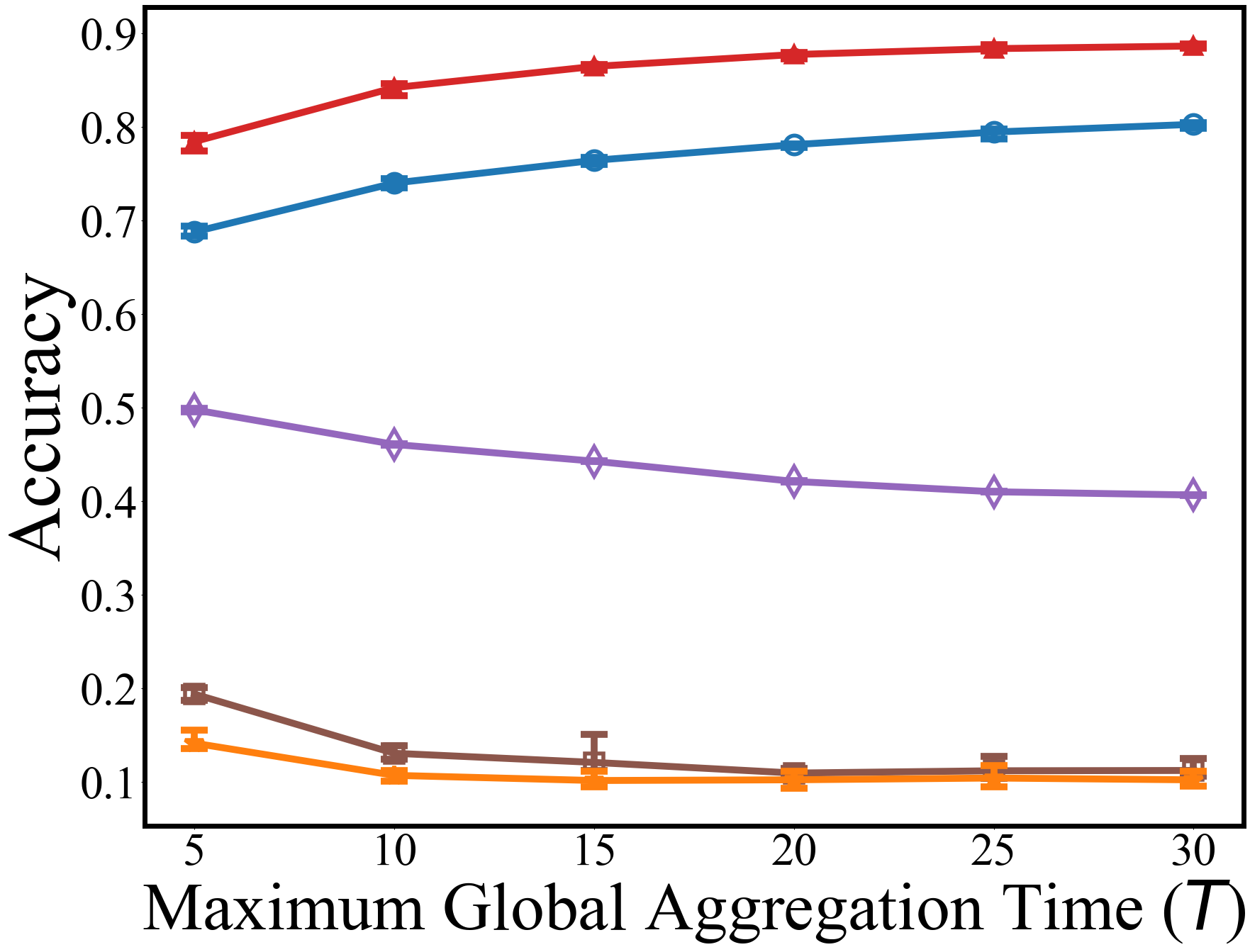}}
\hspace{-1mm}
\subfigure[{Fairness vs. $T$ (DNN, MNIST)}]{
\label{benchmarks_fairness_mnist_dnn-minmax_test_1}
\includegraphics[width=0.23\textwidth]{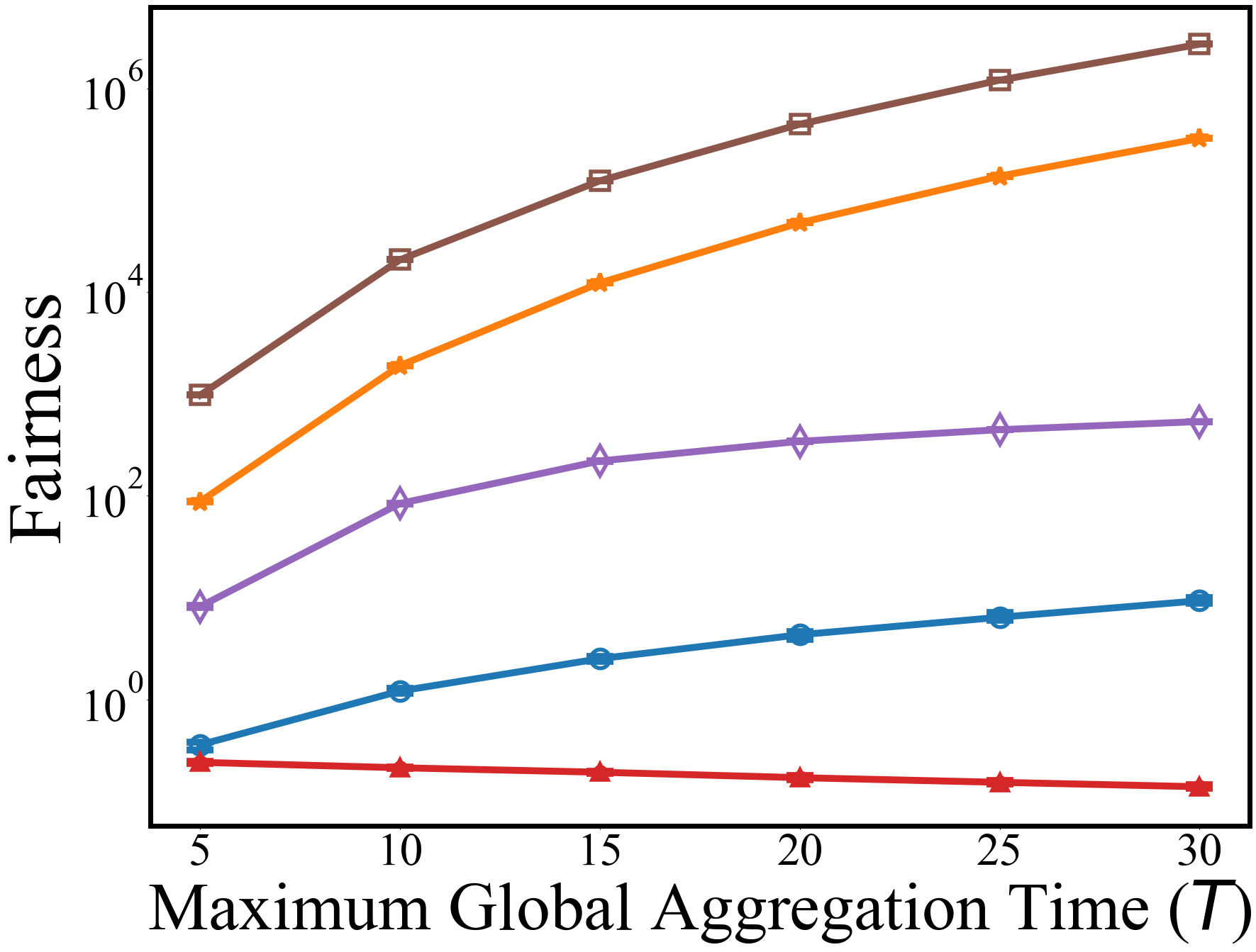}}
\subfigure[{Accuracy vs. $T$ (CNN, FMNIST)}]{
\label{benchmarks_accuracy_fmnist_cnn-minmax_1}
\includegraphics[width=0.23\textwidth]{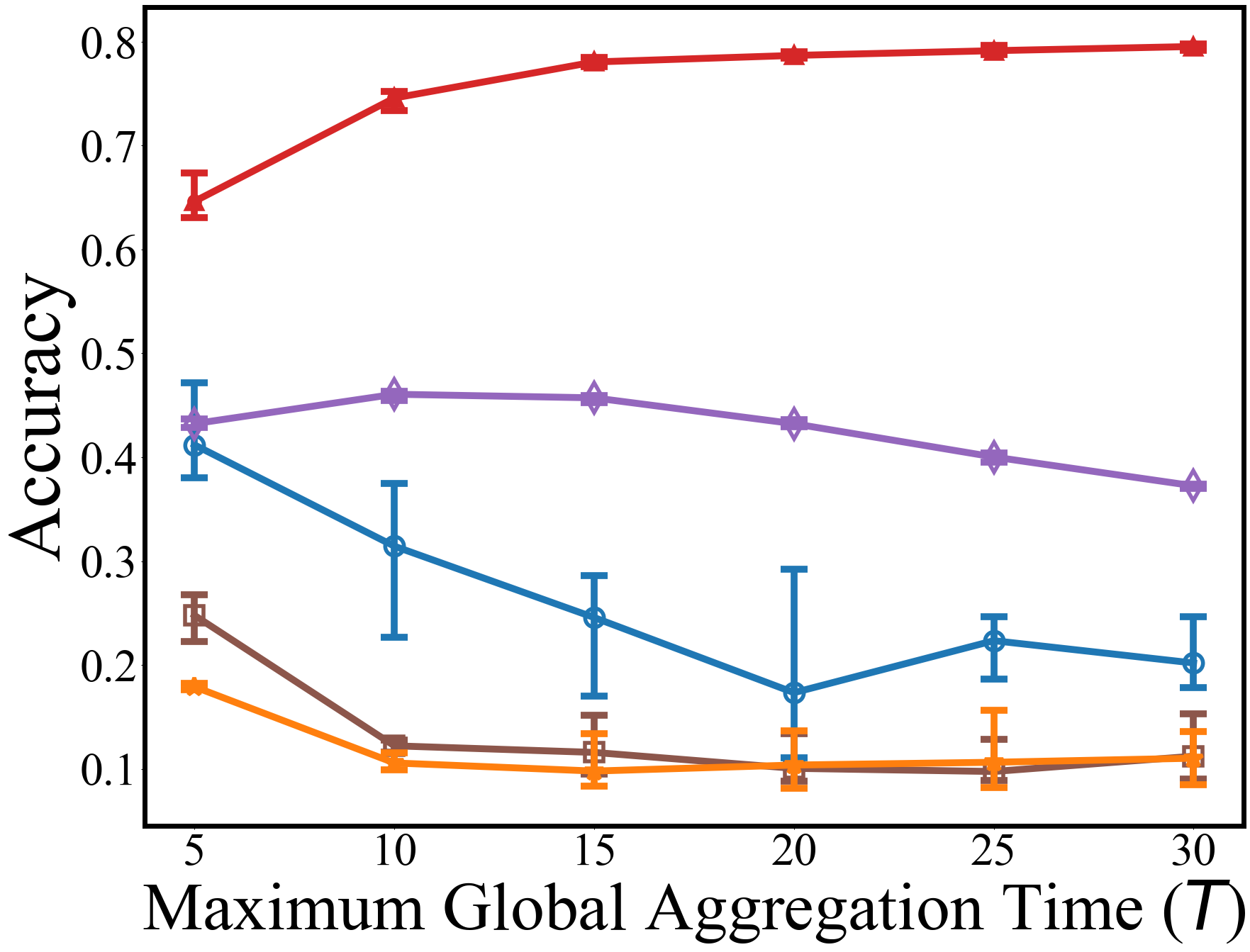}}
\hspace{-1mm}
\subfigure[{Fairness vs. $T$ (CNN, FMNIST)}]{
\label{benchmarks_fairness_fmnist_cnn-minmax_1}
\includegraphics[width=0.23\textwidth]{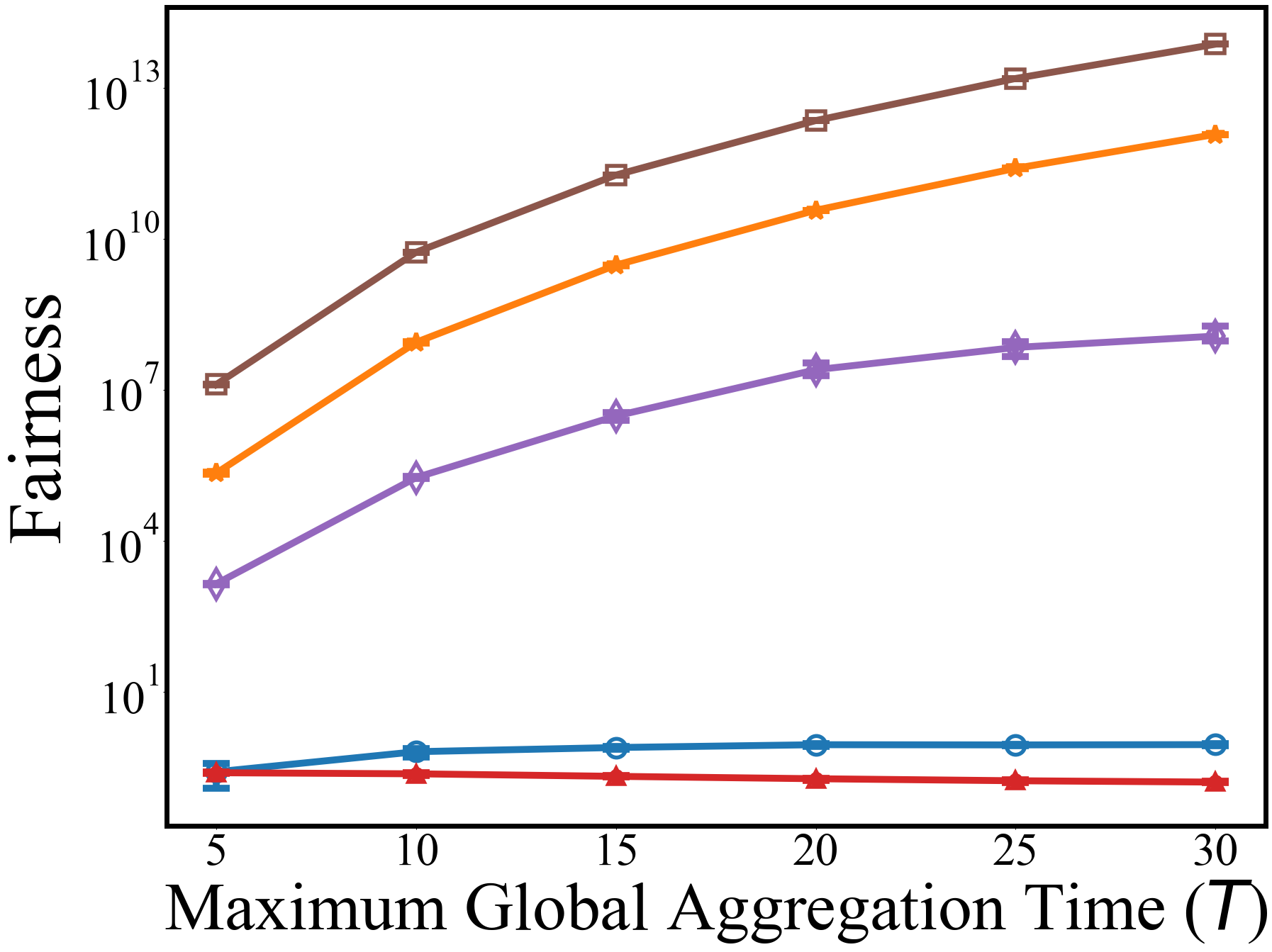}}
\\
\subfigure[{Accuracy vs. $T$ (CNN, CIFAR10)}]{
\label{benchmarks_accuracy_cifar10_cnn-minmax_1}
\includegraphics[width=0.23\textwidth]{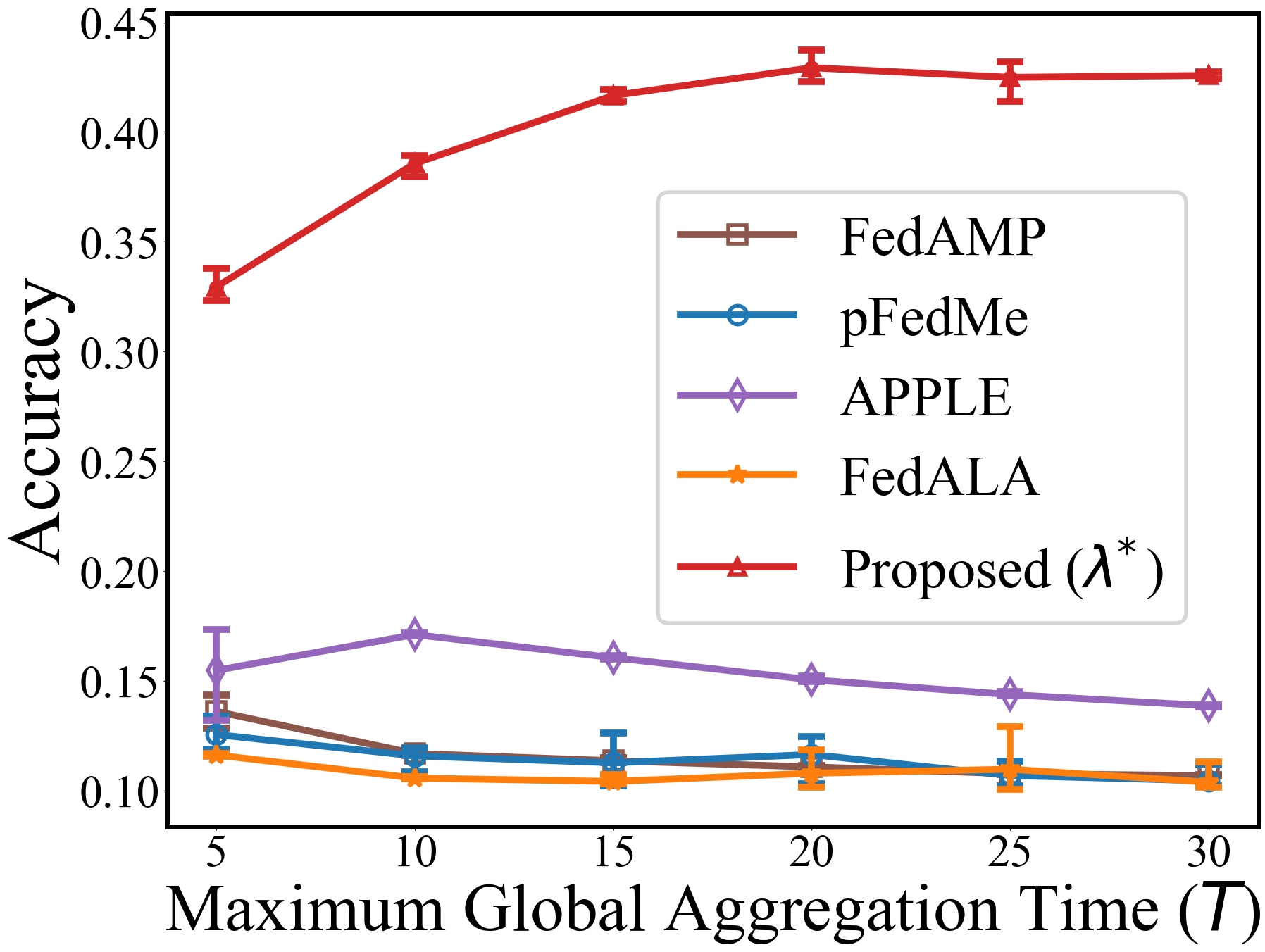}}
\hspace{-1mm}
\subfigure[{Fairness vs. $T$ (CNN, CIFAR10)}]{
\label{benchmarks_fairness_cifar10_cnn-minmax_1}
\includegraphics[width=0.23\textwidth]{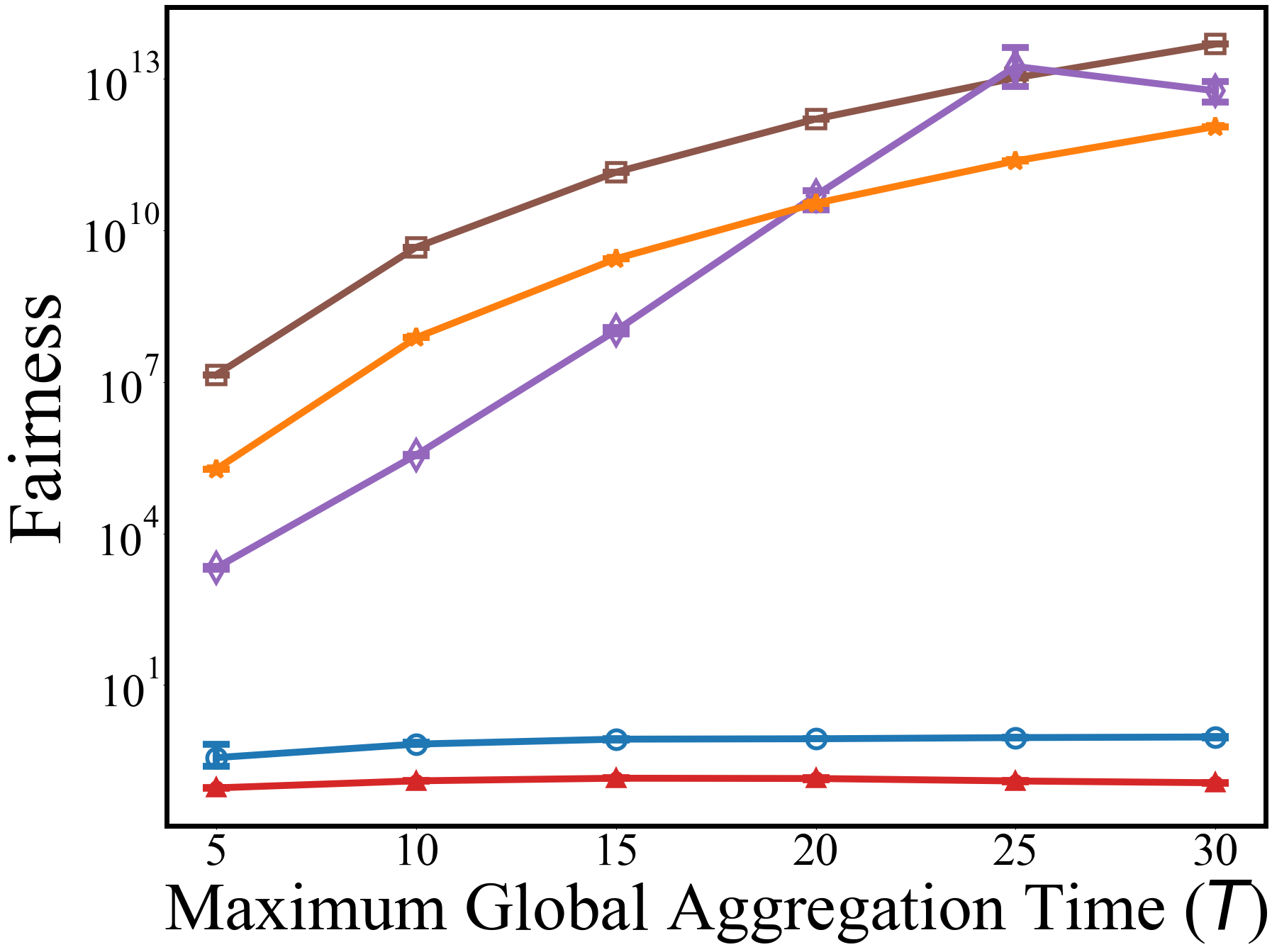}}
\caption{{Comparison of testing accuracy and fairness between the benchmarks with DP and DP-Ditto under the optimal $\lambda^{\ast}=0.005$, $\epsilon=10$, and $\delta=0.01$.}}
\label{benchmarks-dpe10}
\end{figure}

\begin{figure}[t]
\centering
\subfigure[{Accuracy vs. $T$ (DNN, MNIST)}]{
\label{benchmarks_accuracy_mnist_dnn-minmax_test_1}
\includegraphics[width=0.23\textwidth]{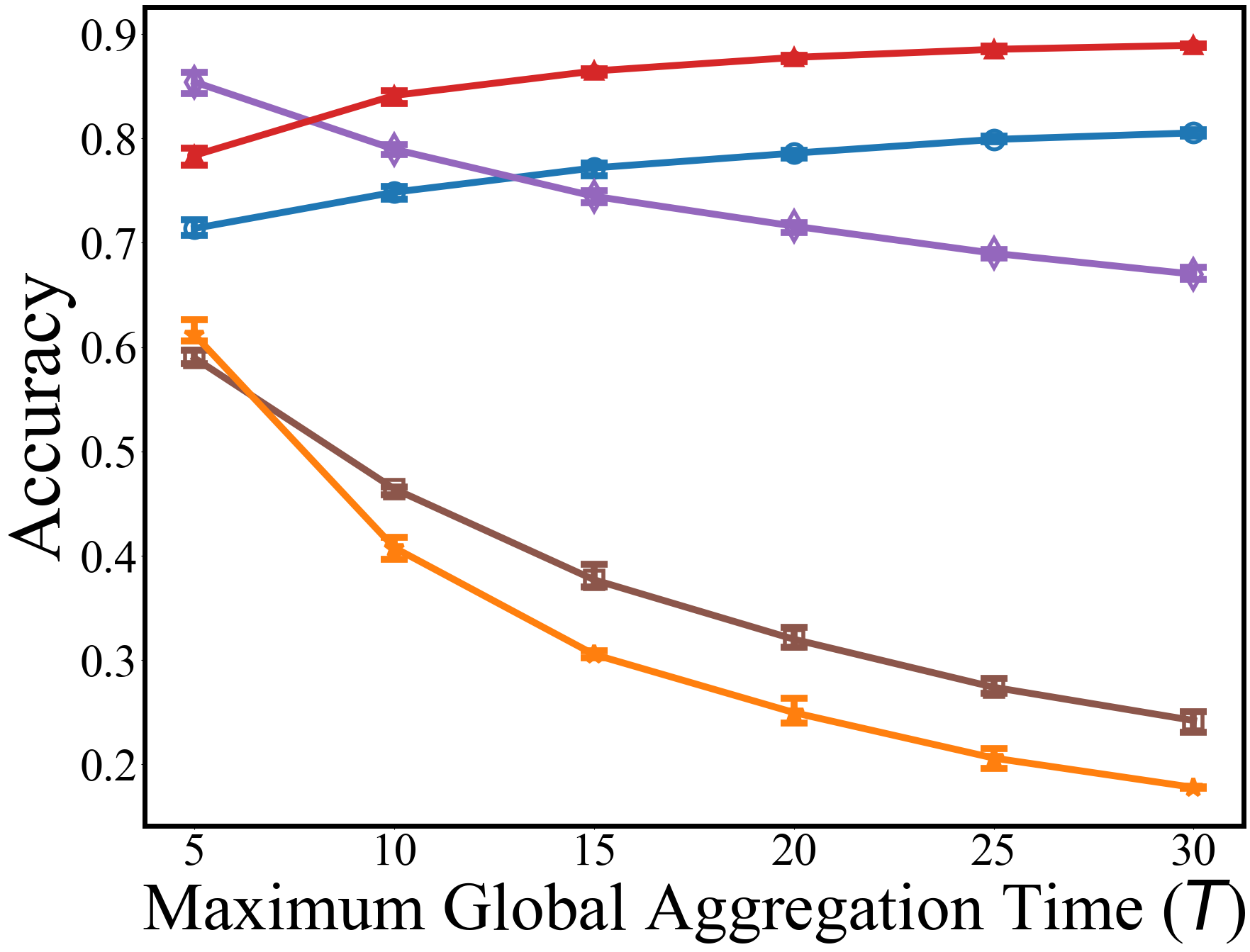}}
\hspace{-1mm}
\subfigure[{Fairness vs. $T$ (DNN, MNIST)}]{
\label{benchmarks_fairness_mnist_dnn-minmax_test_1}
\includegraphics[width=0.23\textwidth]{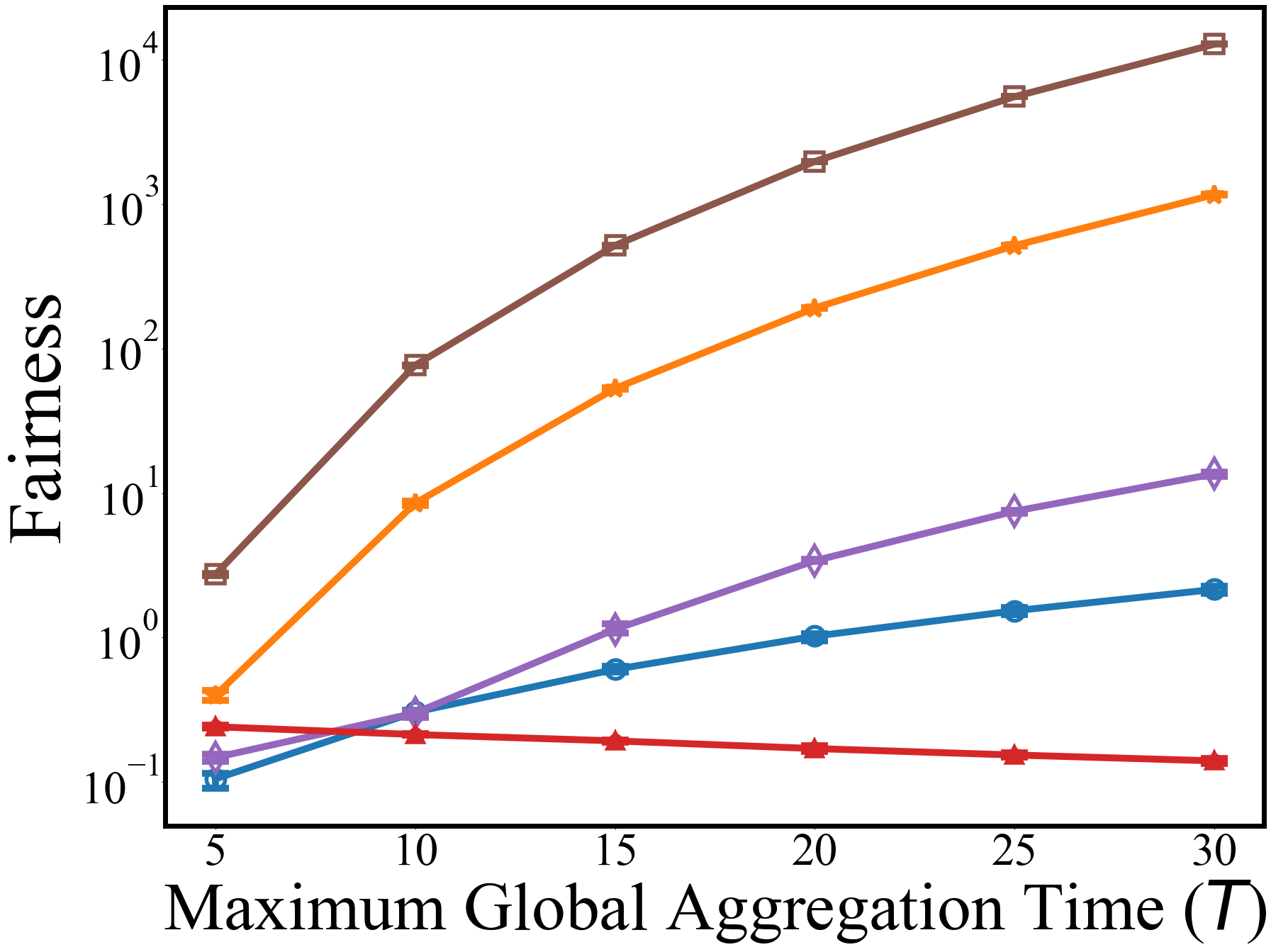}}
\subfigure[{Accuracy vs. $T$ (CNN, FMNIST)}]{
\label{benchmarks_accuracy_fmnist_cnn-minmax_1}
\includegraphics[width=0.23\textwidth]{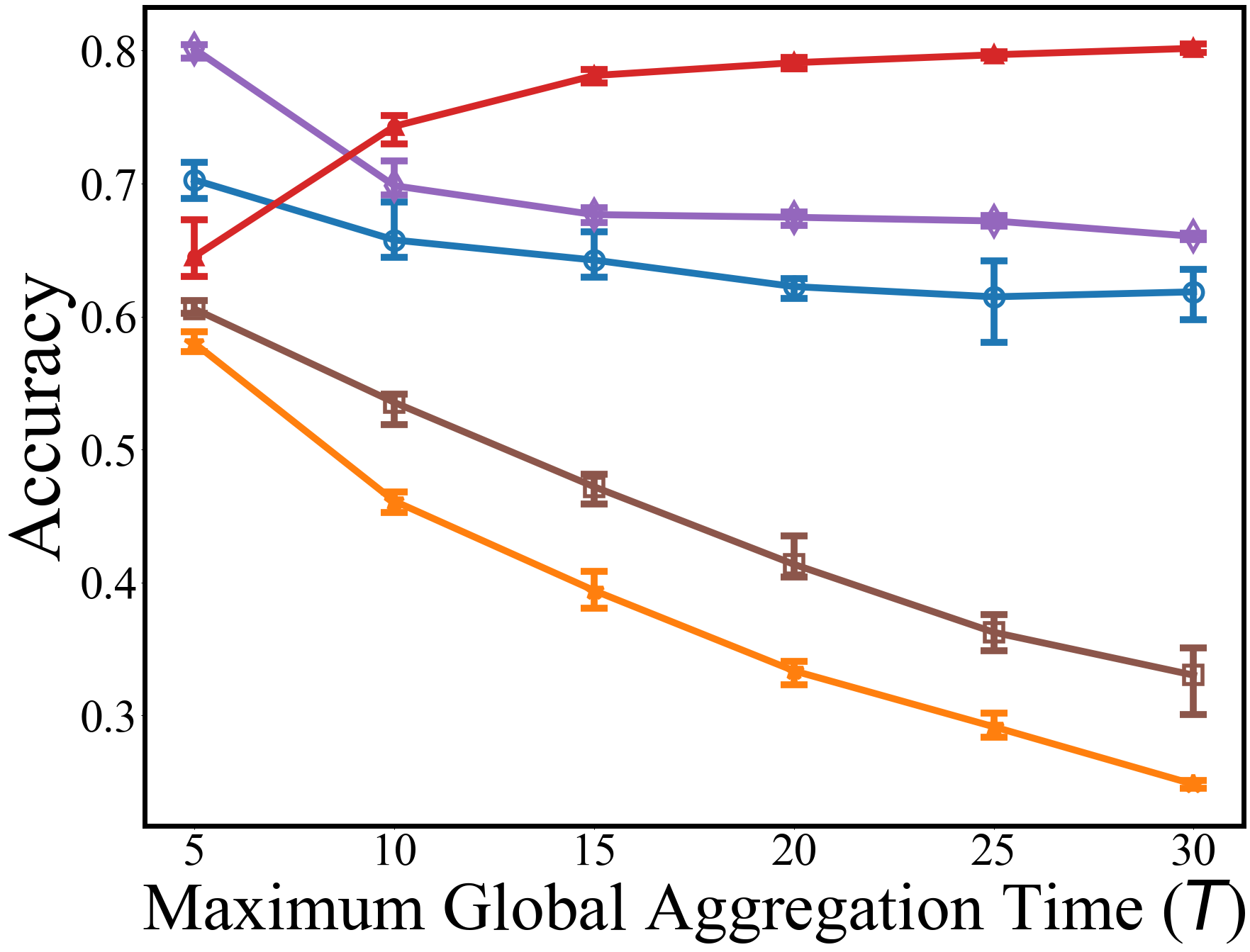}}
\hspace{-1mm}
\subfigure[{Fairness vs. $T$ (CNN, FMNIST)}]{
\label{benchmarks_fairness_fmnist_cnn-minmax_1}
\includegraphics[width=0.23\textwidth]{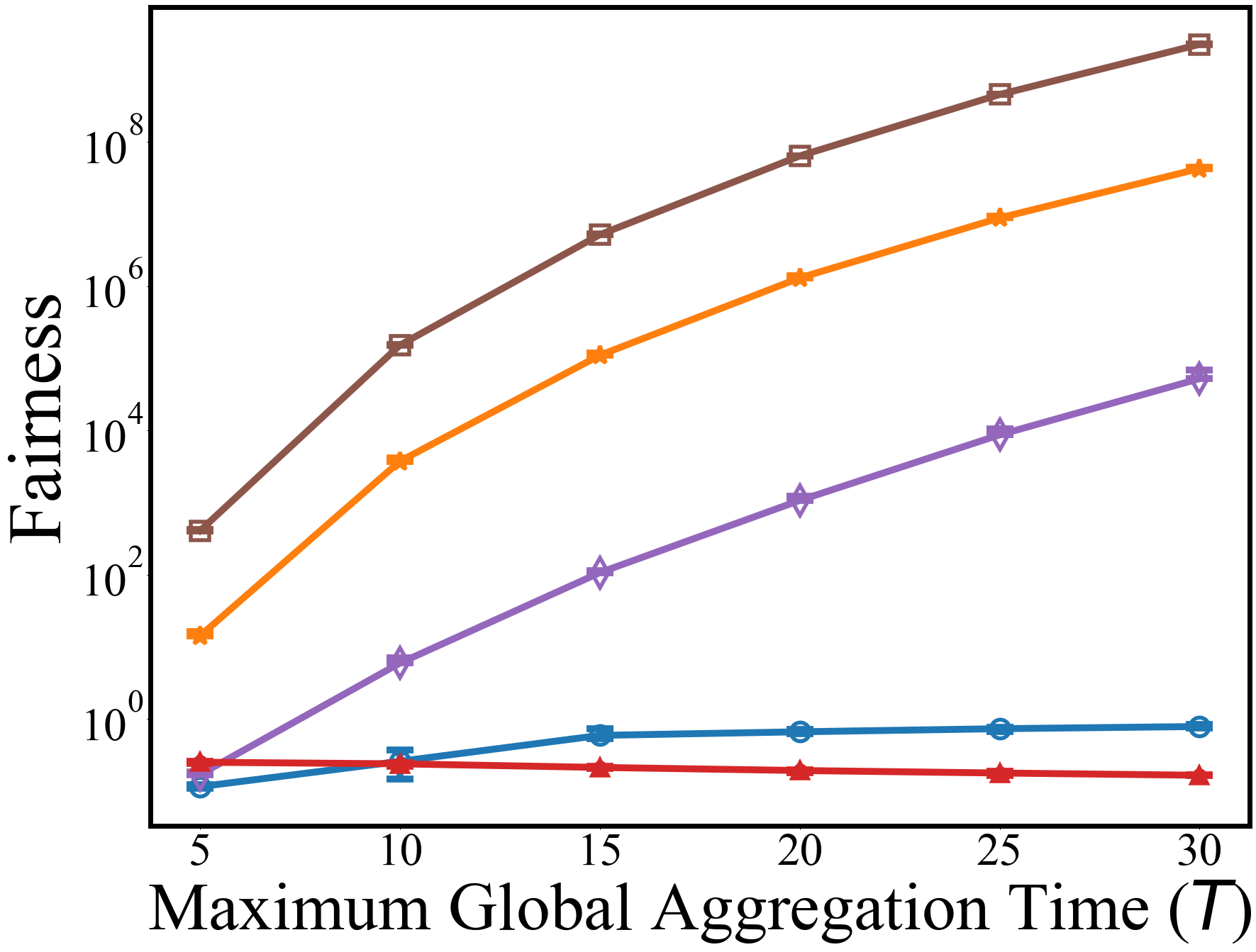}}
\\
\subfigure[{Accuracy vs. $T$ (CNN, CIFAR10)}]{
\label{benchmarks_accuracy_cifar10_cnn-minmax_1}
\includegraphics[width=0.23\textwidth]{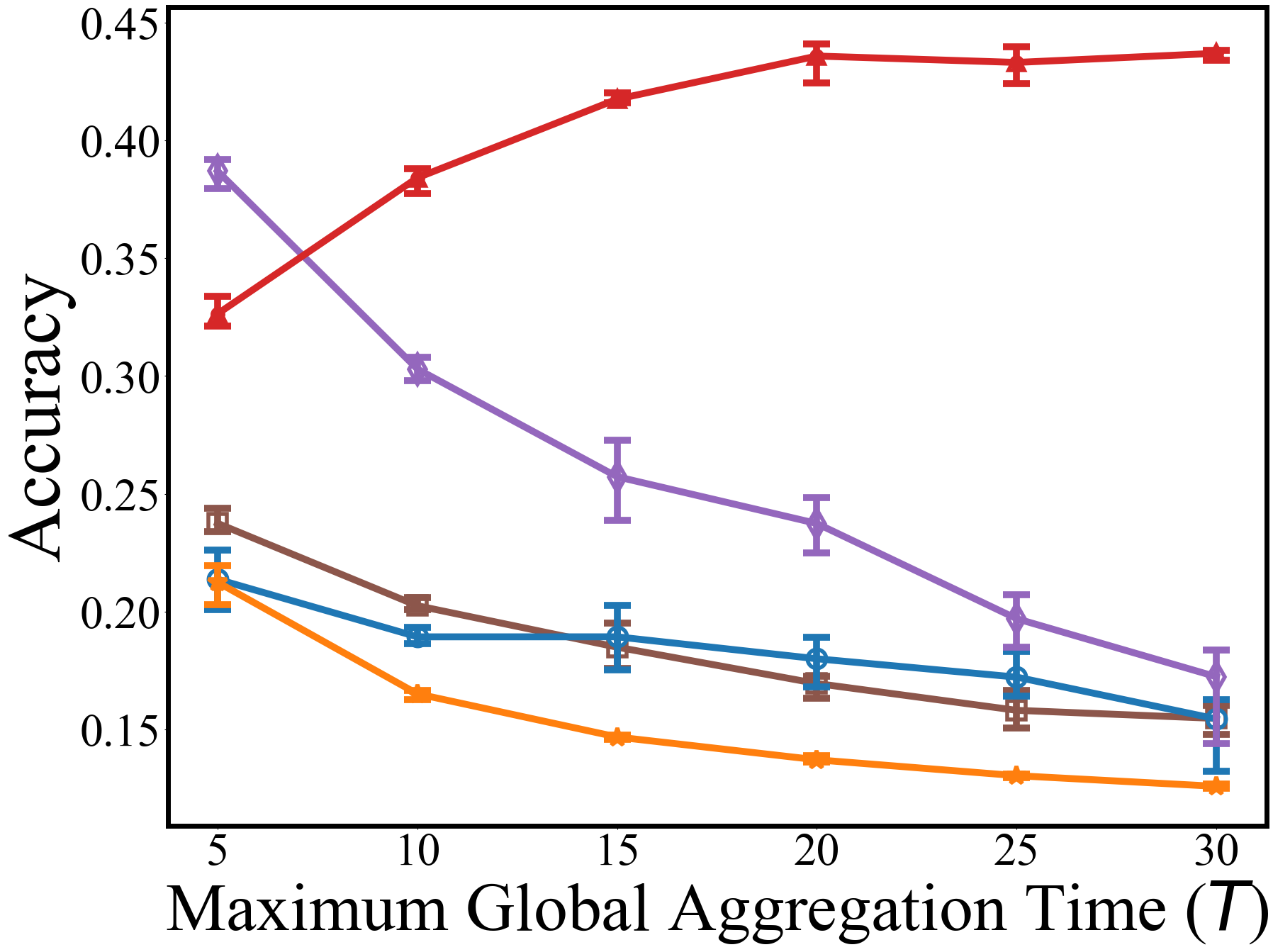}}
\hspace{-1mm}
\subfigure[{Fairness vs. $T$ (CNN, CIFAR10)}]{
\label{benchmarks_fairness_cifar10_cnn-minmax_1}
\includegraphics[width=0.23\textwidth]{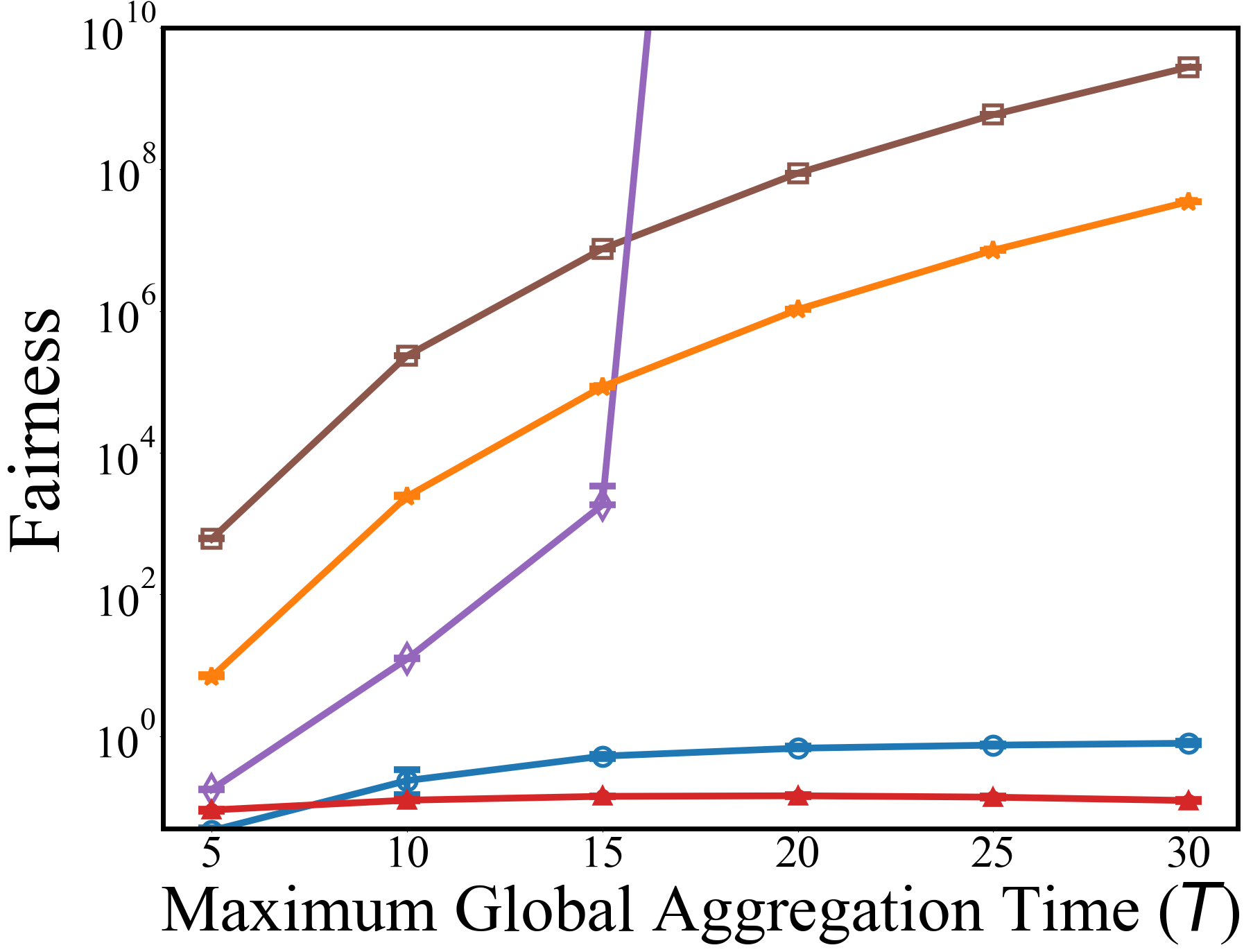}}
\caption{{Comparison of testing accuracy and fairness with the benchmarks with DP and DP-Ditto under the optimal $\lambda^{\ast}=0.01$, $\epsilon=100$, and $\delta=0.01$.}}
\label{benchmarks-dpe100}
\end{figure}

We compare the accuracy and fairness between the proposed DP-Ditto and the benchmarks, i.e., FedAMP~\cite{huang2021personalized}, pFedMe~\cite{t2020personalized}, APPLE~\cite{luo2022adapt}, and FedALA~\cite{zhang2023fedala}, where DP noises are added to perturb the local models in the benchmarks {under different $\epsilon$ values (i.e., $\epsilon=10$, $100$) and datasets (i.e., MNIST, FMNIST, and CIFAR10).} $\delta=0.01$. 
The DNN model is considered on the MNIST dataset. The CNN model is considered on the FMNIST and CIFAR10 datasets.

Figs.~\ref{benchmarks-dpe10} and \ref{benchmarks-dpe100} plot the testing accuracy and fairness of privacy-preserving PFL with the growth of $T$, where $\epsilon=10$ and $100$, respectively. 
{
DP-Ditto (with $\lambda^*=0.005$ under $\epsilon=10$, or $\lambda^*=0.01$ under $\epsilon=100$) provides the best accuracy and fairness compared to the privacy-enhanced benchmarks (i.e., FedAMP, pFedMe, APPLE, and FedALA), since $\lambda^{\ast}$ is adapted to the DP perturbation in DP-Ditto. We note that 
the PL models are obtained by aggregating the downloaded models in APPLE~\cite{luo2022adapt} and FedALA~\cite{zhang2023fedala}, or based on a weighted downloaded global model aggregated from the previous personalized models in pFedMe~\cite{t2020personalized} and FedAMP~\cite{huang2021personalized}, both without considering privacy.

For fair comparisons with DP-Ditto, the PL models of the benchmarks are updated based on the aggregated noisy models perturbed using DP to enhance the privacy aspect of the models.
Due to their limited flexibility in balancing personalization and generalization, the benchmarks are highly sensitive to DP noises, resulting in significant performance degradation.
By contrast, under DP-Ditto, the impact of DP noise can be adjusted by properly configuring the weighting coefficient $\lambda$ between personalization and generalization,  hence alleviating the adverse effect of the DP noise.}

By comparing Figs.~\ref{benchmarks-dpe10} and \ref{benchmarks-dpe100}, we see that the testing accuracy and fairness of DP-Ditto can be maintained even under high privacy requirements (e.g., $\epsilon=10$). 
{This is achieved 
by configuring a smaller $\lambda^*$ when $\epsilon$ is smaller (i.e., $\sigma_u^2$ is higher), which encourages the PL models to be closer to the local models and hence alleviates the adverse effect of DP noise.}
By contrast, the benchmarks degrade significantly when  $\epsilon$ is small (e.g., $\epsilon=10$) and/or $T$ is large because of their susceptibility to the DP noises.  
{As $T$ grows, the testing accuracy decreases under FedALA and FedAMP, or increases when $T\leq 5$ and then decreases under APPLE. Moreover, the testing accuracy of pFedMe 
decreases under the CNN model on the FMNIST and CIFAR10 datasets.}

\subsubsection{Impact of Privacy Budget}

\begin{figure}[t]
\centering  
\subfigure[Train. loss vs. $T$ (DNN, MNIST)]
{
\label{maxt_loss_mnist_MA_dnn_1}
\includegraphics[width=0.23\textwidth]{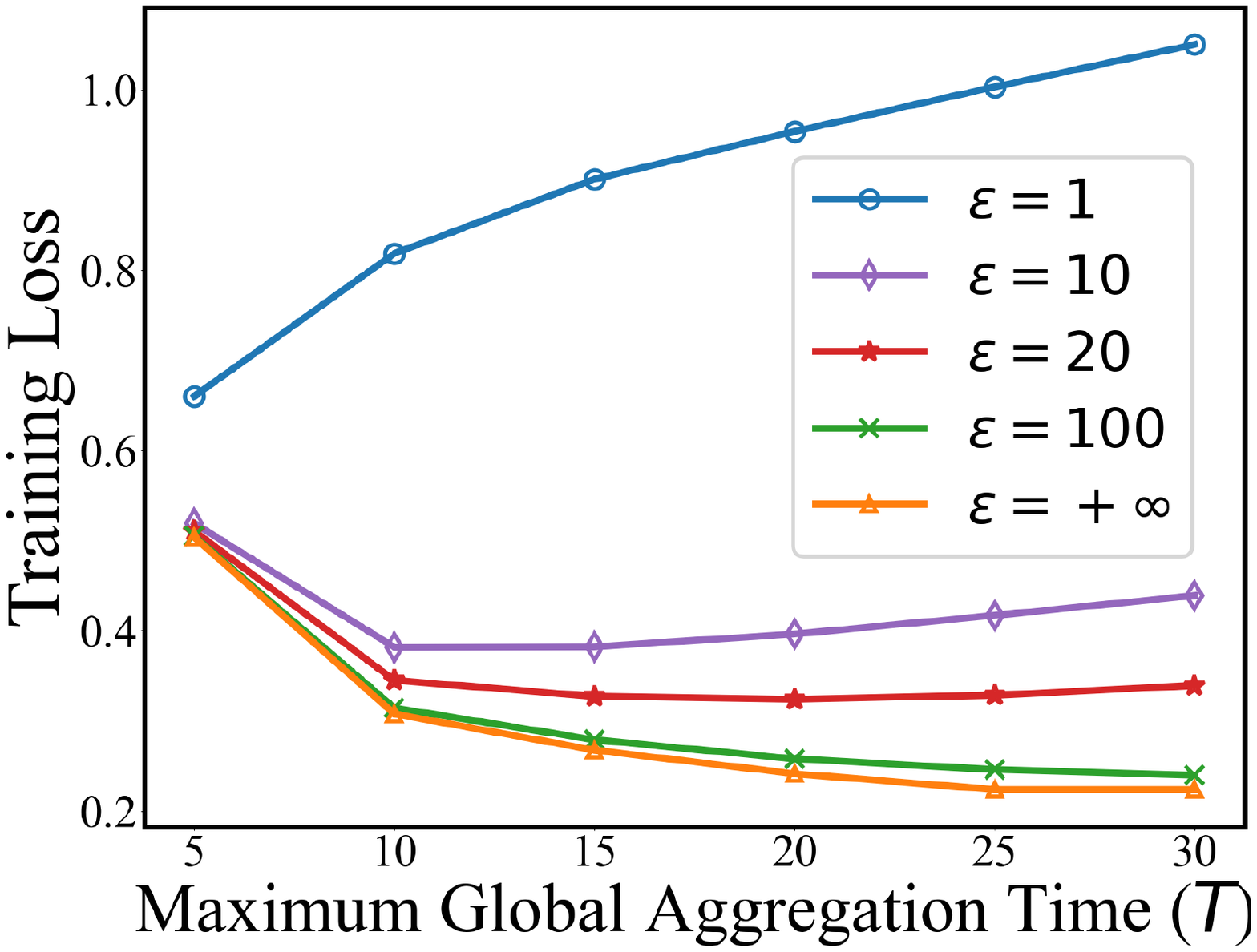}} 
\hspace{-1mm}
\subfigure[Train. loss vs. $T$ (MLR, MNIST)]{
\label{maxt_loss_mnist_MA_mlr_1}
\includegraphics[width=0.23\textwidth]{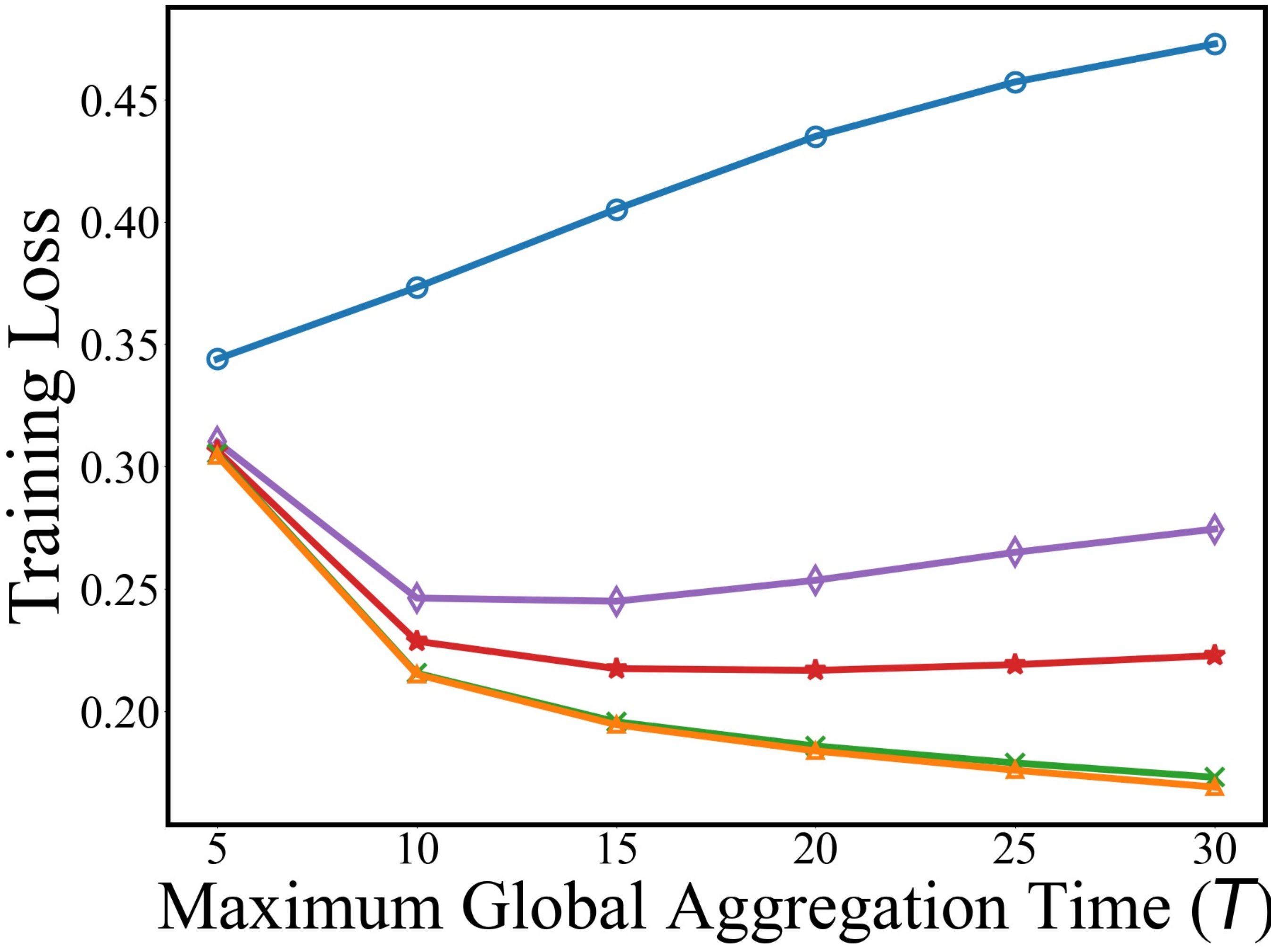}}
\\
\subfigure[Train. loss vs. $T$ (CNN,CIFAR10)]{
\label{maxt_loss_CIFAR10_MA_cnn_1}
\includegraphics[width=0.23\textwidth]{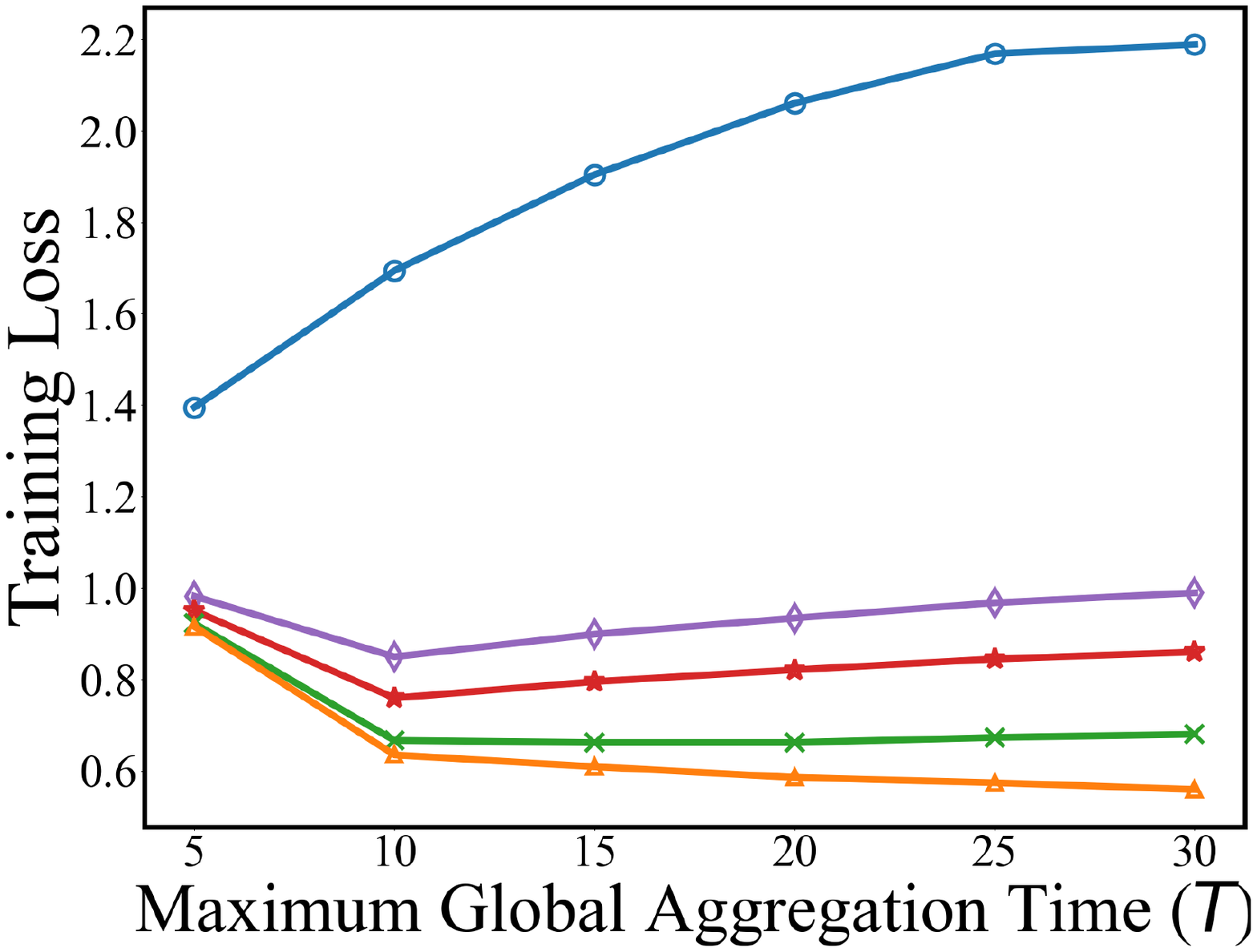}}
\hspace{-1mm}
\subfigure[Train. loss vs. $T$ (CNN, FMNIST)]{
\label{maxt_loss_FMNIST_MA_cnn_1}
\includegraphics[width=0.23\textwidth]{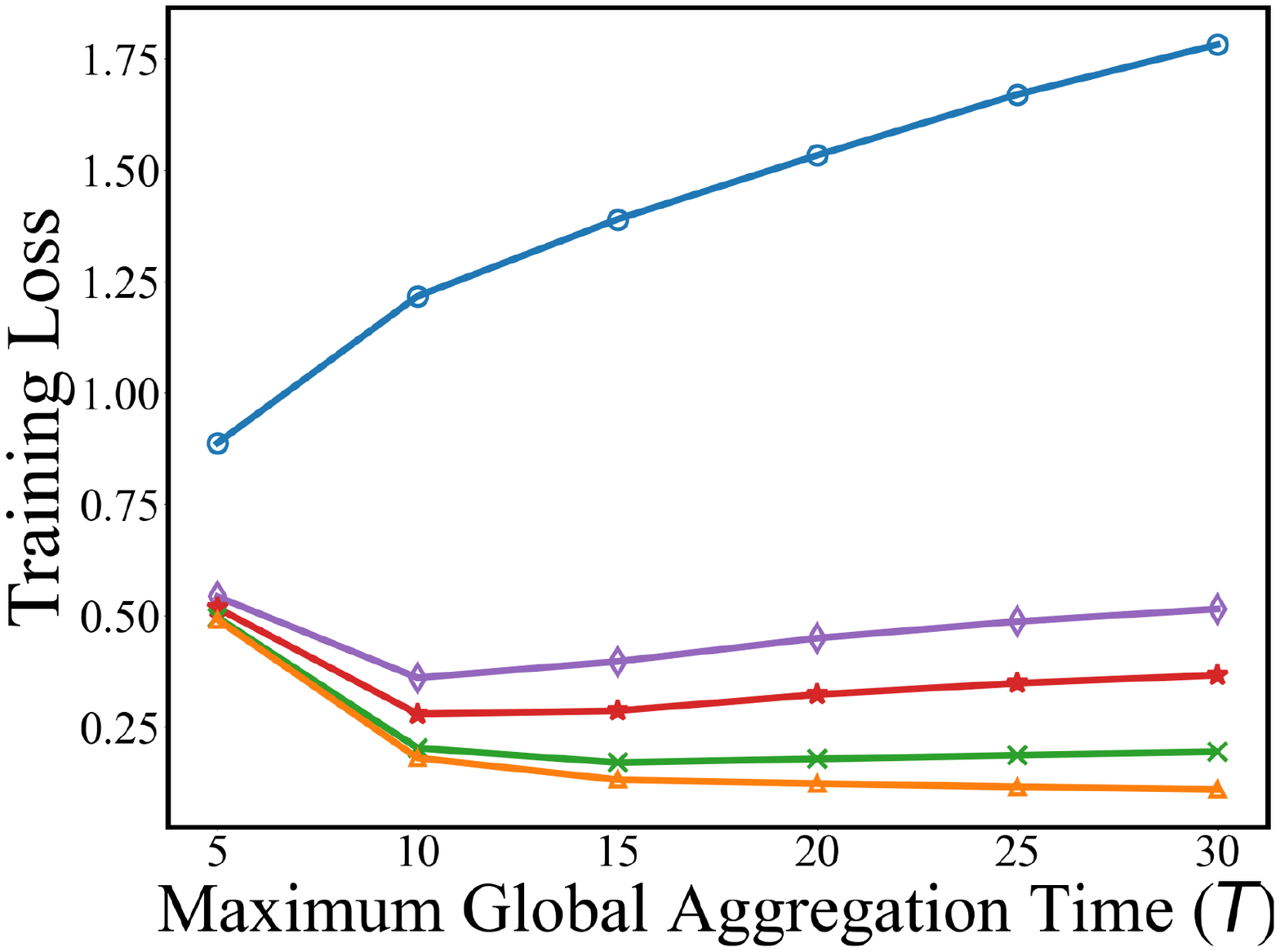}}
\caption{Training loss of the personalized model with respect to the maximum global aggregation number $T$ under different $\epsilon$ values. $\delta=0.01$.}
\label{loss-T}
\end{figure}

Fig.~\ref{loss-T} evaluates the impact of $\epsilon$ and $T$ on the convergence of DP-Ditto, where $\lambda=0.1$ for the DNN and MLR models on the MNIST dataset, the CNN model on the CIFAR10 dataset, and the CNN model on the FMNIST dataset. $\epsilon=1, 10, 20, 100$, or $\epsilon=+\infty$ (i.e., Ditto). Figs.~\ref{maxt_loss_mnist_MA_dnn_1}--\ref{maxt_loss_FMNIST_MA_cnn_1} show that the training loss decreases and eventually approaches the case with no DP (i.e., $\epsilon=+\infty$), as $\epsilon$ increases. 
{
This is because a larger $\epsilon$ leads to a smaller variance of the DP noise and, consequently, the server can obtain better-quality local models from the clients. 
On the other hand, 
a smaller $\epsilon$ leads to a larger  DP noise variance (i.e., $\sigma_u^2$) based on \eqref{noiseSigma1} and, hence, the accumulated effect of the DP noise would be significant. 
However, when both $\epsilon$ and $T$ are small, the training loss of PL decreases initially due to the benefits of generalization brought from FL. 
Once $T$ exceeds a certain critical value, the accumulated impact of the DP noises outweighs the benefits of generalization, leading to a degradation in PL performance.}
This observation validates the discussion in Section~IV.

\subsubsection{Impact of $T$ and $\lambda$}

\begin{figure}[t]
\centering  
\subfigure[Accuracy vs. $t$ (DNN, MNIST)]
{
\label{Accuracy_mnist_lambda_1_10}
\includegraphics[width=0.23\textwidth]{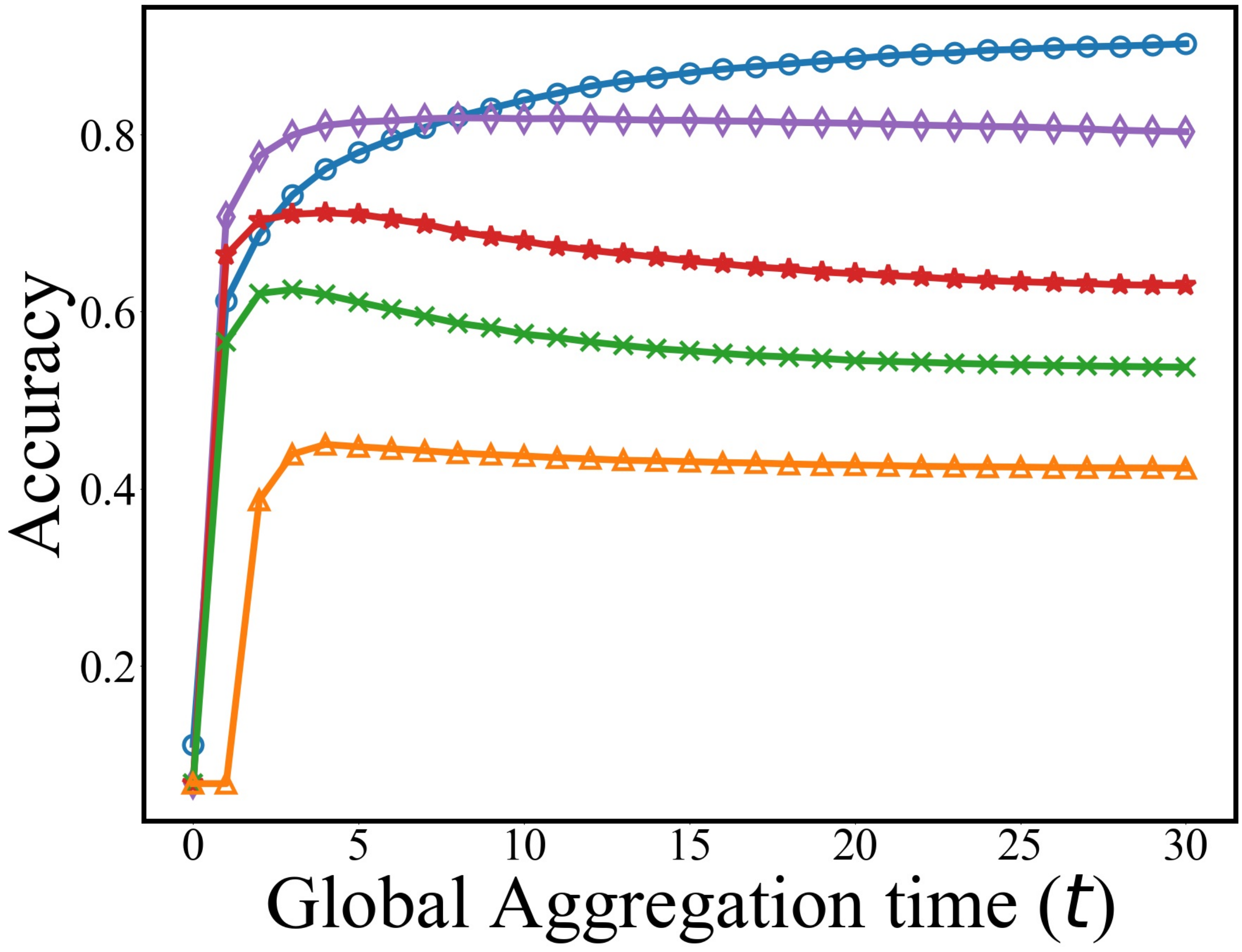}} 
\hspace{-1mm}
\subfigure[Fairness vs. $t$ (DNN, MNIST)]{
\label{Fairness_mnist_lambda_2_10}
\includegraphics[width=0.23\textwidth]{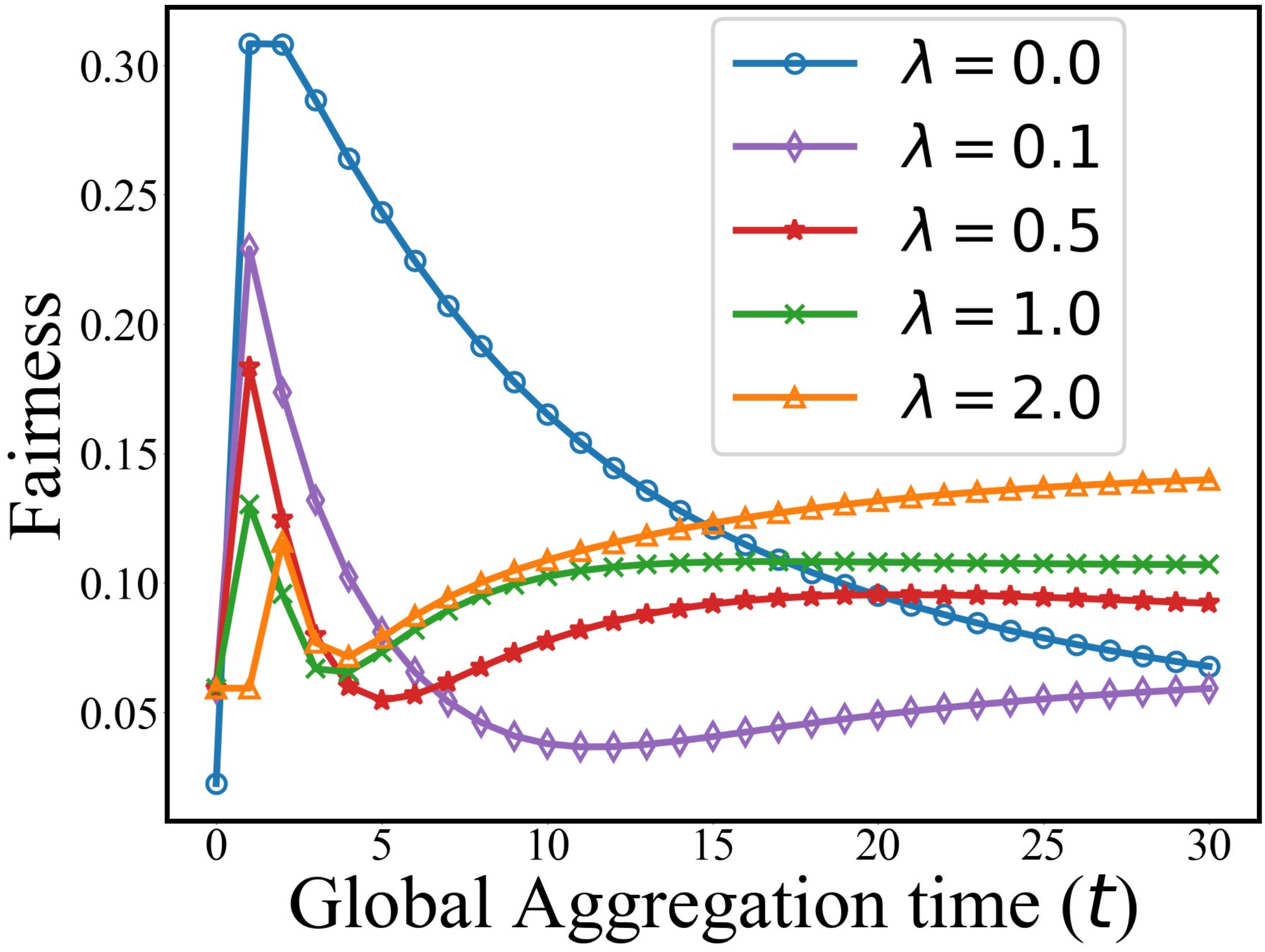}}
\\

\subfigure[{Accuracy vs. $t$ (CNN, FMNIST)}]
{
\label{Accuracy_fmnist_cnn_lambda_T30}
\includegraphics[width=0.23\textwidth]{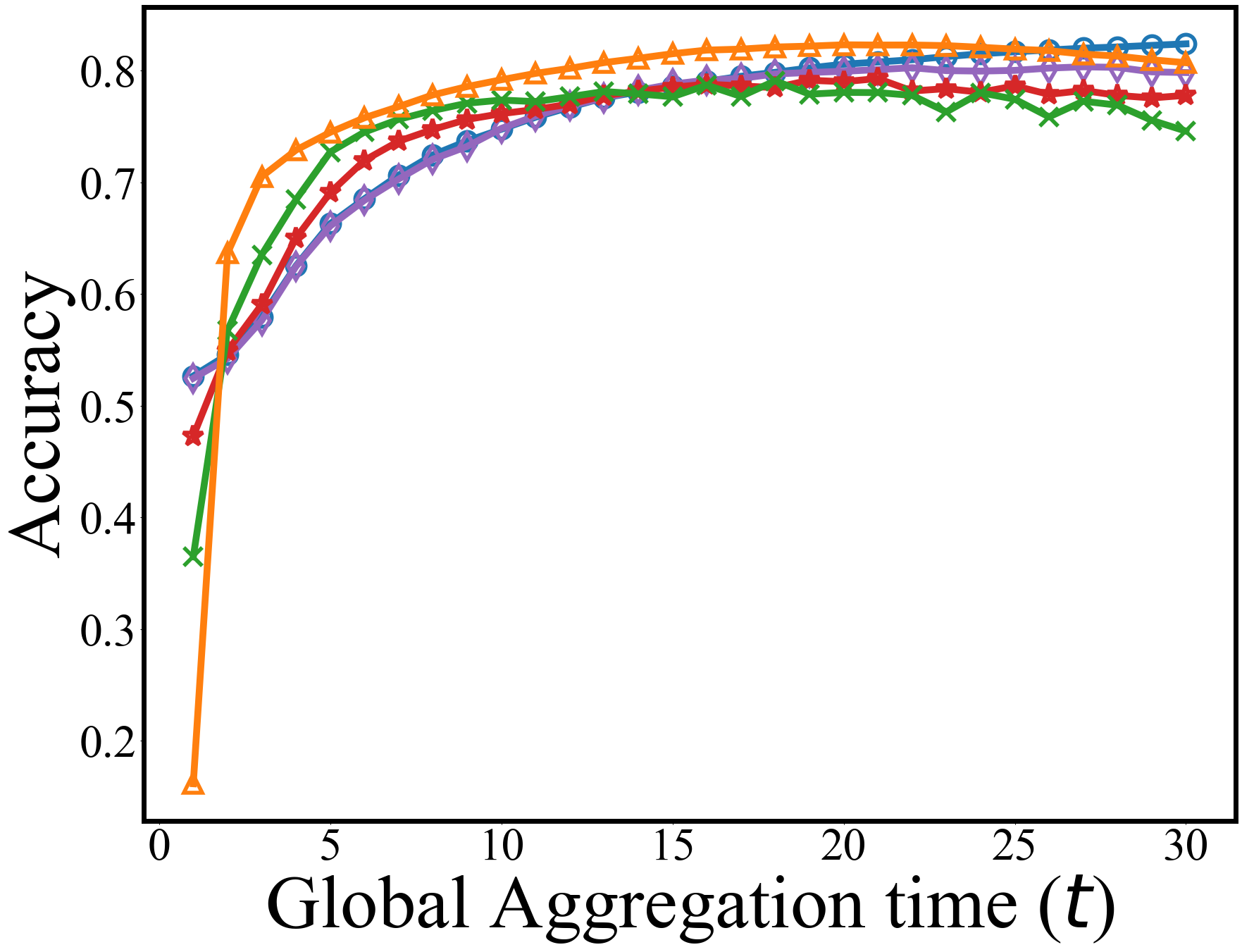}} 
\hspace{-1mm}
\subfigure[{Fairness vs. $t$ (CNN, FMNIST)}]{
\label{Fairness_fmnist_cnn_lambda_T30}
\includegraphics[width=0.23\textwidth]{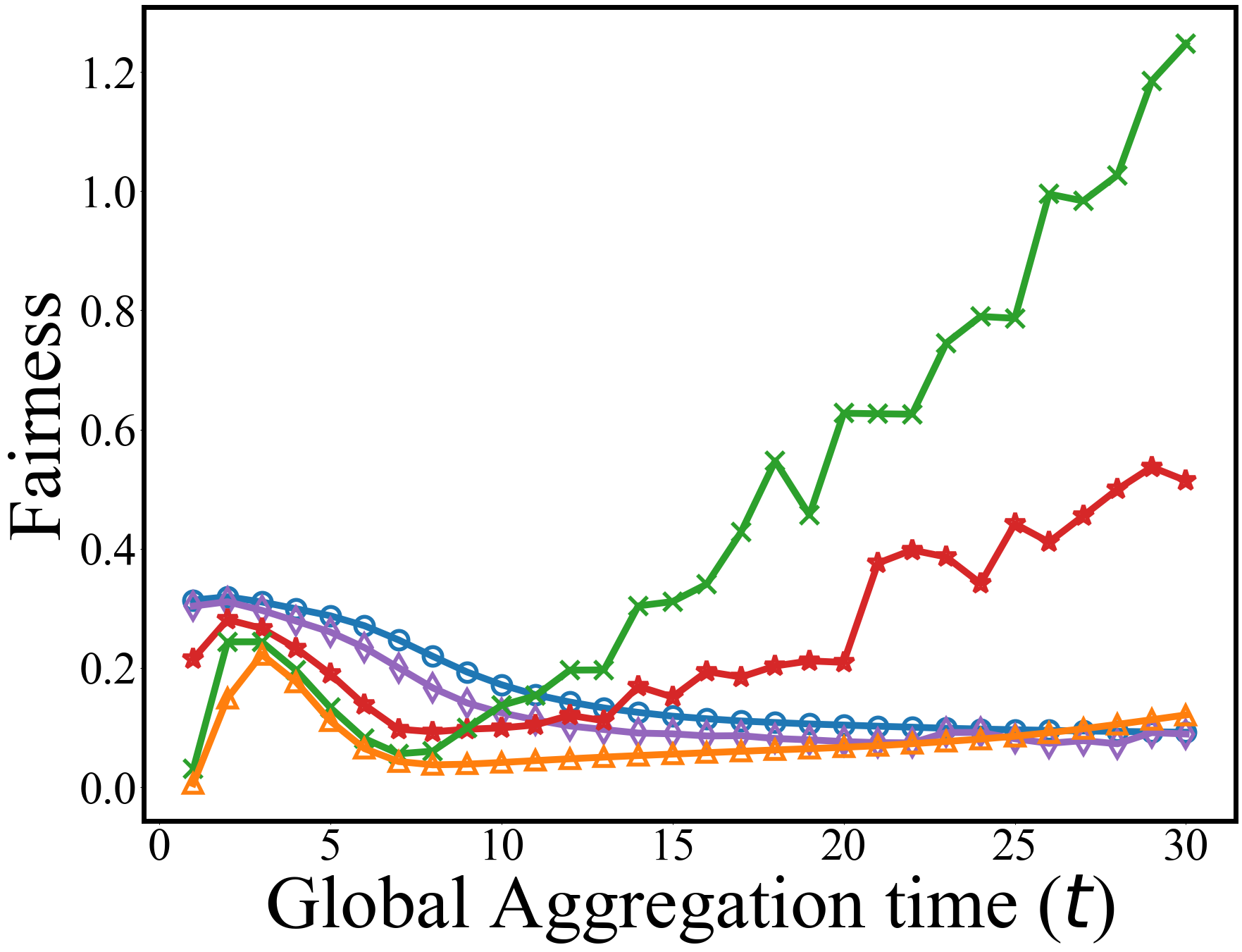}}
\\
\subfigure[{Accuracy vs. $t$ (CNN, CIFAR10)}]
{
\label{Accuracy_Cifar10_cnn_lambda_T30}
\includegraphics[width=0.23\textwidth]{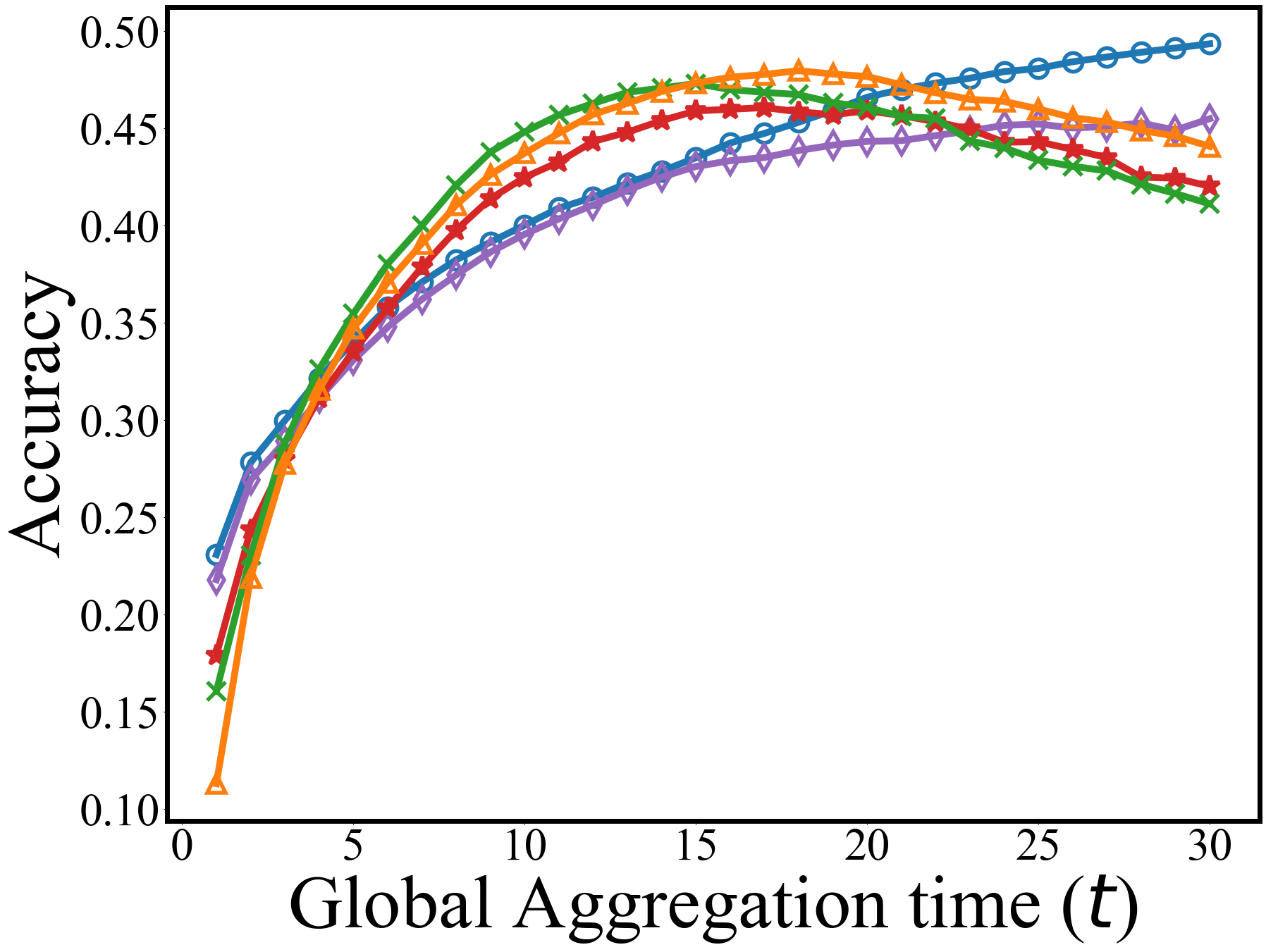}} 
\hspace{-1mm}
\subfigure[{Fairness vs. $t$ (CNN, CIFAR10)}]{
\label{Fairness_Cifar10_cnn_lambda_T30}
\includegraphics[width=0.23\textwidth]{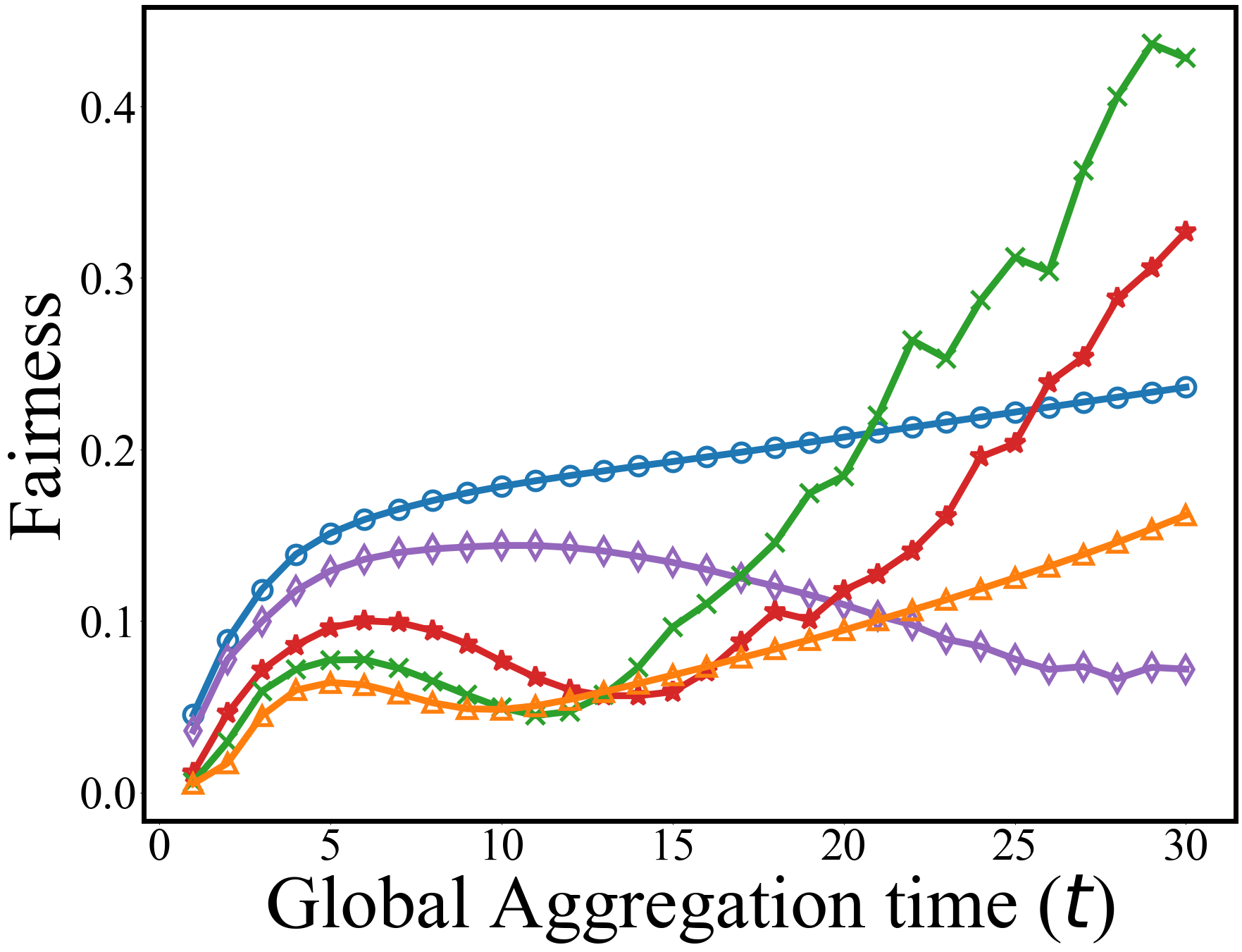}}
\caption{Comparison of accuracy and fairness vs. aggregation time $t$, under different $\lambda$ values with $T=30$, $\epsilon=10$, $\delta=0.01$.}
\label{acc_lambda_iid}
\end{figure}

Figs.~\ref{acc_lambda_iid} demonstrates the testing accuracy and fairness of DP-Ditto with the growing number $t$ of aggregations under different $\lambda$ values and datasets (i.e., MNIST, FMNIST, and CIFAR10).  
We set $\epsilon=10$, $\delta=0.01$, $ T=30$, 
$\lambda=0.1$, $0.5$, and $1.0$, as well as two special cases with $\lambda=0$ (i.e., local training only with no global aggregation) and $\lambda=2$ (i.e., FL with no personalization).
Fig.~\ref{Accuracy_mnist_lambda_1_10} shows the model accuracy increases with $t$ when $\lambda<0.1$, but increases first and then decreases when $\lambda>0.1$.
The reason is that the {FL} global model is most affected by the DP noise through the aggregations of noisy local model parameters. By contrast, the FL local training only depends on the local datasets and is unaffected by the DP noises. 
Moreover, $\lambda$ can be controlled to balance the accuracy between the {FL} global and local models. It adjusts the effect of the DP noises on the PL models.

Fig.~\ref{Fairness_mnist_lambda_2_10} gauges the fairness of DP-Ditto measured by the standard deviation of the  {training losses} concerning the PL models of all clients.  
The fairness of the PL models first degrades and then improves, {and then degrades again} as $t$ rises. The reason is that the accuracy of the clients' models is poor and, hence, fairness is high initially. As $t$ rises, the DP noises increasingly affect the PL models. 
The fairness is better than that of the global FL model {after a specific $t$} because the clients ignore the heterogeneity of their local data when all clients utilize the same global FL model for image classification. The PL model with a smaller $\lambda$ offers better fairness. {When $\lambda=0.1$, the PL model offers the best fairness.}


It is observed from Figs.~\ref{Accuracy_mnist_lambda_1_10}, \ref{Accuracy_fmnist_cnn_lambda_T30}, and \ref{Accuracy_Cifar10_cnn_lambda_T30} that the testing accuracy first increases and then decreases with the increase of $t$ under $\lambda<2$. This is due to the fact that the effect of the DP noise accumulates as $t$ grows, causing performance degradation when $t$ is excessively large. By contrast, the updates of the local models (i.e., the PL models at $\lambda=0$) are unaffected by the DP noises. 
{As shown in Fig.~\ref{Accuracy_fmnist_cnn_lambda_T30}--\ref{Fairness_Cifar10_cnn_lambda_T30}, under the CNN models on the FMNIST and CIFAR10 datasets, when $t$ is small, the PL models operate the best in accuracy and fairness at $\lambda=1$ and $2$; in other words, the PL models are closer to the FL global models. This is because the adverse effect caused by noise accumulation is insignificant when $t$ is small, and the PL models can benefit from the generalization offered by the FL global model.
On the other hand, when $t$ is large, the PL models perform the best at $\lambda=0$; i.e., the PL models are closer to the FL local models, due to 
DP noise accumulation.}

We further obtain the optimal ($T^{\ast}$, $\lambda^{\ast}$) for the DNN model on the MNIST dataset, where $\epsilon=100$, $\delta=0.01$, {and $C=10$}, as described in Section~V-C. Fig.~{\ref{T-lambda-fairness-1}} demonstrates the fairness of the PL models with $\lambda$ under different $T$ values. Fig.~\ref{T-lambda-fairness-1} shows that
the optimal $\lambda^{\ast}$ decreases as $T$ increases.
When $\lambda$ is small (i.e., $\lambda \rightarrow 0^+$), the fairness of the PL models improves with $T$, as DP-Ditto is dominated by PL with little assistance from FL. 
Consequently, the DP noise has little impact on the PL models.
When $\lambda$ is large {(i.e., $\lambda \rightarrow 2$)}, the fairness degrades as $T$ increases, as priority is given to generalization over personalization. The adverse effect of the DP noise becomes increasingly strong, compromising the fairness of the PL models.
Fig.~\ref{lambda_T_loss-1} demonstrates the training loss of the PL models against $T$, under different $\lambda$ values. Given $T$, the training loss decreases as $\lambda$ decreases. Moreover, the optimal $T^{\ast}$ increases as {$\lambda$} 
 decreases. By comparing Figs.~\ref{T-lambda-fairness-1} and \ref{lambda_T_loss-1}, the optimal configuration ($T^{\ast}=80$, $\lambda^{\ast}=0.1$) is obtained to minimize the training loss and guarantee the fairness of the PL models.

\section{Conclusion}
We have proposed DP-Ditto, and analyzed its convergence and fairness performance under a given privacy budget. In particular, we have derived the convergence upper bound of the PL models under DP-Ditto and the optimal number of global aggregations. 
We also analyzed the performance fairness of the PL models and illustrated the feasibility of jointly optimizing the DP-Ditto configuration for convergence and fairness.
Extensive experiments have corroborated our analysis and showed that DP-Ditto can outperform its benchmarks significantly by over $24.3\%$ in fairness and $48.16\%$ in accuracy. 
{This work contributes to a better understanding of the interplay among the privacy, convergence, and fairness of PFL.
In the future, we will analyze fairness for nonlinear PFL models in imperfect wireless environments and design methods to balance the trade-off among the convergence, fairness, and privacy of the models.}

\begin{figure}[t]
	\centering
    \subfigure[Fairness vs. $\lambda$]{
\label{T-lambda-fairness-1}
\includegraphics[width=0.22\textwidth]{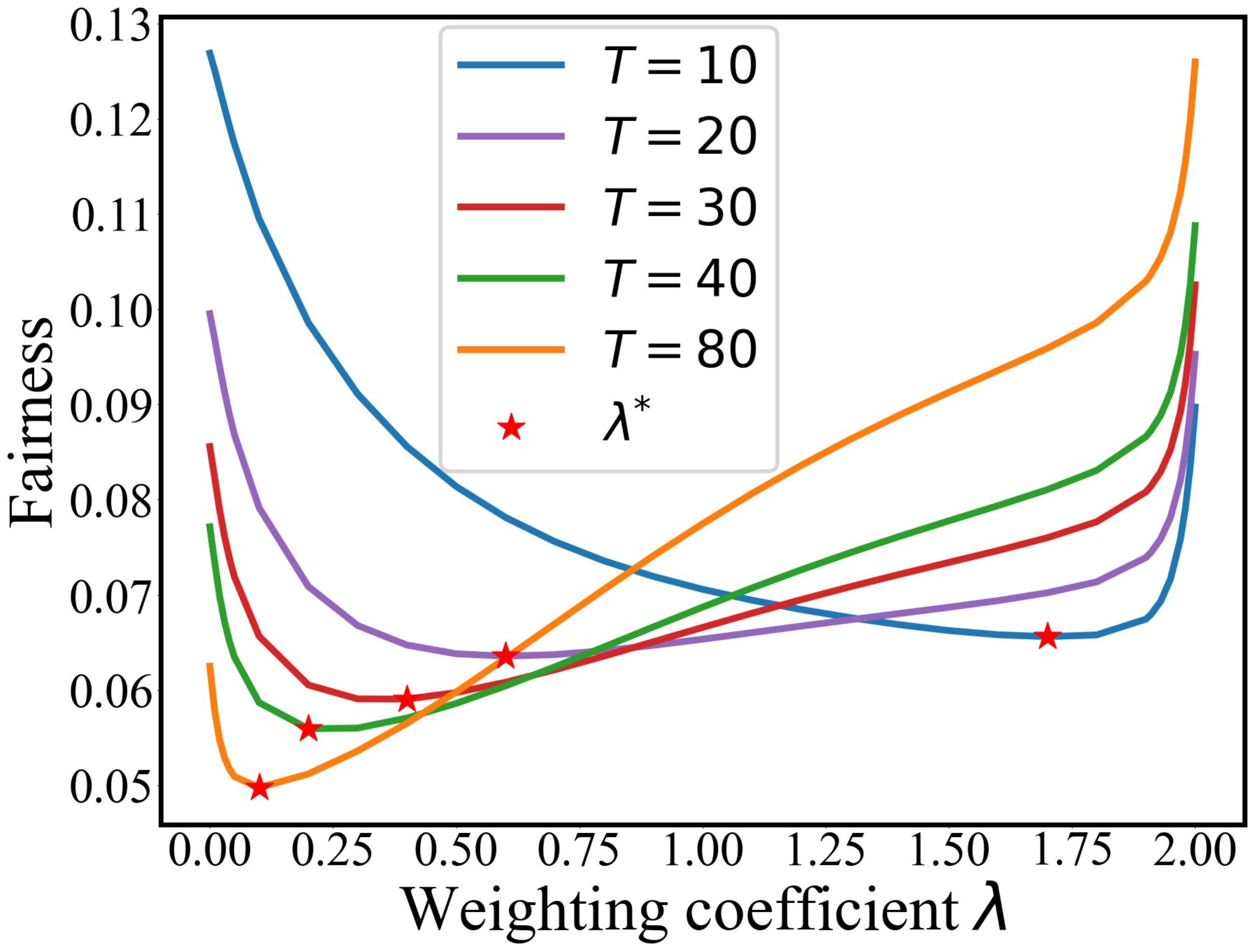}
\hspace{-1mm}
}
\subfigure[Training loss vs. $T$]{ 	\label{lambda_T_loss-1}
\includegraphics[width=0.23\textwidth]{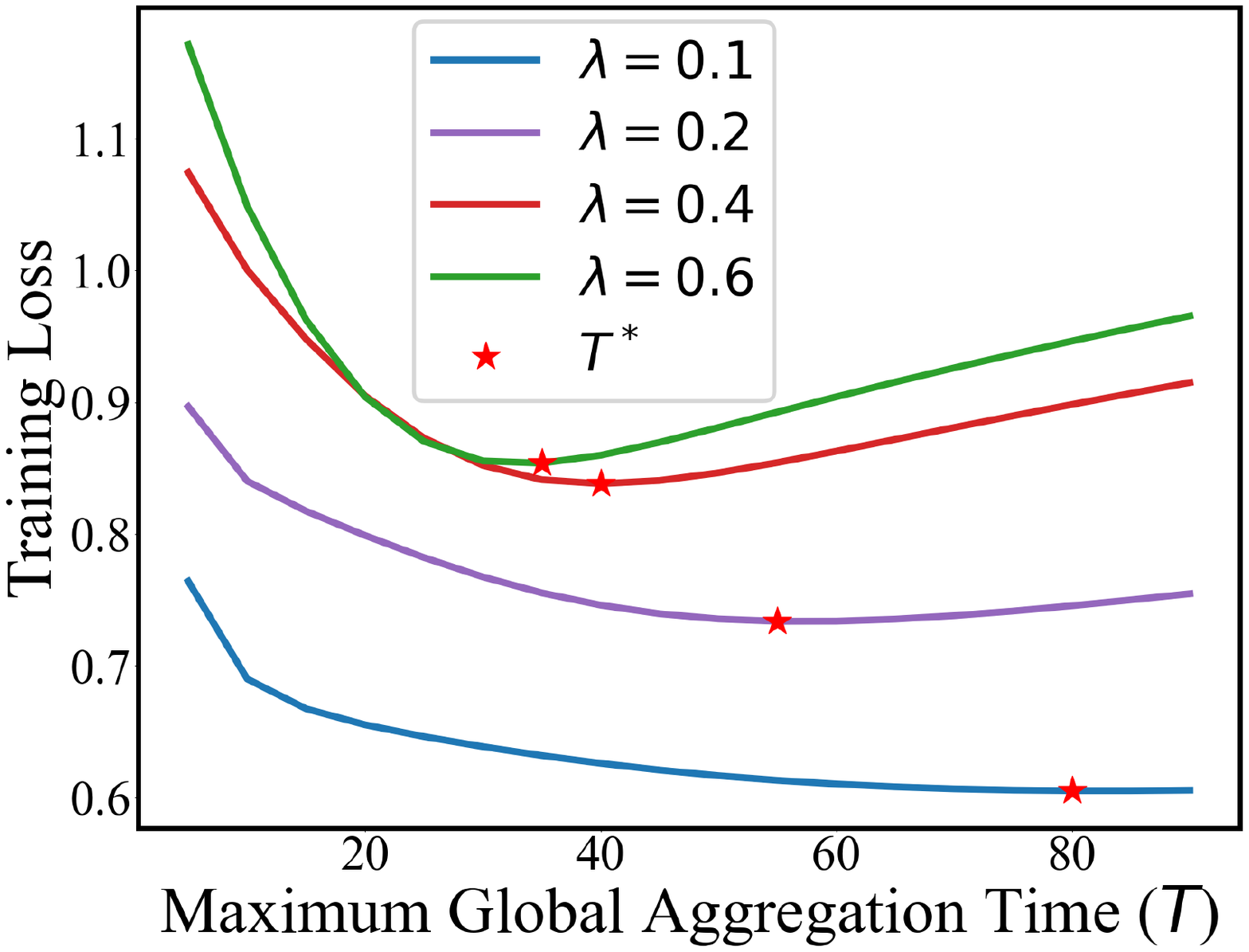}}
	\caption{Fairness and training loss of PL vs. $\lambda$ and $T$, where $\epsilon=100$, $\delta=0.01$, and $C=10$. MNIST is considered.}
    \label{Fairness_loss}
\end{figure}

\appendix
\subsection{Proof of \textbf{Lemma \ref{Lemma1}}}
\label{pLemma1}
Let $g_{n}(\boldsymbol{\varpi}_{n}^{t};\boldsymbol{\omega}^{t})$ be the stochastic gradient of $f_{n}(\boldsymbol{\varpi}_{n}^{t};\boldsymbol{\omega}^{t})$:
\begin{equation}
\label{gradient}
    g_{n}(\boldsymbol{\varpi}_{n}^{t};\boldsymbol{\omega}^{t})=\left(1-\frac{\lambda}{2}\right)\nabla F_{n}(\boldsymbol{\varpi}_{n}^{t})+\lambda(\boldsymbol{\varpi}_{n}^{t}-\boldsymbol{\omega}^{t}) \,.
\end{equation}
As per the PL model of client $n$ at the $(t+1)$-th model update, we have~\cite[Eq.~(96)]{li2021ditto}
	\begin{align}
        \label{PerOneStep00_2}
&\mathbb{E}\!\left[\parallel\!\!\overset{\sim}{\boldsymbol{\varpi}}_{n}^{t+1}\!\!\!\!-\!\boldsymbol{\varpi}_{n}^{\ast}\!\!\parallel^{2}\right]
=\mathbb{E}\left[\parallel\overset{\sim}{\boldsymbol{\varpi}}_{n}^{t}-\boldsymbol{\varpi}_{n}^{\ast}\parallel^{2}\right]+\eta_{\mathrm{L}}^{2}\times \nonumber\\
&\mathbb{E}\!\left[\parallel g_{n}(\overset{\sim}{\boldsymbol{\varpi}}_{n}^{t};\overset{\sim}{\boldsymbol{\omega}}^{t})\parallel^{2}\right]\!\!+\!2\eta_{\mathrm{L}}\mathbb{E}\!\left\langle g_{n}(\overset{\sim}{\boldsymbol{\varpi}}_{n}^{t};\overset{\sim}{\boldsymbol{\omega}}^{t}),\boldsymbol{\varpi}_{n}^{\ast}\!\!-\!\!\overset{\sim}{\boldsymbol{\varpi}}_{n}^{t}\right\rangle.
	\end{align}
The third term on the RHS of (\ref{PerOneStep00_2}) becomes 
{\small
        \begin{subequations} 
	\begin{align}
2\eta_{\mathrm{L}}\mathbb{E}&\left\langle g_{n}(\overset{\sim}{\boldsymbol{\varpi}}_{n}^{t};\overset{\sim}{\boldsymbol{\omega}}^{t}),\boldsymbol{\varpi}_{n}^{\ast}-\overset{\sim}{\boldsymbol{\varpi}}_{n}^{t}\right\rangle \nonumber
\\
\begin{split}
\leq& 2\eta_{\mathrm{L}}\left(1-\frac{\lambda}{2}\right)\mathbb{E}\left[F_{n}\left(\boldsymbol{\varpi}_{n}^{\ast};\overset{\sim}{\boldsymbol{\omega}}^{t}\right)-F_{n}\left(\overset{\sim}{\boldsymbol{\varpi}}_{n}^{t};\overset{\sim}{\boldsymbol{\omega}}^{t}\right)\right] \\
&-\eta_{\mathrm{L}}\left(1-\frac{\lambda}{2}\right)\mu \mathbb{E}\left[\parallel\boldsymbol{\varpi}_{n}^{\ast}-\overset{\sim}{\boldsymbol{\varpi}}_{n}^{t}\parallel^{2}\right]\\
&+2\eta_{\mathrm{L}}\mathbb{E}\left\langle \lambda(\overset{\sim}{\boldsymbol{\varpi}}_{n}^{t}-\overset{\sim}{\boldsymbol{\omega}}^{t}),\boldsymbol{\varpi}_{n}^{\ast}-\overset{\sim}{\boldsymbol{\varpi}}_{n}^{t}\right\rangle 
\end{split}
\label{divengence3}
\\
\begin{split}
=&2\eta_{\mathrm{L}}\mathbb{E}\left[f_{n}\left(\boldsymbol{\varpi}_{n}^{\ast};\overset{\sim}{\boldsymbol{\omega}}^{t}\right)-f_{n}\left(\overset{\sim}{\boldsymbol{\varpi}}_{n}^{t};\overset{\sim}{\boldsymbol{\omega}}^{t}\right)\right]\\
&-\eta_{\mathrm{L}}\left(\left(1-\frac{\lambda}{2}\right)\mu+\lambda\right)\mathbb{E}\left[\parallel\boldsymbol{\varpi}_{n}^{\ast}-\overset{\sim}{\boldsymbol{\varpi}}_{n}^{t}\parallel^{2}\right]         \,.
\end{split}
\label{divengence4}
	\end{align}
	\end{subequations} %
 }%
Here, 
(\ref{divengence3}) is obtained by first applying  (\ref{gradient}) and considering the $\mu$-strong convexity of $F_n(\cdot)$, followed by substituting (\ref{fn}) into (\ref{divengence3}) and yielding (\ref{divengence4}).

By substituting (\ref{divengence4}) into (\ref{PerOneStep00_2}), we obtain the upper bound of $\mathbb{E}\left[\parallel\overset{\sim}{\boldsymbol{\varpi}}_{n}^{t+1}-\boldsymbol{\varpi}_{n}^{\ast}\parallel^{2}\right]$ in \eqref{divengence4}, as given by
{\small
    \begin{subequations} 
	\begin{align}
        \label{PerOneStep1_1}
        \begin{split}
\!\!\!\mathbb{E}&\left[\parallel\overset{\sim}{\boldsymbol{\varpi}}_{n}^{t+1}\!-\!\boldsymbol{\varpi}_{n}^{\ast}\parallel^{2}\right]\leq(1-\eta_{\mathrm{L}}\left(\left(1-\frac{\lambda}{2}\right)\mu+\lambda\right))\times \\
&\,\,\,\,\,\,\,\,\,\,\,\,\,\,\,\,\,\,\,\,\,\,\mathbb{E}\left[\parallel\overset{\sim}{\boldsymbol{\varpi}}_{n}^{t}-\boldsymbol{\varpi}_{n}^{\ast}\parallel^{2}\right]+\eta_{\mathrm{L}}^{2}\mathbb{E}\left[\parallel g(\overset{\sim}{\boldsymbol{\varpi}}_{n}^{t};\overset{\sim}{\boldsymbol{\omega}}^{t})\parallel^{2}\right] \\
&\,\,\,\,\,\,\,\,\,\,\,\,\,\,\,\,\,\,\,\,+2\eta_{\mathrm{L}}\mathbb{E}\left[f_{n}\left(\boldsymbol{\varpi}_{n}^{\ast};\overset{\sim}{\boldsymbol{\omega}}^{t}\right)-f_{n}\left(\overset{\sim}{\boldsymbol{\varpi}}_{n}^{t};\overset{\sim}{\boldsymbol{\omega}}^{t}\right)\right] 
\end{split}
\\
\label{PerOneStep1_3}
\begin{split}
\leq&(1-\eta_{\mathrm{L}}\left(\left(1-\frac{\lambda}{2}\right)\mu+\lambda\right))\mathbb{E}\left[\parallel\overset{\sim}{\boldsymbol{\varpi}}_{n}^{t}-\boldsymbol{\varpi}_{n}^{\ast}\parallel^{2}\right]\\
&+\eta_{\mathrm{L}}^{2}\mathbb{E}\left[\parallel g(\overset{\sim}{\boldsymbol{\varpi}}_{n}^{t};\boldsymbol{\omega}^{\ast})\parallel^{2}\right]+\eta_{\mathrm{L}}^{2}\lambda^{2}\mathbb{E}\left[\parallel{\overset{\sim}{\boldsymbol{\omega}}}^{t}-\boldsymbol{\omega}^{\ast}\parallel^{2}\right] \\
&+2\eta_{\mathrm{L}}^{2}\lambda\sqrt{\mathbb{E}\left[\parallel g(\overset{\sim}{\boldsymbol{\varpi}}_{n}^{t};\boldsymbol{\omega}^{\ast})\parallel^{2}\right]}\sqrt{\mathbb{E}\left[\parallel\overset{\sim}{\boldsymbol{\omega}}^{t}-\boldsymbol{\omega}^{\ast}\parallel^{2}\right]} \\
&+2\eta_{\mathrm{L}}\lambda\sqrt{\mathbb{E}\left[\parallel\overset{\sim}{\boldsymbol{\varpi}}_{n}^{t}-\boldsymbol{\varpi}_{n}^{\ast}\parallel^{2}\right]\mathbb{E}\left[\parallel\overset{\sim}{\boldsymbol{\omega}}^{t}-\boldsymbol{\omega}^{\ast}\parallel^{2}\right]} 
\end{split} 
\\
\label{PerOneStep1_4}
\begin{split}
\leq&(1\!\!-\!\!\eta_{\mathrm{L}}\!\left(\!\left(1\!\!-\!\!\frac{\lambda}{2}\!\right)\!\mu\!\!+\!\!\lambda\!\right)\!\!+\!\!\eta_{\mathrm{L}}\!)\!\mathbb{E}\!\left[\!\parallel\!\overset{\sim}{\boldsymbol{\varpi}}_{n}^{t}\!\!-\!\!\boldsymbol{\varpi}_{n}^{\ast}\!\parallel^{2}\!\right]+\!\!\left(\!\eta_{\mathrm{L}}^{2}\!\!+\!\!\eta_{\mathrm{L}}^{2}\lambda^{2}\right)\!\!\bullet\\&\mathbb{E}\!\!\left[\!\parallel \!g(\overset{\sim}{\boldsymbol{\varpi}}_{n}^{t};\boldsymbol{\omega}^{\ast})\!\!\parallel^{2}\right]\!\!+\!\!\left(\!2\eta_{\mathrm{L}}^{2}\lambda^{2}\!\!+\!\!\eta_{\mathrm{L}}\lambda^{2}\!\right)\!\mathbb{E}\!\!\left[\!\parallel\!\!{\overset{\sim}{\boldsymbol{\omega}}}^{t}\!\!-\!\!\boldsymbol{\omega}^{\ast}\!\!\parallel^{2}\!\right]
\end{split}
\\
\label{PerOneStep1_5}
\begin{split}
\leq&(1-\eta_{\mathrm{L}}\left(\left(1-\frac{\lambda}{2}\right)\mu+\lambda\right)+\eta_{\mathrm{L}})\mathbb{E}\left[\parallel\overset{\sim}{\boldsymbol{\varpi}}_{n}^{t}-\boldsymbol{\varpi}_{n}^{\ast}\parallel^{2}\right]\\
&+\left(\eta_{\mathrm{L}}^{2}+\eta_{\mathrm{L}}^{2}\lambda^{2}\right)\mathbb{E}\left[\parallel g(\overset{\sim}{\boldsymbol{\varpi}}_{n}^{t};\boldsymbol{\omega}^{\ast})\parallel^{2}\right]\\
&+\frac{4\eta_{\mathrm{L}}^{2}\lambda^{2}+2\eta_{\mathrm{L}}\lambda^{2}}{\mu}\mathbb{E}\left[F\left(\overset{\sim}{\boldsymbol{\omega}}^{t}\right)-F\left(\boldsymbol{\omega}^{\ast}\right)\right],
        \end{split}
	\end{align}
	\end{subequations} %
 }%
 where (\ref{PerOneStep1_3}) and (\ref{PerOneStep1_4}) are based on Cauchy–Schwarz inequality, and
 (\ref{PerOneStep1_5}) is due to the $\mu$-strong convexity of $F(\cdot)$ and the inequality $\parallel\overset{\sim}{\boldsymbol{\omega}}^{t}-\boldsymbol{\omega}^{\ast}\parallel^{2}\leq\frac{2}{\mu}\left[F\left(\overset{\sim}{\boldsymbol{\omega}}^{t}\right)-F\left(\boldsymbol{\omega}^{\ast}\right)\right]$.

 Further, we establish the upper bounds for the squared distances between the PL model and the FL local model, and between the PL model and the optimal FL global model, and for the squared norm of the gradient of the PL model:
 {\small
       	\begin{equation} 
        \label{E1}
\mathbb{E}\left[\parallel\overset{\sim}{\boldsymbol{\varpi}}_{n}^{t}-\boldsymbol{u}_{n}^{\ast}\parallel^{2}\right]\leq\frac{1}{\mu^{2}}\mathbb{E}\left[\parallel\nabla F_{n}(\overset{\sim}{\boldsymbol{\varpi}}_{n}^{t})\parallel^{2}\right]\leq\frac{G_{0}^{2}}{\mu^{2}} ;
	\end{equation}
\begin{align}
        \label{E2}
\mathbb{E}\left[\parallel\overset{\sim}{\boldsymbol{\varpi}}_{n}^{t}-\boldsymbol{\omega}^{\ast}\parallel^{2}\right]
\leq&\mathbb{E}\left[\parallel\overset{\sim}{\boldsymbol{\varpi}}_{n}^{t}-\boldsymbol{u}_{n}^{\ast}\parallel^{2}\right]+\mathbb{E}\left[\parallel \boldsymbol{u}_{n}^{\ast}-\boldsymbol{\omega}^{\ast}\parallel^{2}\right]\nonumber\\
&+2\mathbb{E}\left[\parallel\overset{\sim}{\boldsymbol{\varpi}}_{n}^{t}-\boldsymbol{u}_{n}^{\ast}\parallel\times \parallel \boldsymbol{u}_{n}^{\ast}-\boldsymbol{\omega}^{\ast}\parallel\right]\nonumber\\
\leq&\frac{G_{0}^{2}}{\mu^{2}}+M^{2}+\frac{2MG_{0}}{\mu} ;
	\end{align} 
         \begin{align}
        \label{E3} 
        \mathbb{E}\!\left[\parallel \!\!g(\overset{\sim}{\boldsymbol{\varpi}}_{n}^{t};\boldsymbol{\omega}^{\ast})\!\!\parallel^{2}\right]
        \leq {G} .
	\end{align}  %
 }%
Here, (\ref{E1}) is due to the convexity of $F_n(\cdot)$ and the assumption that $\mathbb{E}\left[\parallel\nabla F_{n}(\overset{\sim}{\boldsymbol{\omega}}^{t})\parallel^{2}\right]\leq G_{0}^{2}$.
 (\ref{E2}) is obtained by first leveraging the Cauchy-Schwarz inequality and then substituting (\ref{E1}). Likewise, (\ref{E3}) is obtained by first using the Cauchy-Schwarz inequality and then plugging (\ref{E2}).

By plugging (\ref{E1})--(\ref{E3}) into (\ref{PerOneStep1_5}), it readily follows that 
{\small
	\begin{align}
        \label{PerOneStep2}
        \mathbb{E}\!\left[\parallel\!\!\overset{\sim}{\boldsymbol{\varpi}}_{n}^{t+1}\!\!\!\!-\!\boldsymbol{\varpi}_{n}^{\ast}\!\!\parallel^{2}\right]
        &\!\leq \varepsilon_{\mathrm{L}}\mathbb{E}\left[\parallel\!\!\overset{\sim}{\boldsymbol{\varpi}}_{n}^{t}-\boldsymbol{\varpi}_{n}^{\ast}\!\!\parallel^{2}\right]\!+\!\left(\eta_{\mathrm{L}}^{2}+\eta_{\mathrm{L}}^{2}\lambda^{2}\right)G\nonumber\\        &\!\!\!\!\!\!\!\!\!\!\!\!\!\!\!\!\!\!\!\!\!\!\!+\frac{4\eta_{\mathrm{L}}^{2}\lambda^{2}+2\eta_{\mathrm{L}}\lambda^{2}}{\mu}\mathbb{E}\left[F\left(\overset{\sim}{\boldsymbol{\omega}}^{t}\right)-F\left(\boldsymbol{\omega}^{\ast}\right)\right],
	\end{align}%
 }%
which concludes this proof. 

\subsection{Proof of \textbf{Lemma \ref{lemma2}}}
\label{plemma2}

    Considering the aggregation process with the DP noise added in the $(t+1)$-th aggregation, the expected difference between $\overset{\sim}{\boldsymbol{\omega}}^{t+1}$ and $\boldsymbol{\omega}^{\ast}$ is upper bounded by \cite[Eq.~(59)]{wei2020federated}
    {\small
   	\begin{equation} 
	\begin{split}
        \label{glotplus1}
&\mathbb{E}\left[F\left(\overset{\sim}{\boldsymbol{\omega}}^{t+1}\right)\right]-F\left(\boldsymbol{\omega}^{\ast}\right)\leq\mathbb{E}\left[F\left(\overset{\sim}{\boldsymbol{\omega}}^{t}\right)\right]-F\left(\boldsymbol{\omega}^{\ast}\right)\\
&+\eta_{\mathrm{G}}\left(\frac{\eta_{\mathrm{G}}L}{2}-1\right)\parallel\nabla F(\overset{\sim}{\boldsymbol{\omega}}^{t})\parallel^{2}+\frac{L}{2 N^2}\mathbb{E}\left[\parallel\mathbf{z}^{t}\parallel^{2}\right]
        \,.
	\end{split}
	\end{equation}%
 }%
 {By substituting $F\left(\boldsymbol{\omega}\right)-F\left(\boldsymbol{\omega}^{\ast}\right)\leq\frac{1}{2\mu}\parallel\nabla F\left(\boldsymbol{\omega}\right)\parallel^{2}$, $\mathbb{E}\left[\parallel\mathbf{z}^{t}\parallel^{2}\right]=d\sigma_z^2$, and \eqref{noiseSigma1} into \eqref{glotplus1}, \textbf{Lemma \ref{lemma2}} readily follows. }

\subsection{Proof of \textbf{Theorem \ref{con_UpperBound}}}
\label{pcon_UpperBound}

 When $\varepsilon_{\mathrm{L}} \neq \varepsilon_{\mathrm{G}}$, it follows from (\ref{PerOneStep2}) and (\ref{gloOneStep}) that
 {\small
        \begin{subequations} 
        \label{PerConvergence}
	\begin{align}
        \label{PerConvergence1}
        \begin{split}
\mathbb{E}&\!\left[\!\parallel\!\!\overset{\sim}{\boldsymbol{\varpi}}_{n}^{t+1}\!\!\!\!-\!\boldsymbol{\varpi}_{n}^{\ast}\!\!\parallel^{2}\!\right]\!\!
\leq \!\varepsilon_{\mathrm{L}}^{t+1}\mathbb{E}\!\left[\parallel\!\!\boldsymbol{\varpi}_{n}^{0}\!-\!\boldsymbol{\varpi}_{n}^{\ast}\!\!\parallel^{2}\right]\!\!+\!\!\left(\eta_{\mathrm{L}}^{2}\!\!+\!\!\eta_{\mathrm{L}}^{2}\lambda^{2}\right)\!G\!\!\sum_{x=0}^{t}\varepsilon_{\mathrm{L}}^{x}\\
&+\frac{4\eta_{\mathrm{L}}^{2}\lambda^{2}+2\eta_{\mathrm{L}}\lambda^{2}}{\mu}\left[\sum_{x=0}^{t}\varepsilon_{\mathrm{L}}^{x}\varepsilon_{\mathrm{G}}^{t-x}\mathbb{E}\left[F\left(\boldsymbol{\omega}^{0}\right)-F\left(\boldsymbol{\omega}^{\ast}\right)\right] \right. \\
&\left.+\sum_{x=0}^{t-1}\sum_{y=0}^{t-1-x}\varepsilon_{\mathrm{L}}^{x}\varepsilon_{\mathrm{G}}^{y}{\varphi_{\mathrm{L}}}T\right]
\end{split}
\\
\label{PerConvergence2}
\begin{split}
= &\varepsilon_{\mathrm{L}}^{t+1}\mathbb{E}\left[\parallel\boldsymbol{\varpi}_{n}^{0}-\boldsymbol{\varpi}_{n}^{\ast}\parallel^{2}\right]+\left(1+\lambda^{2}\right)\eta_{\mathrm{L}}^{2}G\frac{\varepsilon_{\mathrm{L}}^{t+1}-1}{\varepsilon_{\mathrm{L}}-1}\\
&+\frac{(4\eta_{\mathrm{L}}^{2}+2\eta_{\mathrm{L}})\lambda^{2}}{\mu}\left[\frac{\varepsilon_{\mathrm{L}}^{t+1}-\varepsilon_{\mathrm{G}}^{t+1}}{\varepsilon_{\mathrm{L}}-\varepsilon_{\mathrm{G}}}\mathbb{E}\left[F\left(\boldsymbol{\omega}^{0}\right)-F\left(\boldsymbol{\omega}^{\ast}\right)\right] \right.\\
&\left.+\left(\frac{\varepsilon_{\mathrm{L}}^{t}-\varepsilon_{\mathrm{G}}^{t}}{\frac{\varepsilon_{\mathrm{L}}}{\varepsilon_{\mathrm{G}}}-1}-\frac{\varepsilon_{\mathrm{L}}^{t}-1}{\varepsilon_{\mathrm{L}}-1}\right)\frac{{\varphi_{\mathrm{L}}}T}{\varepsilon_{\mathrm{G}}-1}\right].        
        \end{split}
	\end{align}
	\end{subequations}%
 }%
Here, (\ref{PerConvergence2}) is based on geometric series $\sum_{x=0}^{t}a^x=\frac{a^{t+1}-1}{a-1}$, where $a = \frac{\varepsilon_{\mathrm{L}}}{\varepsilon_{\mathrm{G}}} \neq 1$. 

When $\varepsilon_{\mathrm{L}} = \varepsilon_{\mathrm{G}}$, 
(\ref{PerConvergence1}) is rewritten as 
{\small
\begin{subequations} 
        \label{PerConvergence_LG}
	\begin{align}
        \label{PerConvergence1_LG}
        \begin{split}
\mathbb{E}&\!\!\left[\!\parallel\!\!\overset{\sim}{\boldsymbol{\varpi}}_{n}^{t+1}\!\!\!\!-\!\boldsymbol{\varpi}_{n}^{\ast}\!\!\parallel^{2}\!\right]
\!\!\leq \!\varepsilon_{\mathrm{L}}^{t+1}\mathbb{E}\!\left[\parallel\!\!\boldsymbol{\varpi}_{n}^{0}\!\!-\!\boldsymbol{\varpi}_{n}^{\ast}\!\!\parallel^{2}\right]\!\!+\!\!\left(\eta_{\mathrm{L}}^{2}\!\!+\!\!\eta_{\mathrm{L}}^{2}\lambda^{2}\right)\!G\!\!\sum_{x=0}^{t}\!\varepsilon_{\mathrm{L}}^{x}\\
&+\frac{4\eta_{\mathrm{L}}^{2}\lambda^{2}+2\eta_{\mathrm{L}}\lambda^{2}}{\mu}\left[\sum_{x=0}^{t}\varepsilon_{\mathrm{L}}^{t}\mathbb{E}\left[F\left(\boldsymbol{\omega}^{0}\right)-F\left(\boldsymbol{\omega}^{\ast}\right)\right] \right. \\
&\left.+\sum_{x=0}^{t-1}\varepsilon_{\mathrm{L}}^{x}\sum_{y=0}^{t-1-x}\varepsilon_{\mathrm{L}}^{y}{\varphi_{\mathrm{L}}}T\right]
\end{split}
\\
\label{PerConvergence2_LG}
\begin{split}
= &\varepsilon_{\mathrm{L}}^{t+1}\mathbb{E}\left[\parallel\boldsymbol{\varpi}_{n}^{0}-\boldsymbol{\varpi}_{n}^{\ast}\parallel^{2}\right]+\left(1+\lambda^{2}\right)\eta_{\mathrm{L}}^{2}G\frac{\varepsilon_{\mathrm{L}}^{t+1}-1}{\varepsilon_{\mathrm{L}}-1}\\
&+\frac{(4\eta_{\mathrm{L}}^{2}+2\eta_{\mathrm{L}})\lambda^{2}}{\mu}\left[(t+1)\varepsilon_{\mathrm{L}}^{t}\mathbb{E}\left[F\left(\boldsymbol{\omega}^{0}\right)-F\left(\boldsymbol{\omega}^{\ast}\right)\right] \right.\\
&\left.+\left(t\varepsilon_{\mathrm{L}}^{t}-\frac{\varepsilon_{\mathrm{L}}^{t}-1}{\varepsilon_{\mathrm{L}}-1}\right)\frac{{\varphi_{\mathrm{L}}}T}{\varepsilon_{\mathrm{L}}-1}\right],        
        \end{split}
	\end{align}
	\end{subequations}%
 }%
With $\mu >\frac{2-2\lambda}{2-\lambda}$ under Assumption 1, $\varepsilon_{\mathrm{L}}<1$ and therefore DP-Ditto converges as $t$ increases. After $T$ aggregations, the convergence upper bound of DP-Ditto is (\ref{theoConverT})  {or (\ref{theoConverT_LG})}. 

\subsection{Proof of \textbf{Lemma \ref{noisy_pL}}}
\label{BLR_models}
By finding the solution to the first derivative of $f_{n}(\boldsymbol{\varpi}_{n};{\boldsymbol{\omega}}^{\ast})$ in (\ref{fn}) and then replacing $\boldsymbol{\omega}^{\ast}$ with $\overset{\sim}{\boldsymbol{\omega}}^{\ast}$ in the solution, the optimal PL model is given by
{\small
 	\begin{subequations} 
      \label{OpLiPersonal0}
	\begin{align}
     \overset{\sim}{\boldsymbol{\varpi}}_{n}^{\ast}&\!\left(\lambda\right)\!\!= \!\!\left(\!\frac{2\!-\!\lambda}{b}\mathbf{X}_{n}^{\intercal}\mathbf{X}_{n}\!+\!\lambda I\!\right)^{-1}\!\!\left(\!\frac{2-\lambda}{b}\mathbf{X}_{n}^{\intercal}\mathbf{Y}_{n}\!+\!\lambda\overset{\sim}{\boldsymbol{\omega}^{\ast}}\!\right)
     \label{OpLiPersonal1}
 \\
    \begin{split}
     =&\left(\frac{2-\lambda}{b}\mathbf{X}_{n}^{\intercal}\mathbf{X}_{n}+\lambda I\right)^{-1}\left(\left(\frac{2-\lambda}{b}\mathbf{X}_{n}^{\intercal}\mathbf{X}_{n}\right)\hat{\boldsymbol{u}}_{n} \right. \\
     &+\left. {\lambda}\sum_{n=1}^{N}  \left(\mathbf{X}^{\intercal}\mathbf{X}\right)^{-1}\mathbf{X}_{n}^{\intercal}\mathbf{X}_{n}\left(\hat{\boldsymbol{u}}_{n}+{\mathbf{z}_{n}}\right)\right) ,
    \label{OpLiPersonal3}
    \end{split} 
	\end{align}
	\end{subequations} %
 }%
which is obtained by plugging (\ref{omega*2}) 
 and then~(\ref{noisyestimation1}) in~(\ref{OpLiPersonal1}).
 
Without DP, given $\hat{\boldsymbol{u}}_n,n\in \mathbb{N}$ in~(\ref{estimatedmodel}) and the distributions of $\boldsymbol{\tau}_n$ and $\boldsymbol{\nu}_n$, $\boldsymbol{u}_n^{\ast}$ is given in \cite[Lemma 6]{li2021ditto}:
\begin{equation}
\begin{split}
    \label{un*}
    \boldsymbol{u}_n^{\ast}=&\Phi(\boldsymbol{u}_n^{\ast})(\Phi_n-\zeta^{2}\mathbf{I}_{d})^{-1}\hat{\boldsymbol{u}}_n \\
    &+\Phi(\boldsymbol{u}_n^{\ast})(\Phi_{\setminus n}(\boldsymbol{u}_n^{\ast}))^{-1} \boldsymbol{u}_{\setminus n}(\boldsymbol{\omega}^{\ast})+\boldsymbol{\vartheta}_n
    \,,
    \end{split}
\end{equation}
where $\boldsymbol{\vartheta}_n \sim \mathcal{N}(0,\Phi(\boldsymbol{u}_n^{\ast}))$, and
\begin{subequations}
\begin{align}
    \label{sig1}
    &\Phi_n \triangleq \sigma^{2}\left(\mathbf{X}_{n}^{\intercal}\mathbf{X}_{n}\right)^{-1}+\zeta^{2}\mathbf{I}_{d} ;\\
    \label{sig2}
    &\Phi_{\setminus n}(\boldsymbol{\omega}^{\ast}) \triangleq (\underset{m\in\mathbb{N},m\neq n}{\sum}\Phi_m^{-1})^{-1} ;\\
    \label{sig4}
    &\Phi_{\setminus n}(\boldsymbol{u}_n^{\ast})\triangleq \Phi_{\setminus n}(\boldsymbol{\omega}^{\ast})+\zeta^{2}\mathbf{I}_{d}; \\
    \label{sig5}
    & \Phi(\boldsymbol{u}_n^{\ast}) \triangleq ((\Phi_{\setminus n}(\boldsymbol{u}_n^{\ast}))^{-1}+(\Phi_n-\zeta^{2}\mathbf{I}_{d})^{-1})^{-1}; \\
        \label{sig3}
    &\boldsymbol{u}_{\setminus n}(\boldsymbol{\omega}^{\ast}) \triangleq \Phi_{\setminus n}(\boldsymbol{\omega}^{\ast}) \underset{m\in\mathbb{N},m\neq n}{\sum} \Phi_m^{-1} \hat{\boldsymbol{u}}_m .
\end{align}
\end{subequations}
%
%
%
%
%
After substituting $\mathbf{X}_{n}^{\intercal}\mathbf{X}_{n}=\rho\mathbf{I}_{d}$ first into (\ref{OpLiPersonal3}) and (\ref{sig1}), we then sequentially plug the results of (\ref{sig1}) -- (\ref{sig5}) into the one after. 
Then, \textbf{Lemma \ref{noisy_pL}} follows.

\subsection{Proof of \textbf{Theorem \ref{FAIRNESS}}}
\label{pFAIRNESS}

According to \textbf{Definition \ref{DefFairness}} and (\ref{loss}), the optimal $\lambda^{\ast}$ is
{\small
\begin{equation} 
	\begin{split}
        \label{opLambda1}
        \lambda^{\ast}
        =&\arg\underset{\lambda}{\min}\,\mathbb{E}\left\{\frac{1}{N}\sum_{n=1}^{N}\left[\left(\parallel \boldsymbol{u}_{n}^{\ast}-\overset{\sim}{\boldsymbol{\varpi}_{n}^{\ast}}\left(\lambda\right)\parallel^{2}\right)^{2}\right] \right.\\
        &-\left.\left(\frac{1}{N}\sum_{n=1}^{N}\left[\parallel \boldsymbol{u}_{n}^{\ast}-\overset{\sim}{\boldsymbol{\varpi}_{n}^{\ast}}\left(\lambda\right)\parallel^{2}\right]\right)^{2} \right\} 
        \triangleq \arg\underset{\lambda}{\min}\, R(\lambda)
        \,.
	\end{split}
	\end{equation} 
 }

From (\ref{Optimal_w_n}) and (\ref{Optimal_u_n}), it follows that 
 {\small
	\begin{align}
        \label{subtract}
        &\overset{\sim}{\boldsymbol{\varpi}_{n}^{\ast}}\left(\lambda\right)-\boldsymbol{u}_{n}^{\ast}=\left[\frac{\left(2-\lambda\right)\rho N+b\lambda}{\left(\left(2-\lambda\right)\rho+b\lambda\right)N}-\frac{\sigma^{2}+\rho N\zeta^{2}}{N\sigma^{2}+\rho N\zeta^{2}}\right]\hat{\boldsymbol{u}}_{n} \nonumber \\
        &+\left[\frac{b\lambda}{\left(\left(2-\lambda\right)\rho+b\lambda\right)N}-\frac{\sigma^{2}}{N\sigma^{2}+\rho N\zeta^{2}}\right] \underset{m\in\mathbb{N},m\neq n}{\sum}\hat{\boldsymbol{u}}_{m} \nonumber \\
        &+ \left[\frac{b\lambda}{\left(\left(2-\lambda\right)\rho+b\lambda\right)N}\sum_{n=1}^{N}\mathbf{z}_{n}-\boldsymbol{\vartheta}_{n}\right]
        \triangleq \mathbf{A}_n+\mathbf{B}_n,
	\end{align}
    \begin{equation} 
	\begin{split}
        \label{A}
        \text{where } &\mathbf{A}_n=\frac{\left(N-1\right)\rho\left(\left(2-\lambda\right)\sigma^{2}-b\lambda\zeta^{2}\right)}{N\left(\left(2-\lambda\right)\rho+b\lambda\right)\left(\sigma^{2}+\rho\zeta^{2}\right)}\hat{\boldsymbol{u}}_{n} +\\
        &\frac{\rho\left(b\lambda\zeta^{2}-\left(2-\lambda\right)\sigma^{2}\right)}{N\left(\left(2-\lambda\right)\rho+b\lambda\right)\left(\sigma^{2}+\rho\zeta^{2}\right)}\underset{m\in\mathbb{N},m\neq n}{\sum}\hat{\boldsymbol{u}}_{m}
        \,;
	\end{split}
	\end{equation}
  	\begin{equation} 
        \label{B}
        \mathbf{B}_n=\frac{b\lambda}{\left(\left(2-\lambda\right)\rho+b\lambda\right)N}\sum_{n=1}^{N}{\mathbf{z}_{n}}-\boldsymbol{\vartheta}_{n}
        \,.
	\end{equation} 
T}he $l$-th ($l=1,\cdots,d$) elements of $\mathbf{A}_n$ and $\mathbf{B}_n$ are given~by
{\small
	\begin{equation} 
	\begin{split}
        \label{Anl1}
        A_{nl}
 =\alpha_{1}\left(\lambda\right)\alpha_{nl} \,;
	\end{split}
	\end{equation}
  	\begin{equation} 
        \label{Bnl}
        {B_{nl}}=\frac{b\lambda}{\left(\left(2-\lambda\right)\rho+b\lambda\right)N}\sum_{n=1}^{N}{{z}_{nl}}-{\vartheta}_{nl}
        \,,
	\end{equation}%
 }%
 where $\hat{{u}}_{nl}$, ${z}_{nl}$, and ${\vartheta}_{nl}$ are the $l$-th elements of $\hat{\boldsymbol{u}}_{n}$, $\mathbf{z}_{n}$, and $\boldsymbol{\vartheta}_n$, respectively; ${\vartheta}_{nl} \sim \mathcal{N}(0,\sigma_{w}^{2})$, c.f. (\ref{Optimal_u_n});
 and
  	\begin{equation} 
	\begin{split}
    & \alpha_{1}\left(\lambda\right)=\frac{\rho\left(\left(2-\lambda\right)\sigma^{2}-b\lambda\zeta^{2}\right)}{N\left(\left(2-\lambda\right)\rho+b\lambda\right)\left(\sigma^{2}+\rho\zeta^{2}\right)}  \,;\\
    & \alpha_{nl}=\left(N-1\right)\hat{{u}}_{nl}-\underset{m\in\mathbb{N},m\neq n}{\sum}\hat{{u}}_{ml} \,. \label{alpha} 
	\end{split}
	\end{equation}
 For conciseness, let $\alpha_{0}\left(\lambda\right)=\frac{b\lambda}{\left(2-\lambda\right)\rho+b\lambda}$, then we have
 	\begin{equation} 
	\begin{split}
        \label{alpha1}
        \alpha_{1}\!\left(\lambda\right)\!=\!&\frac{\sigma^{2}}{N\left(\sigma^{2}+\rho\zeta^{2}\right)}-\frac{1}{N}\alpha_{0}\left(\lambda\right) 
        \!=\!S_{1}\!-\!S_{2}\alpha_{0}\left(\lambda\right)
        \,,
	\end{split}
	\end{equation}
where $S_{1}=\frac{\sigma^{2}}{N\left(\sigma^{2}+\rho\zeta^{2}\right)}$, and $S_{2}=\frac{1}{N}$. 
According to (\ref{Bnl}), ${B_{nl}}\sim \mathcal{N}(0,{\sigma_{B}^{2}})$ and ${\sigma_{B}^{2}}=\sigma_{w}^{2}+\left[\alpha_{0}\left(\lambda\right)\right]^{2}\frac{\sigma_{z}^{2}}{N^2}$.
Consider that $\parallel \boldsymbol{u}_{n}^{\ast}-\overset{\sim}{\boldsymbol{\varpi}_{n}^{\ast}}\left(\lambda\right)\parallel^{2}=\sum_{l=1}^{d}\left({u}_{nl}^{\ast}-\overset{\sim}{{\varpi}_{nl}^{\ast}}\left(\lambda\right)\right)^{2}$. 
With (\ref{subtract}), $R(\lambda)$ can be written as
{\small
	   \begin{align}
        R(\lambda)
        =&\mathbb{E}\left\{\frac{1}{N}\sum_{n=1}^{N}\left[\left(\sum_{l=1}^{d}\left(A_{nl}+{{B_{nl}}}\right)^{2}\right)^{2}\right] \right. \nonumber\\
        &-\left.\left(\frac{1}{N}\sum_{n=1}^{N}\left[\sum_{l=1}^{d}\left(A_{nl}+{{B_{nl}}}\right)^{2}\right]\right)^{2} \right\}.
         \label{opLambda2_2}
	\end{align}%
 }%
{Given $\left\{\hat{\boldsymbol{u}}_n,n=1,\cdots,N\right\}$ estimated from samples,} (\ref{opLambda2_2}) can be rewritten as
 {\small
  \begin{subequations} 
	\begin{align}
        \begin{split}
            &R(\lambda)=\mathbb{E}\left\{\sum_{l=1}^d \frac{1}{N}\sum_{n=1}^{N}\left[\left(A_{nl}+{B_{nl}}\right)^4\right]\right. \\
            &+2\sum_{l,l' \in \left[d\right],l \neq l'}\frac{1}{N}\sum_{n=1}^{N}\left[\left(A_{nl}+{B_{nl}}\right)^2\left(A_{nl'}+B_{nl'}\right)^2\right] \\
            &-\sum_{l=1}^d\left(\frac{1}{N}\sum_{n=1}^{N}\left[\left(A_{nl}+{B_{nl}}\right)^2\right]\right)^2- 2\sum_{l,l' \in \left[d\right],l \neq l'}\frac{1}{N} \\
            &\left.\times \sum_{n=1}^{N}\left[\left(A_{nl}+{B_{nl}}\right)^2\right]\frac{1}{N}\sum_{n=1}^{N}\left[\left(A_{nl'}+B_{nl'}\right)^2\right]\right\}
            \label{var0}
        \end{split}
        \\
        \begin{split}
        &=2d{\sigma_{B}^{2}}+4{\sigma_{B}^{2}}\alpha_{1}^{2}\left(\lambda\right)\sum_{l=1}^{d}\frac{1}{N}\sum_{n=1}^{N}\left[\alpha_{nl}^{2}\right]+\alpha_{1}^{4}\left(\lambda\right) \\
        &\frac{1}{N}\!\sum_{n=1}^{N}\left[\left(\sum_{l=1}^{d}\alpha_{nl}^{2}\right)^{2}\right]\!\!-\!\alpha_{1}^{4}\!\left(\!\lambda\!\right)\!\left(\frac{1}{N}\!\sum_{n=1}^{N}\left[\sum_{l=1}^{d}\alpha_{nl}^{2}\right]\right)\!^{2}
        \end{split} 
        \label{var2}
        \\
        &\triangleq 2d{\sigma_{B}^{2}}+4{\sigma_{B}^{2}}\alpha_{1}^{2}\left(\lambda\right)G_{1}+\alpha_{1}^{4}\left(\lambda\right)G_{2}-\alpha_{1}^{4}\left(\lambda\right)G_{1}^{2}
        \,,
        \label{var3}
	\end{align}
	\end{subequations}%
 }%
 where (\ref{var0}) is due to the fact that 
 $\frac{1}{N}\sum_{n=1}^{N}\left[\left(\sum_{l=1}^{d} a_{nl}\right)^{2}\right]=\frac{1}{N}\sum_{n=1}^{N}\left[\sum_{l=1}^{d} a_{nl}^{2}+2\sum_{l,l' \in \left[d\right],l \neq l'} a_{nl} a_{nl'}\right]$, 
 with $a_{nl}=(A_{nl}+B_{nl})^2$; 
  (\ref{var2}) is obtained by taking the expectation of (\ref{var0}) with $B_{nl}\sim \mathcal{N}(0,\sigma_B^2)$~\cite[Eq.~(56)]{li2021ditto} and then plugging (\ref{Anl1}).
 Here, $G_{1}= \sum_{l=1}^{d} \frac{1}{N}\sum_{n=1}^{N}\left[\alpha_{nl}^{2}\right]$; $G_{2}=\frac{1}{N}\sum_{n=1}^{N}\left[\left( \sum_{l=1}^{d} \alpha_{nl}^{2}\right)^{2}\right]$.
By plugging (\ref{alpha1}) and 
{$\sigma_B^2$}
into (\ref{var3}), 
\textbf{Theorem \ref{FAIRNESS}} follows.



\subsection{Proof of \textbf{Theorem \ref{Op_lambda}}}
\label{pOp_lambda}
Based on \textbf{Theorem {\ref{FAIRNESS}}}, the first and second derivatives of $R(\lambda)$ with respect to $\alpha_0$ are given by
{\small
  	\begin{equation} 
	\begin{split}
        \label{1derivative}
        \frac{\partial R}{\partial\alpha_{0}}=&4d\frac{\sigma_{z}^{2}}{N^{2}}\alpha_{0}+8G_{1}\frac{\sigma_{z}^{2}}{N^{2}}\left(S_{1}-S_{2}\alpha_{0}\right)^{2}\alpha_{0} \\
        &-8S_{2}G_{1}\left(\sigma_{w}^{2}+\left[\alpha_{0}\right]^{2}\frac{\sigma_{z}^{2}}{N^{2}}\right)\left(S_{1}-S_{2}\alpha_{0}\right)\\
        &-4S_{2}\left(G_{2}-G_{1}^{2}\right)\left(S_{1}-S_{2}\alpha_{0}\right)^{3};
	\end{split}
	\end{equation}
\begin{equation} 
        \label{2derivative}
        \frac{\partial^{2}R}{\partial\alpha_{0}^{2}}= 
    4D\frac{\sigma_z^2}{N^2}\!+\!8G_{1}\sigma_{w}^{2}S_{2}^{2}\!+\!12S_{2}^{2}\left(G_{2}\!-\!G_{1}^{2}\right)\left(S_{1}\!\!-\!\!S_{2}\alpha_{0}\right)^{2},
	\end{equation} %
 }%
 where $D=d+2G_1\left(6S_{2}^{2}\alpha_{0}^{2}-6S_{1}S_{2}\alpha_{0}+S_{1}^{2}\right)$ for brevity. After reorganization, $D$ can be rewritten as
 {\small
 \begin{subequations}
    \begin{align}
    D
    \label{D_5}
    =&(d-\frac{S_{1}^{2}}{N}\sum_{n=1}^{N}\parallel \alpha_n \parallel^2)+\frac{12G_1(\alpha_0-\frac{1}{2}NS_1)^2}{N^2} \\
        \label{D_6}
    =&(\!d\!-\!\frac{S_{1}^{2}}{N}\!\!\sum_{n=1}^{N}\!\!\parallel\!\!  N\hat{\boldsymbol{u}}_{n}\!\!-\!\!\sum_{m=1}^N \!\!\hat{\boldsymbol{u}}_{m}\!\!\parallel^2)\! +\!\frac{12G_1\!(\alpha_0\!-\!\frac{NS_1}{2})^2}{N^2}
    \\
        \label{D_8}
    >&(\!d\!-\!\!\frac{S_{1}^{2}}{N}\!\!\sum_{n=1}^{N}\!(\parallel\!\!  N\hat{\boldsymbol{u}}_{n}\!\!\parallel\!\!+\!\!\sum_{m=1}^N\!\parallel \!\!\hat{\boldsymbol{u}}_{m}\!\!\parallel)^2)\!\!+\!\!\frac{12G_1\!(\alpha_0\!\!-\!\!\frac{NS_1}{2}\!)\!^2}{N^2}
    \\
        \label{D_9}
    >&(d\!-\!\frac{S_{1}^{2}}{N}\!\sum_{n=1}^{N}\!(NC\!+\!NC)^2) \!+\!\frac{12G_1(\alpha_0\!-\!\frac{N}{2}S_1)^2}{N^2}
    \\
    \label{D_i}
 =&(d-4C^2N^2S_{1}^{2})+\frac{12G_1(\alpha_0-\frac{N}{2}S_1)^2}{N^2} \,,
    \end{align}
\end{subequations}%
}%
 where (\ref{D_6}) is obtained by substituting (\ref{alpha}) into (\ref{D_5}); 
 (\ref{D_8}) is based on the Cauchy–Schwarz inequality.

When $C<\frac{\sqrt{d}}{2NS_1}$ in~(\ref{D_i}), $D>0$ in (\ref{2derivative}) and, in turn, $\frac{\partial^{2}R}{\partial\alpha_{0}^{2}}>0$ in (\ref{2derivative}). 
In other words, $\frac{\partial R}{\partial\alpha_{0}}$ increases monotonically in $\alpha_{0} \in [0,1]$. 
Clearly, $\frac{\partial R}{\partial\alpha_{0} }<0$ when $\alpha_{0}=0$; $\frac{\partial R}{\partial\alpha_{0} }>0$ when $\alpha_{0}=1$. Therefore, $R(\lambda)$ first increases and then decreases in $\alpha_0 \in [0,1]$.
There must be a unique $\alpha_{0}^{\ast} \in [0,1]$ satisfying~\eqref{opAlpha0}.
Since $\alpha_{0}^{\ast}=\frac{b\lambda^{\ast}}{\left(2-\lambda^{\ast}\right)\rho+b\lambda^{\ast}}$ increases monotonically in $\lambda \in [0,2]$, 
with the uniqueness of $\alpha_{0}^{\ast}$, the existence and uniqueness of the optimal $\lambda^{\ast}=\frac{2\rho\alpha_{0}^{\ast}}{\left(1-\alpha_{0}^{\ast}\right)b+\rho\alpha_{0}^{\ast}}$ are confirmed.

\subsection{Proof of \textbf{Corollary \ref{o_fairness}}}
\label{po_fairness}

According to (\ref{2derivative}), when $C<\frac{\sqrt{d}}{2S_0}$, $\frac{\partial^{2}R}{\partial\alpha_{0}^{2}}$ monotonically increases with $\sigma_z^2$; in other words, $\frac{\partial R}{\partial\alpha_{0}}$ grows fast with $\sigma_z^2$, when $C<\frac{\sqrt{d}}{2S_0}$.
On the other hand, when $\alpha_0=0$, $\frac{\partial R}{\partial\alpha_{0}}$ is independent of $\sigma_z^2$, which can be readily concluded by substituting $\alpha_0=0$ into (\ref{2derivative}):
  	\begin{equation} 
	\begin{split}
        \label{1derivative_0}
        \frac{\partial R}{\partial\alpha_{0}}\Big|_{\alpha_0=0}= -8S_1S_{2}G_{1}\sigma_{w}^{2}-4S_1^6S_{2}\left(G_{2}-G_{1}^{2}\right) \,.
	\end{split}
	\end{equation} 
Note that $\lambda^*$ corresponds to $\alpha_{0}^*$, i.e., the solution to $\frac{\partial R}{\partial\alpha_{0}}=0$; in other words, $\alpha_{0}^*$ is the intersection of the curve $\mathcal{F}(\alpha_{0}) = \frac{\partial R}{\partial\alpha_{0}}$ with the $x$-axis. 
Given all $\mathcal{F}(\alpha_{0}) = \frac{\partial R}{\partial\alpha_{0}}$ curves pass $(0,\frac{\partial R}{\partial\alpha_{0}}\big|_{\alpha_0=0})$ under any $\sigma_z^2$ and the slopes of the curves increase with $\sigma_z^2$, $\alpha_{0}^*$ decreases as $\sigma_z^2$ increases. Since $\lambda^{\ast}=\frac{2\rho\alpha_{0}^{\ast}}{\left(1-\alpha_{0}^{\ast}\right)b+\rho\alpha_{0}^{\ast}}$ increases monotonically with $\alpha_{0}^*$, it is concluded that $\lambda^*$ decreases as $\sigma_z^2$ increases.

\bibliography{DittoDP}

\begin{thebibliography}{10}
\providecommand{\url}[1]{#1}
\csname url@samestyle\endcsname
\providecommand{\newblock}{\relax}
\providecommand{\bibinfo}[2]{#2}
\providecommand{\BIBentrySTDinterwordspacing}{\spaceskip=0pt\relax}
\providecommand{\BIBentryALTinterwordstretchfactor}{4}
\providecommand{\BIBentryALTinterwordspacing}{\spaceskip=\fontdimen2\font plus
\BIBentryALTinterwordstretchfactor\fontdimen3\font minus
  \fontdimen4\font\relax}
\providecommand{\BIBforeignlanguage}[2]{{%
\expandafter\ifx\csname l@#1\endcsname\relax
\typeout{** WARNING: IEEEtran.bst: No hyphenation pattern has been}%
\typeout{** loaded for the language `#1'. Using the pattern for}%
\typeout{** the default language instead.}%
\else
\language=\csname l@#1\endcsname
\fi
#2}}
\providecommand{\BIBdecl}{\relax}
\BIBdecl

\bibitem{li2021ditto}
T.~Li, S.~Hu, A.~Beirami \emph{et~al.}, ``Ditto: Fair and robust federated
  learning through personalization,'' in \emph{Proc. 38th Int. Conf. Mach.
  Learn.}, vol. 139, 2021, pp. 6357--6368.

\bibitem{zhang2021optimizing}
W.~Zhang, D.~Yang, W.~Wu \emph{et~al.}, ``Optimizing federated learning in
  distributed industrial iot: A multi-agent approach,'' \emph{IEEE J. Sel.
  Areas Commun.}, vol.~39, no.~12, pp. 3688--3703, 2021.

\bibitem{zhang2024det}
W.~Zhang, N.~Tang, D.~Yang \emph{et~al.}, ``Det (com) 2: Deterministic
  communication and computation integration toward {AIGC} services,''
  \emph{IEEE Wirel. Commun.}, vol.~31, no.~3, pp. 32--41, 2024.

\bibitem{cui2022multi}
Q.~Cui, X.~Zhao, W.~Ni \emph{et~al.}, ``Multi-agent deep reinforcement
  learning-based interdependent computing for mobile edge computing-assisted
  robot teams,'' \emph{IEEE Trans. Veh. Technol}, vol.~72, no.~5, pp.
  6599--6610, 2022.

\bibitem{abadi2016deep}
M.~Abadi, A.~Chu, I.~Goodfellow \emph{et~al.}, ``Deep learning with
  differential privacy,'' in \emph{n Proc. ACM SIGSAC Conf. Comput. Commun.
  Secur. (CCS)}, 2016, pp. 308--318.

\bibitem{zhang2022federated}
J.~Zhang, J.~Zhang, D.~W.~K. Ng \emph{et~al.}, ``{Federated learning-based
  cell-free massive MIMO system for privacy-preserving},'' \emph{IEEE Trans.
  Wirel. Commun.}, vol.~22, no.~7, pp. 4449--4460, 2022.

\bibitem{wei2021low}
K.~Wei, J.~Li, C.~Ma \emph{et~al.}, ``Low-latency federated learning over
  wireless channels with differential privacy,'' \emph{IEEE J. Sel. Areas
  Commun.}, vol.~40, no.~1, pp. 290--307, 2021.

\bibitem{elgabli2021harnessing}
A.~Elgabli, J.~Park, C.~B. Issaid \emph{et~al.}, ``Harnessing wireless channels
  for scalable and privacy-preserving federated learning,'' \emph{IEEE Trans.
  Commun.}, vol.~69, no.~8, pp. 5194--5208, 2021.

\bibitem{wei2020federated}
K.~Wei, J.~Li, M.~Ding \emph{et~al.}, ``Federated learning with differential
  privacy: Algorithms and performance analysis,'' \emph{IEEE Trans. Inf.
  Forensics Secur.}, vol.~15, pp. 3454--3469, 2020.

\bibitem{zhao2020local}
Y.~Zhao, J.~Zhao, M.~Yang \emph{et~al.}, ``Local differential privacy-based
  federated learning for internet of things,'' \emph{IEEE Internet Things J.},
  vol.~8, no.~11, pp. 8836--8853, 2020.

\bibitem{truex2020ldp}
S.~Truex, L.~Liu, K.-H. Chow \emph{et~al.}, ``{LDP-Fed: Federated learning with
  local differential privacy},'' in \emph{Proc. 3rd ACM Int. Workshop Edge
  Syst. Anal. Netw.}, 2020, pp. 61--66.

\bibitem{yuan2023amplitude}
X.~Yuan, W.~Ni, M.~Ding \emph{et~al.}, ``Amplitude-varying perturbation for
  balancing privacy and utility in federated learning,'' \emph{IEEE Trans. Inf.
  Forensics Security}, vol.~18, pp. 1884--1897, 2023.

\bibitem{chen2022feddual}
Q.~Chen, Z.~Wang, H.~Wang \emph{et~al.}, ``{FedDual: Pair-wise gossip helps
  federated learning in large decentralized networks},'' \emph{IEEE Trans. Inf.
  Forensics Security}, vol.~18, pp. 335--350, 2022.

\bibitem{hu2020personalized}
R.~Hu, Y.~Guo, H.~Li \emph{et~al.}, ``Personalized federated learning with
  differential privacy,'' \emph{IEEE Internet Things J.}, vol.~7, no.~10, pp.
  9530--9539, 2020.

\bibitem{liu2022privacy}
K.~Liu, S.~Hu, S.~Z. Wu \emph{et~al.}, ``On privacy and personalization in
  cross-silo federated learning,'' \emph{Proc. Adv. Neural Inf. Process. Syst.
  (NeurIPS)}, vol.~35, pp. 5925--5940, 2022.

\bibitem{sun2021pain}
P.~Sun, H.~Che, Z.~Wang, Y.~Wang, T.~Wang, L.~Wu, and H.~Shao, ``Pain-fl:
  Personalized privacy-preserving incentive for federated learning,''
  \emph{IEEE J. Sel. Areas Commun.}, vol.~39, no.~12, pp. 3805--3820, 2021.

\bibitem{liu2024differentially}
H.~Liu, J.~Yan, and Y.-J.~A. Zhang, ``{Differentially private over-the-air
  federated learning over MIMO fading channels},'' \emph{IEEE Trans. Wirel.
  Commun.}, vol.~41, no.~11, pp. 3533--3547, 2024.

\bibitem{okegbile2023differentially}
S.~D. Okegbile, J.~Cai, H.~Zheng \emph{et~al.}, ``Differentially private
  federated multi-task learning framework for enhancing human-to-virtual
  connectivity in human digital twin,'' \emph{IEEE J. Sel. Areas Commun.},
  2023.

\bibitem{park2023differential}
S.~Park and W.~Choi, ``On the differential privacy in federated learning based
  on over-the-air computation,'' \emph{IEEE Trans. Wirel. Commun.}, 2023.

\bibitem{yan2024peaches}
J.~Yan, J.~Liu, H.~Xu \emph{et~al.}, ``Peaches: Personalized federated learning
  with neural architecture search in edge computing,'' \emph{IEEE Trans. Mob.
  Comput.}, pp. 1--17, 2024.

\bibitem{wei2023personalized}
K.~Wei, J.~Li, C.~Ma \emph{et~al.}, ``Personalized federated learning with
  differential privacy and convergence guarantee,'' \emph{IEEE Trans. Inf.
  Forensics Security}, vol.~18, pp. 4488--4503, 2023.

\bibitem{t2020personalized}
C.~T~Dinh, N.~Tran, and J.~Nguyen, ``Personalized federated learning with
  moreau envelopes,'' \emph{Proc. Adv. Neural Inf. Process. Syst. (NeurIPS)},
  vol.~33, pp. 21\,394--21\,405, 2020.

\bibitem{li2020federated}
T.~Li, A.~K. Sahu, M.~Zaheer \emph{et~al.}, ``Federated optimization in
  heterogeneous networks,'' \emph{Proc. 3rd Conf. Mach. Learn. Syst. (MLSys)},
  vol.~2, pp. 429--450, 2020.

\bibitem{huang2021personalized}
Y.~Huang, L.~Chu, Z.~Zhou \emph{et~al.}, ``{Personalized cross-silo federated
  learning on non-IID data},'' in \emph{Proc. AAAI Conf. Artif. Intell.},
  vol.~35, no.~9, 2021, pp. 7865--7873.

\bibitem{luo2022adapt}
J.~Luo and S.~Wu, ``Adapt to adaptation: Learning personalization for
  cross-silo federated learning,'' in \emph{Proc. 31th Int. Joint Conf. Artif.
  Intell. (IJCAI)}, vol. 2022, 2022, pp. 2166--2173.

\bibitem{zhang2023fedala}
J.~Zhang, Y.~Hua, H.~Wang \emph{et~al.}, ``{FedALA: Adaptive local aggregation
  for personalized federated learning},'' in \emph{Proc. AAAI Conf. Artif.
  Intell.}, vol.~37, no.~9, 2023, pp. 11\,237--11\,244.

\bibitem{fallah2020personalized}
A.~Fallah, A.~Mokhtari, and A.~Ozdaglar, ``Personalized federated learning with
  theoretical guarantees: A model-agnostic meta-learning approach,''
  \emph{Proc. Adv. Neural Inf. Process. Syst. (NeurIPS)}, vol.~33, pp.
  3557--3568, 2020.

\bibitem{li2019fedmd}
D.~Li and J.~Wang, ``{FedMD: Heterogenous federated learning via model
  distillation},'' in \emph{Proc. Adv. Neural Inf. Process. Syst. Workshop
  Federated Learn. Data Privacy Confidentiality}, 2019, pp. 1--8.

\bibitem{you2022semi}
C.~You, D.~Feng, K.~Guo \emph{et~al.}, ``Semi-synchronous personalized
  federated learning over mobile edge networks,'' \emph{IEEE Trans. Wirel.
  Commun.}, vol.~22, no.~4, pp. 2262--2277, 2022.

\bibitem{zhang2023federated}
H.~Zhang, M.~Tao, Y.~Shi \emph{et~al.}, ``Federated multi-task learning with
  non-stationary and heterogeneous data in wireless networks,'' \emph{IEEE
  Trans. Wirel. Commun.}, vol.~23, no.~4, pp. 2653--2667, 2023.

\bibitem{li2020fair}
T.~Li, M.~Sanjabi, A.~Beirami \emph{et~al.}, ``Fair resource allocation in
  federated learning,'' in \emph{Proc. Int. Conf. Learn. Represent.}, 2020, pp.
  1--13.

\bibitem{hu2022federated}
Z.~Hu, K.~Shaloudegi, G.~Zhang \emph{et~al.}, ``Federated learning meets
  multi-objective optimization,'' \emph{IEEE Trans. Netw. Sci. Eng.}, vol.~9,
  no.~4, pp. 2039--2051, 2022.

\bibitem{tan2022towards}
A.~Z. Tan, H.~Yu, L.~Cui \emph{et~al.}, ``Towards personalized federated
  learning,'' \emph{IEEE Trans. Neural Netw. Learn. Syst.}, vol.~34, no.~12,
  pp. 9587--9603, 2022.

\bibitem{nasr2019comprehensive}
M.~Nasr, R.~Shokri, and A.~Houmansadr, ``Comprehensive privacy analysis of deep
  learning: Passive and active white-box inference attacks against centralized
  and federated learning,'' in \emph{Proc. IEEE Symp. Secur. Privacy (SP)},
  2019, pp. 739--753.

\bibitem{fredrikson2015model}
M.~Fredrikson, S.~Jha, and T.~Ristenpart, ``Model inversion attacks that
  exploit confidence information and basic countermeasures,'' in \emph{Proc.
  22nd ACM SIGSAC Conf. Comput. Commun. Secur.}, 2015, pp. 1322--1333.

\bibitem{lyu2022privacy}
L.~Lyu, H.~Yu, X.~Ma, C.~Chen, L.~Sun, J.~Zhao, Q.~Yang, and S.~Y. Philip,
  ``Privacy and robustness in federated learning: Attacks and defenses,''
  \emph{IEEE Trans. Neural Netw. Learn. Syst.}, vol.~35, no.~7, pp. 8726--8746,
  2024.

\bibitem{bonawitz2017practical}
K.~Bonawitz, V.~Ivanov, B.~Kreuter \emph{et~al.}, ``Practical secure
  aggregation for privacy-preserving machine learning,'' in \emph{Proc. CCS},
  2017, pp. 1175--1191.

\bibitem{acar2018survey}
A.~Acar \emph{et~al.}, ``A survey on homomorphic encryption schemes: Theory and
  implementation,'' \emph{ACM Computing Surveys (Csur)}, vol.~51, no.~4, pp.
  1--35, 2018.

\bibitem{dwork2014algorithmic}
C.~Dwork and A.~Roth, ``The algorithmic foundations of differential privacy,''
  \emph{Found. Trends Theor. Comput. Sci.}, vol.~9, no. 3--4, pp. 211--407,
  2014.

\bibitem{karimi2016linear}
H.~Karimi, J.~Nutini, and M.~Schmidt, ``Linear convergence of gradient and
  proximal-gradient methods under the polyak-{\l}ojasiewicz condition,'' in
  \emph{Proc. Joint Eur. Conf. Mach. Learn. Knowl. Discovery Databases}.\hskip
  1em plus 0.5em minus 0.4em\relax Springer, 2016, pp. 795--811.

\bibitem{10123399}
W.~Li, T.~Lv \emph{et~al.}, ``Multi-carrier {NOMA}-empowered wireless federated
  learning with optimal power and bandwidth allocation,'' \emph{IEEE Trans.
  Wirel. Commun.}, vol.~22, no.~12, pp. 9762--9777, 2023.

\bibitem{10542235}
W.~Li \emph{et~al.}, ``Decentralized federated learning over imperfect
  communication channels,'' \emph{IEEE Trans. Commun.}, pp. 1--1, 2024.

\bibitem{o2006metric}
M.~O'Searcoid, \emph{Metric spaces}.\hskip 1em plus 0.5em minus 0.4em\relax
  Berlin, Germany: Springer, 2006.

\bibitem{heidarian2016hybrid}
A.~Heidarian and M.~J. Dinneen, ``A hybrid geometric approach for measuring
  similarity level among documents and document clustering,'' in \emph{Proc.
  IEEE 2nd Int. Conf. Big Data Comput. Service Appl. (BigDataService)}.\hskip
  1em plus 0.5em minus 0.4em\relax IEEE, 2016, pp. 142--151.

\bibitem{aggarwal2001surprising}
C.~C. Aggarwal, A.~Hinneburg, and D.~A. Keim, ``On the surprising behavior of
  distance metrics in high dimensional space,'' in \emph{Proc. ICDT}.\hskip 1em
  plus 0.5em minus 0.4em\relax Springer, 2001, pp. 420--434.

\bibitem{dwork2012fairness}
C.~Dwork, M.~Hardt, T.~Pitassi \emph{et~al.}, ``Fairness through awareness,''
  in \emph{Proc. ITCS}, 2012, pp. 214--226.

\bibitem{hardt2016equality}
M.~Hardt, E.~Price, and N.~Srebro, ``Equality of opportunity in supervised
  learning,'' vol.~29, 2016.

\bibitem{cho2020bandit}
Y.~J. Cho, S.~Gupta, G.~Joshi \emph{et~al.}, ``Bandit-based
  communication-efficient client selection strategies for federated learning,''
  in \emph{Proc. 54th Asilomar Conf. Signals, Syst., Comput.}\hskip 1em plus
  0.5em minus 0.4em\relax IEEE, 2020, pp. 1066--1069.

\bibitem{seber2003linear}
G.~A. Seber and A.~J. Lee, \emph{Linear regression analysis}.\hskip 1em plus
  0.5em minus 0.4em\relax Hoboken, NJ, USA: Wiley, 2013, vol. 329.

\bibitem{shmakov2011universal}
S.~L. Shmakov, ``A universal method of solving quartic equations,'' \emph{Int.
  J. Pure Appl. Math}, vol.~71, no.~2, pp. 251--259, 2011.

\bibitem{artin2011algebra}
\BIBentryALTinterwordspacing
M.~Artin, \emph{Algebra}.\hskip 1em plus 0.5em minus 0.4em\relax Pearson
  Education, 2011. [Online]. Available:
  \url{https://books.google.com.au/books?id=S6GSAgAAQBAJ}
\BIBentrySTDinterwordspacing

\end{thebibliography}

\end{document}